\documentclass[11pt,a4paper]{article}

\usepackage{amsmath,amssymb,amsfonts,amsthm,mathtools}
\usepackage{graphicx}
\usepackage{booktabs}
\usepackage{microtype}
\usepackage{subcaption}
\usepackage{adjustbox}
\usepackage{tikz}
\usepackage{xcolor}
\usepackage{placeins}
\usepackage{float}
\usepackage{hyperref}

\usetikzlibrary{arrows.meta,positioning,calc,fit,backgrounds,shapes.geometric}

\newtheorem{theorem}{Theorem}

\theoremstyle{definition}

\theoremstyle{remark}

\definecolor{cblue}{RGB}{59,130,246}
\definecolor{cgreen}{RGB}{34,197,94}
\definecolor{corange}{RGB}{249,115,22}
\definecolor{cred}{RGB}{239,68,68}
\definecolor{cpurple}{RGB}{168,85,247}
\definecolor{cteal}{RGB}{20,184,166}

\usepackage{natbib}

\title{
Support-Safe Variational Hybrid Filtering for Contact-Mode and Sparse-Law Recovery

}

\author{
  Marios Papamichalis\thanks{Human Nature Lab, Yale University, New Haven, CT 06511, \texttt{marios.papamichalis@yale.edu}}
  \and
  Regina Ruane\thanks{Department of Statistics and Data Science, The Wharton School, University of Pennsylvania, 3733 Spruce Street, Philadelphia, PA 19104-6340, \texttt{ruanej@wharton.upenn.edu}}
}
\date{}

\newcommand{\method}{\textsc{VHyDRO}}
\usepackage{tikz}
\usepackage{xcolor}
\usepackage{graphicx}

\usetikzlibrary{
  arrows.meta,
  calc,
  positioning
}

\usepackage{tikz}
\usepackage{xcolor}
\usepackage{graphicx}
\usetikzlibrary{arrows.meta,positioning,calc,fit,backgrounds}

\usepackage{tikz}
\usepackage{xcolor}
\usepackage{caption}
\usetikzlibrary{arrows.meta,positioning,calc}

\definecolor{cblue}{RGB}{59,130,246}
\definecolor{cgreen}{RGB}{34,197,94}
\definecolor{corange}{RGB}{249,115,22}
\definecolor{cred}{RGB}{239,68,68}
\definecolor{cpurple}{RGB}{168,85,247}
\definecolor{cteal}{RGB}{20,184,166}

\begin{document}
\maketitle

\begin{abstract}
Contact-rich robot dynamics are hybrid: a single observation can match several latent states and contact regimes (free, impact, stick--slip). A standard amortized filter that places no probability on a feasible contact transition will permanently lose the branch the robot actually follows. We introduce \method{}, a variational hybrid dynamics learner that prevents this branch loss. At each step, \method{} mixes the learned proposal with a feasible transition law before sampling and importance weighting, ensuring that every transition retained by the model-feasible carrier remains covered. \method{} jointly infers a continuous latent state and a discrete contact mode, and fits a sparse port-Hamiltonian law to each recovered regime. On top of this, three guarantees connect: support coverage stabilizes filtering, the stabilized filter concentrates the discrete contact posterior on coherent regimes, and mode-pure segments admit sparse port-Hamiltonian recovery. The recovery error separates cleanly into filtering, derivative, mode-impurity, and physics-residual parts. Three empirical findings track the same mechanism. Under heavy occlusion the support-safe filter stays usable while a non-defensive proposal collapses. On ManiSkill demonstrations and on four Sawyer/BridgeData task families the discrete state forms temporally coherent contact-regime segments that the discrete state yields a stronger joint profile across ARI, change-point F1, and segment purity than post-hoc and mode-free baselines. On hybrid systems with known equations the mode-conditioned sparse fit recovers the active physical terms; purely predictive baselines do not.
\end{abstract}

\section{Introduction}

Contact-rich manipulation exposes a failure mode that average one-step error can
hide: a support-mismatched amortized filter may delete a model-feasible contact
branch before later observations reveal it. Latent dynamics models and variational state-space models are useful because they maintain hidden state, produce action-conditioned predictive rollouts, and reason under uncertainty. Modern latent dynamics models and variational state-space models provide scalable predictive states from images, proprioception, and recorded robot-object state streams\citep{watter2015embed,ha2018world,hafner2019learning,hafner2019dream,karl2016deep,maddison2017filtering,naesseth2018variational}. The challenge addressed here is to combine scalable variational
filtering with hybrid contact structure, interpretable mode-conditioned
dynamics, and explicit protection against proposal-level support loss.

Contact-rich dynamics, however, are not globally smooth. A trajectory may switch between free motion, sticking, sliding, and impact, and each regime follows different local physics. Treating these regimes as one smooth transition can make a model confident in the wrong future, which is dangerous for planning because rollout errors compound over imagined trajectories. Hybrid and contact-aware models address parts of this issue through switching state-space models, Hybrid-SINDy, ContactNets, and physically structured contact dynamics \citep{linderman2017bayesian,mangan2019model,pfrommer2021contactnets,hochlehnert2021learning}. Particle and sequential Monte Carlo methods provide stronger posterior inference for nonlinear state-space models \citep{doucet2001introduction,andrieu2010particle}, and defensive mixtures reduce support mismatch by ensuring that proposals cover the target support \citep{hesterberg1995weighted}; yet these methods are often expensive to amortize at the scale of vision-based robot learning.

A complementary line of work seeks interpretable physical structure. Sparse and symbolic discovery methods recover governing equations from trajectories or learned coordinates \citep{brunton2016discovering,rudy2017data,schmidt2009distilling,champion2019data}, while Bayesian extensions add uncertainty quantification and robustness under noisy or partial observations \citep{mars2024bayesian,fung2025rapid,rosafalco2024ekf,conti2024veni,conti2026veni}. Physics-informed, continuous-time, Hamiltonian, Lagrangian, and port-Hamiltonian models inject useful inductive biases for stability, passivity, and extrapolation \citep{raissi2019physics,chen2018neural,greydanus2019hamiltonian,cranmer2020lagrangian,lutter2019deep,desai2021port,rettberg2025data}. Relational dynamics and object-centric models capture multi-object and articulated interactions \citep{battaglia2016interaction,kipf2018neural,sanchez2018graph,sanchez2020learning,bishnoi2024discovering,jatavallabhula2023bayesian}, and contact-rich manipulation benchmarks provide natural evaluation settings \citep{mandlekar2018roboturk,mandlekar2021matters,walke2023bridgedata,gu2023maniskill2,xiang2020sapien,mo2019partnet}.

Despite this progress, the ingredients needed for contact-rich hybrid filtering
remain largely separated across method families. Scalable variational latent
models are typically not explicitly hybrid or physically interpretable.
Physics-discovery methods are usually not designed for raw observations,
actions, and partial observability. Posterior-correct particle methods provide
accurate inference, but are difficult to deploy as end-to-end learned perception
models. This leaves a need for hybrid variational state-space models that combine
scalable perception, interpretable dynamics, uncertainty diagnostics, and
model-relative support-safe inference.

We propose \method{} (\emph{variational hybrid dynamics with robust proposal support}), a support-safe hybrid variational state-space model for
partially observed contact dynamics. \method{} represents the manipulation scene by a continuous latent state \(z_t\), and represents the active predictive regime by
a discrete mode \(s_t\). The transition associated with each mode is
parameterized by a port-Hamiltonian core with sparse library corrections,
yielding dynamics that are physically structured and inspectable.

Inference in \method{} uses a support-safe variational family. A standard
amortized proposal may assign zero or very small probability to model-feasible
contact states, producing unstable importance weights and degrading learning.
\method{} mitigates this failure mode by mixing the learned proposal with a
feasible transition carrier. The resulting mixture satisfies \(P_t \ll Q_{t,\lambda}\) and gives positive \(Q_{t,\lambda}\)-mass to every measurable transition set with positive \(P_t\)-mass.

The contribution is not a new defensive-IS inequality; it is a deployable support-safety mechanism for amortized hybrid filtering: choose the mixture mass before sampling, sample from the same mixture used in the importance weight, and audit the resulting representation through filtering stability, mode coherence, and sparse-law recovery.

\paragraph{Contributions.}
\textbf{(i) Support-safe IW/FIVO gate.}
We sample from
\[
Q_{t,\lambda}=(1-\lambda)Q_t+\lambda P_t
\]
and weight using the same realized mixture. Specializing the classical
defensive-mixture inequality \citep{hesterberg1995weighted} to the one-step
IW/FIVO increment gives \(P_t\ll Q_{t,\lambda}\),
\(dP_t/dQ_{t,\lambda}\le 1/\lambda\), and the pre-sampling rule
\(\lambda_t\ge \bar\rho_t/(1+N\tau^2)\) whenever the requested
relative-variance budget is certifiable. The paper-specific point is
operational rather than a new defensive-IS inequality: the certificate is used
before sampling, inside an amortized hybrid variational filter, and the same
\(\lambda_t\) is used in both the proposal and the importance weight.

\textbf{(ii) Predictive contact modes under support repair.}
The stabilized filter maintains a continuous latent state and a discrete
predictive mode. Theorem~\ref{thm:mode_recovery_vhydro} states the standard
predictive-separation condition under which such modes concentrate. The
empirical target is alignment with frozen kinematic proxy regimes, not contact
force ground truth.

\textbf{(iii) Controlled sparse-law audit.}
Once mode-pure segments are available, Theorem~\ref{thm:sparse_ph_recovery}
applies a classical Lasso oracle bound \citep{bickel2009simultaneous} to
mode-conditioned port-Hamiltonian regression with aggregate plug-in
perturbations. Experiment~3 measures aggregate support and coefficient recovery
on known-equation systems; it does not separately certify the individual
filtering, derivative, mode-impurity, and port-Hamiltonian residual terms.

\paragraph{Headline result and falsification target.}
The clearest predicted-then-verified signature is the high-occlusion regime of
Experiment~1. Theorem~\ref{thm:support_safe_budgeted_vhydro} directly controls
the relative normalizer-variance and ESS frontier; NLL and coverage are reported
as downstream collapse diagnostics, not as quantities directly bounded by the
theorem. The \(\lambda=0\) baseline uses the same architecture, particle count,
objective, observation likelihood, resampling rule, training budget, and
numerical stabilizers as the support-safe variants; it removes only the
defensive carrier mass. If the gains came merely from generic clipping or
implementation stabilization rather than support repair, this matched
\(\lambda=0\) baseline should retain the same ESS/N and relative-variance
frontier. Instead, at \(90\%\) occlusion, the conservative support-safe variant
improves ESS/N from \(0.042\) to \(0.141\) on Contact-style tasks and from
\(0.046\) to \(0.151\) on Sawyer/BridgeData-style diagnostics.
\section{Theory}
\label{sec:theory}

Classical tools appear here as certifiable components, but the way they are wired together is specific to \method{}. Theorem~\ref{thm:support_safe_budgeted_vhydro} turns defensive mixing into a history-conditioned, pre-sampling IW/FIVO support gate; the same realized \(\lambda_t\) is used in both proposal and weight. Theorem~\ref{thm:mode_recovery_vhydro} states the predictive-separation condition under which the discrete posterior concentrates on a coherent mode using only likelihood-ratio evidence. Theorem~\ref{thm:sparse_ph_recovery} then applies sparse recovery to mode-pure port-Hamiltonian segments with aggregate filtering, derivative, mode-impurity, and plug-in perturbations. The empirical section tests these three links separately.

\paragraph{Model and training.}
\method{} parameterizes a hybrid latent state \(x_t=(s_t,z_t)\), where
\(z_t\) is a continuous belief state and \(s_t\) is a discrete predictive
contact regime. The amortized proposal \(Q_t\) is observation-adaptive and
sharp. The carrier \(P_t\) is the model's one-step transition law before
conditioning on \(o_{t+1}\):
\[
P_t(dx_{t+1})
=
\pi_\theta(ds_{t+1}\mid x_t,a_t)\,
p_\theta(dz_{t+1}\mid s_{t+1},z_t,a_t).
\]
Its discrete support is the model-feasible mode support, and its continuous
support is the corresponding mode-conditioned transition support with the
validation-fixed noise floor shared across variants. Thus \(P_t\) is not
assumed to be the true physical law; it declares the model-feasible support
that the learned proposal is not allowed to delete. Training samples from
\[
Q_{t,\lambda}=(1-\lambda)Q_t+\lambda P_t,
\]
and weights particles using the same realized mixture density. Support safety
is therefore model-relative: \method{} preserves every measurable transition
set assigned positive mass by \(P_t\), but it cannot recover true physical
branches omitted by \(P_t\). After filtering, mode-pure segments are fit with a
sparse port-Hamiltonian library by regularized regression. The mode count and
candidate library are fixed in the present experiments.

The learned parameters are optimized with the sequential IW/FIVO objective:
at step \(t\), the one-step normalizer estimate is
\[
\widehat Z_t
=
\frac{1}{N}\sum_{i=1}^N
p_\theta(o_{t+1}\mid z_{t+1}^i)
\frac{dP_t}{dQ_{t,\lambda_t}}(x_{t+1}^i),
\qquad
x_{t+1}^i\sim Q_{t,\lambda_t},
\]
and training maximizes the corresponding sum of log normalizer increments.
Thus the same realized mixture is used for sampling and weighting in both
training and evaluation. The theorem requires an \(h_t\)-measurable upper
certificate \(\bar\rho_t\ge\rho_t\) before particles are drawn. In the
experiments, \(\lambda_t\) is selected by a validation-frozen rule before
held-out particles are sampled; certificate plots are diagnostics of the
realized variance frontier and are not used for post-hoc retuning. When the
requested budget cannot be certified, the filter marks it uncertified and uses
a validation-fixed fallback \(\lambda_{\rm fb}\), which is not itself claimed
to certify the requested \(\tau\).

\method{} combines support-safe hybrid filtering, mode concentration, and
sparse-law recovery; Theorems~\ref{thm:support_safe_budgeted_vhydro}--\ref{thm:sparse_ph_recovery}
formalize these measured links.

\begin{figure}[t]
\centering
\fbox{
\begin{minipage}{0.96\linewidth}
\textbf{Support-safe hybrid filtering step.}

\textbf{Input:} particles \(\{(s_t^i,z_t^i,w_t^i)\}_{i=1}^N\), action \(a_t\), observation \(o_{t+1}\), amortized proposal \(Q_t\), feasible transition law \(P_t\), certificate \(\bar\rho_t\), tolerance \(\tau\), and validation-fixed fallback \(\lambda_{\rm fb}\in(0,1]\).

\begin{enumerate}
\item Choose \(\lambda_t\) before sampling. Let
\(\lambda_{\min,t}:=\bar\rho_t/(1+N\tau^2)\). If
\(\lambda_{\min,t}\le1\), set \(\lambda_t=\lambda_{\min,t}\), the smallest
mixture mass that certifies the requested budget. Otherwise mark the requested
budget as uncertified and use the validation-fixed fallback
\(\lambda_{\rm fb}\); this fallback is reported as an uncertified operating
point, not as a proof of the requested \(\tau\).
\item Sample each next particle from
\[
Q_{t,\lambda_t}=(1-\lambda_t)Q_t+\lambda_t P_t .
\]
\item Weight the sample using the realized mixture density:
\[
W_t(x)=p_\theta(o_{t+1}\mid z_{t+1})\frac{dP_t}{dQ_{t,\lambda_t}}(x).
\]
\item Normalize weights, compute ESS/N and relative weight variance, and resample when the ESS rule triggers.
\item Return the filtered particles and the mode posterior \(q_\phi(s_{t+1}\mid o_{\le t+1},a_{\le t})\).
\end{enumerate}
\end{minipage}}
\caption{\textbf{Algorithmic summary of the \method{} filtering step.}
The same \(\lambda_t\) is used for sampling and weighting; no post-hoc
proposal adaptation is used.}
\label{alg:vhydro_filter}
\end{figure}

\begin{figure*}[t]
\centering
\makebox[\textwidth][c]{%
\resizebox{\textwidth}{!}{%
\begin{tikzpicture}[
  x=1cm,y=1cm,
  font=\footnotesize,
  line cap=round,line join=round,
  >=Latex
]

\def\W{6.1}
\def\H{4.05}
\def\G{0.78}

\tikzset{
  panel/.style={
    rounded corners=5pt,
    draw=black!8,
    line width=0.45pt,
    fill=white
  },
  title/.style={
    anchor=north west,
    font=\bfseries\footnotesize,
    text=black!85,
    inner sep=0pt
  },
  axis/.style={
    draw=black!45,
    line width=0.60pt,
    -{Latex[length=2.1mm,width=2.1mm]}
  },
  flow/.style={
    draw=black!18,
    line width=0.70pt,
    -{Latex[length=2.2mm,width=2.2mm]}
  },
  prior/.style={
    draw=black!35,
    dashed,
    line width=1.00pt
  },
  target/.style={
    draw=cblue!90!black,
    line width=1.45pt
  },
  proposal/.style={
    draw=cgreen!85!black,
    line width=1.40pt
  },
  mix/.style={
    draw=cpurple!90!black,
    line width=1.45pt
  },
  cleanlabel/.style={
    fill=white,
    fill opacity=0.92,
    text opacity=1,
    rounded corners=1pt,
    inner xsep=1.5pt,
    inner ysep=0.6pt
  },
  guarantee/.style={
    draw=cpurple!65!black,
    fill=white,
    rounded corners=2pt,
    inner xsep=2pt,
    inner ysep=1pt,
    font=\scriptsize,
    text=cpurple!90!black
  }
}

\coordinate (A0) at (0,0);
\coordinate (B0) at ({\W+\G},0);
\coordinate (C0) at ({2*(\W+\G)},0);

\begin{scope}[shift={(A0)}]
  \draw[panel] (0,0) rectangle (\W,\H);
  \fill[cblue!2, rounded corners=5pt] (0.12,0.12) rectangle (\W-0.12,\H-0.55);

  \node[title] at (0.22,\H-0.20) {(a) Contact + partial observations};

  \draw[black!32, line width=1.0pt] (0.62,0.72) -- (5.46,0.72);
  \draw[black!32, line width=1.0pt] (4.52,0.72) -- (4.52,3.18);

  \draw[cblue!90!black, line width=1.8pt]
    (1.00,1.10) .. controls (1.55,2.28) and (2.65,2.98) .. (3.88,2.76);

  \draw[corange!95!black, line width=1.8pt]
    (3.88,2.76) .. controls (4.18,2.67) and (4.38,2.40) .. (4.45,2.12);

  \draw[cgreen!90!black, line width=1.8pt]
    (4.45,2.12) .. controls (3.96,1.48) and (2.96,1.00) .. (1.95,0.95);

  \filldraw[fill=cblue!10, draw=cblue!90!black, line width=0.9pt]
    (1.00,1.10) circle (0.13);
  \draw[cblue!90!black, line width=0.9pt, -{Latex[length=2mm,width=2mm]}]
    (1.00,1.10) -- (1.38,1.70);

  \foreach \x/\y in {
    1.10/1.03,1.42/1.55,1.88/2.26,2.40/2.70,3.00/2.96,
    3.55/2.90,4.05/2.54,4.30/2.22,3.92/1.60,3.02/1.07,2.32/0.98
  }{
    \fill[black!58] (\x,\y) circle (0.030);
  }

  \foreach \x/\y in {1.68/1.93,2.72/3.00,3.90/2.02}{
    \draw[black!42, line width=0.60pt] (\x,\y) circle (0.065);
  }

  \fill[cteal!8] (0.48,3.24) rectangle (1.03,3.44);
  \draw[cteal!85!black, line width=0.75pt] (0.48,3.24) rectangle (1.03,3.44);
  \fill[cteal!85!black] (0.76,3.34) circle (0.043);

  \draw[cblue!90!black, line width=1.30pt] (0.70,0.28) -- (1.05,0.28);
  \node[anchor=west, text=cblue!90!black, font=\scriptsize] at (1.13,0.28) {free};

  \draw[corange!95!black, line width=1.30pt] (1.92,0.28) -- (2.27,0.28);
  \node[anchor=west, text=corange!95!black, font=\scriptsize] at (2.35,0.28) {impact};

  \draw[cgreen!90!black, line width=1.30pt] (3.38,0.28) -- (3.73,0.28);
  \node[anchor=west, text=cgreen!90!black, font=\scriptsize] at (3.81,0.28) {stick/slip};

  \draw[black!42, line width=0.60pt] (5.02,0.28) circle (0.06);
  \node[anchor=west, text=black!55, font=\scriptsize] at (5.16,0.28) {missing};
\end{scope}

\begin{scope}[shift={(B0)}]
  \draw[panel] (0,0) rectangle (\W,\H);
  \fill[cred!3, rounded corners=5pt] (0.12,0.12) rectangle (\W-0.12,\H-0.55);

  \node[title] at (0.22,\H-0.20) {(b) Naive amortized VI};

  \draw[axis] (0.78,0.72) -- (5.36,0.72) node[anchor=west] {$z_{t+1}$};
  \draw[axis] (0.78,0.72) -- (0.78,3.16);

  \fill[cred!8] (3.70,0.72) rectangle (5.12,3.16);
  \draw[cred!42, line width=0.75pt] (3.70,0.72) -- (3.70,3.16);

  \draw[target, domain=-2:2, samples=120]
    plot ({0.78 + (\x+2)*1.14},
          {0.72 + 0.82*exp(-2.8*(\x+0.58)^2) + 1.05*exp(-3.1*(\x-0.82)^2)});
  \node[cleanlabel, text=cblue!90!black, font=\scriptsize] at (1.42,2.67) {target};

  \draw[proposal, domain=-2:0.30, samples=100]
    plot ({0.78 + (\x+2)*1.14},
          {0.72 + 0.96*exp(-3.1*(\x+0.56)^2)});
  \node[cleanlabel, text=cgreen!90!black, font=\scriptsize] at (2.43,1.28) {proposal};

  \draw[cred!85!black, line width=1.0pt] (4.35,2.18) -- (4.65,2.48);
  \draw[cred!85!black, line width=1.0pt] (4.35,2.48) -- (4.65,2.18);

  \fill[black!2, rounded corners=2pt] (1.00,0.08) rectangle (4.20,0.44);
  \node[anchor=west, text=black!55, font=\scriptsize] at (0.18,0.24) {weights};
  \foreach \i/\h in {0/0.08,1/0.08,2/0.08,3/0.08,4/0.50,5/0.08}{
    \fill[black!40] ({1.24+0.28*\i},0.10) rectangle ++(0.16,\h);
  }
  \node[anchor=west, text=cred!85!black, font=\scriptsize] at (4.00,0.24) {ESS$\downarrow$};
\end{scope}

\begin{scope}[shift={(C0)}]
  \draw[panel] (0,0) rectangle (\W,\H);
  \fill[cpurple!3, rounded corners=5pt] (0.12,0.12) rectangle (\W-0.12,\H-0.55);

  \node[title] at (0.22,\H-0.20) {(c) Support-safe \method{}};

  \draw[axis] (0.78,0.72) -- (5.36,0.72) node[anchor=west] {$z_{t+1}$};
  \draw[axis] (0.78,0.72) -- (0.78,3.16);

  \draw[prior, domain=-2:2, samples=120]
    plot ({0.78 + (\x+2)*1.14},
          {0.72 + 0.72*exp(-0.90*(\x)^2)});
  \node[cleanlabel, text=black!55, font=\scriptsize] at (4.42,1.96) {prior};

  \draw[proposal, domain=-2:2, samples=120]
    plot ({0.78 + (\x+2)*1.14},
          {0.72 + 0.93*exp(-2.05*(\x+0.46)^2)});
  \node[cleanlabel, text=cgreen!90!black, font=\scriptsize] at (1.48,2.10) {$q_\phi$};

  \draw[mix, domain=-2:2, samples=120]
    plot ({0.78 + (\x+2)*1.14},
          {0.72 + 0.54*exp(-2.05*(\x+0.46)^2) + 0.42*exp(-0.90*(\x)^2)});
  \node[cleanlabel, text=cpurple!90!black, font=\scriptsize] at (3.02,1.90) {$q_{\phi,\lambda}$};

  \node[guarantee, anchor=east] at (5.18,2.48)
    {$dP_t/dQ_{t,\lambda}\le 1/\lambda$};

  \draw[black!35, line width=0.55pt] (1.05,2.98) -- (2.52,2.98);
  \fill[cpurple!90!black] (1.78,2.98) circle (0.05);
  \node[anchor=north, text=black!50, font=\scriptsize] at (1.05,3.08) {0};
  \node[anchor=north, text=black!50, font=\scriptsize] at (2.52,3.08) {1};
  \node[anchor=west, text=cpurple!90!black, font=\scriptsize] at (2.66,2.98) {$\lambda$};

  \fill[black!2, rounded corners=2pt] (1.00,0.08) rectangle (4.20,0.44);
  \node[anchor=west, text=black!55, font=\scriptsize] at (0.18,0.24) {weights};
  \foreach \i/\h in {0/0.16,1/0.15,2/0.16,3/0.15,4/0.16,5/0.15}{
    \fill[black!40] ({1.24+0.28*\i},0.10) rectangle ++(0.16,\h);
  }
  \node[anchor=west, text=cpurple!90!black, font=\scriptsize] at (4.00,0.24) {ESS$\uparrow$};
\end{scope}

\draw[flow] ({\W+0.10},{\H/2}) -- ({\W+\G-0.10},{\H/2});
\draw[flow] ({2*(\W+\G)-0.10},{\H/2}) -- ({2*\W+\G+0.10},{\H/2});

\end{tikzpicture}%
}}
\caption{\textbf{Why support-safe hybrid inference matters.}
(a) Under partial observation, multiple contact-consistent futures remain
feasible; filled dots denote observed states and hollow dots denote missing
observations.
(b) A narrow amortized proposal can miss a feasible posterior branch,
concentrating importance weights and lowering ESS.
(c) \method{} mixes the learned proposal with a feasible transition law,
retaining support and bounding the transition-to-proposal density ratio
\(dP_t/dQ_{t,\lambda}\le 1/\lambda\).}
\label{fig:vhydro_supportsafe_final}
\end{figure*}
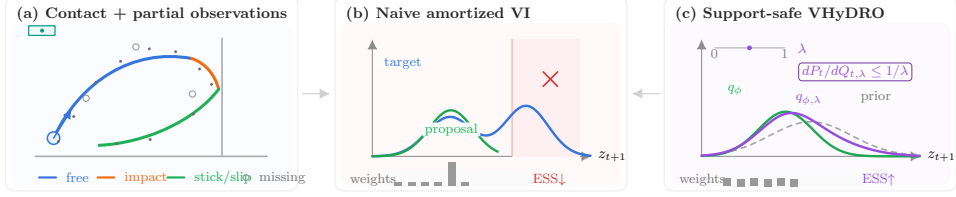

At filtering step \(t\), write
\(x_{t+1}:=(s_{t+1},z_{t+1})\in\mathsf X:=\{1,\ldots,M\}\times\mathbb R^d\).
Conditioned on a particle history \(h_t=(s_t,z_t,a_t,o_{\le t+1})\), let
\(P_t\) be the model one-step transition law used as the feasible support
carrier, and let \(Q_t\) be the observation-adaptive amortized proposal. For
\(\lambda\in(0,1]\), define the defensive proposal

\begin{equation}
\label{eq:defensive_proposal}
Q_{t,\lambda}:=(1-\lambda)Q_t+\lambda P_t .
\end{equation}
Let $g_t(x):=p_\theta(o_{t+1}\mid z_{t+1})$, $Z_t:=\mathbb E_{P_t}[g_t(X)]$, $M_{2,t}:=\mathbb E_{P_t}[g_t(X)^2]$, and $\Pi_t(dx):=g_t(x)P_t(dx)/Z_t$.

\begin{theorem}[Support-safe budgeted IW/FIVO increment]
\label{thm:support_safe_budgeted_vhydro}
Assume $0<Z_t<\infty$ and $M_{2,t}<\infty$.  Then $P_t\ll Q_{t,\lambda}$, $\Pi_t\ll Q_{t,\lambda}$, and
\begin{equation}
\label{eq:rn_bound}
0\le \frac{dP_t}{dQ_{t,\lambda}}(x)\le \frac{1}{\lambda}
\qquad Q_{t,\lambda}\text{-a.s.}
\end{equation}
Consequently, for the one-step weight $W_t(x):=g_t(x)dP_t/dQ_{t,\lambda}(x)$,
\begin{equation}
\label{eq:chi_square_bound}
\chi^2(\Pi_t\|Q_{t,\lambda})
\le
\frac{1}{\lambda}\rho_t-1,
\qquad
\rho_t:=\frac{M_{2,t}}{Z_t^2}.
\end{equation}

Since \(M_{2,t}\ge Z_t^2\) by Jensen's inequality, \(\rho_t\ge1\), so the
right-hand side is nonnegative for \(\lambda\in(0,1]\).

If $X^{1:N}\overset{\mathrm{i.i.d.}}{\sim}Q_{t,\lambda}$ and $\widehat Z_t=N^{-1}\sum_{i=1}^NW_t(X^i)$, then
\begin{equation}
\label{eq:variance_ess_bound}
\begin{aligned}
\mathbb E[\widehat Z_t\mid h_t]&=Z_t,\\
\frac{\operatorname{Var}(\widehat Z_t\mid h_t)}{Z_t^2}
&\le
\frac{1}{N}\left(\frac{\rho_t}{\lambda}-1\right),\\
\frac{\operatorname{ESS}_N}{N}
&\to
\frac{1}{1+\chi^2(\Pi_t\|Q_{t,\lambda})}
\ge\frac{\lambda}{\rho_t}.
\end{aligned}
\end{equation}
Moreover, if an $h_t$-measurable certificate $\bar\rho_t\ge \rho_t$ is available, the relative-variance budget $\operatorname{Var}(\widehat Z_t\mid h_t)/Z_t^2\le\tau^2$ is certified whenever
\begin{equation}
\label{eq:lambda_budget}
\lambda\ge \frac{\bar\rho_t}{1+N\tau^2},
\qquad
\frac{\bar\rho_t}{1+N\tau^2}\le1.
\end{equation}
\end{theorem}

Theorem~\ref{thm:support_safe_budgeted_vhydro} makes the support role explicit: if \(g_t\in L^2(P_t)\), the one-step IW/FIVO normalizer estimate has relative variance bounded by \((\rho_t/\lambda-1)/N\). We test this through ESS/N, relative weight variance, empirical relative variance, NLL, ECE, and Cov90.

\paragraph{Relation to prior bounds.}
The density-ratio inequality is the classical defensive-mixture argument specialized to an IW/FIVO filtering increment. The \method{}-specific contribution is operational and sequential: at each history \(h_t\), \(\lambda_t\) is chosen before particles are drawn, the same realized mixture is used for weighting, and Eq.~\eqref{eq:lambda_budget} becomes a support-coverage gate inside the hybrid filter rather than a post-hoc static importance-sampling diagnostic.

\begin{theorem}[Exponential concentration of predictive modes]
\label{thm:mode_recovery_vhydro}
Consider a length-\(L\) segment generated by a single latent segment mode
\(m\). Actions are treated as fixed interventions in the logged evaluation segment, or are drawn by a nonanticipative behavior policy whose conditional action law is common to all candidate modes; under this intervention/off-policy convention, action probabilities cancel from mode odds. If the policy likelihood depends on the candidate mode, that likelihood must be included in Eq.~\eqref{eq:mode_posterior}, and the theorem below does not apply in this form.
Let \(S\) be the segment mode and let
\begin{equation}
\label{eq:mode_posterior}
\Pi_L(S=s)
\propto
\pi_s\prod_{\ell=0}^{L-1}p_s(O_{\ell+1}\mid\mathcal F_\ell)
\end{equation}

be the clamped-mode \method{} posterior after integrating or
particle-marginalizing the continuous latent state \(z\). The theorem uses
only predictive likelihood increments and does not assume simulator contact
labels. For each wrong mode \(s\neq m\), define
\begin{equation}
\label{eq:log_ratio}
R_{\ell,s}:=\log\frac{p_m(O_{\ell+1}\mid\mathcal F_\ell)}{p_s(O_{\ell+1}\mid\mathcal F_\ell)}.
\end{equation}
Assume $\sum_{\ell=0}^{L-1}\mathbb E_\star[R_{\ell,s}\mid\mathcal F_\ell]\ge \Gamma_{L,s}$ and that the centered increments are conditionally sub-Gaussian with variance proxy $V_s:=\sum_{\ell}\sigma_{\ell,s}^2$.  Let $B_s:=\log(\pi_s/\pi_m)$ and $u_{s,\delta}:=\sqrt{2V_s\log((M-1)/\delta)}$.  For every $\delta\in(0,1)$, with probability at least $1-\delta$,
\begin{equation}
\label{eq:mode_recovery_sharp}
\begin{aligned}
\Pi_L(S\neq m)
&\le
\frac{A_{L,\delta}^{\sharp}}{1+A_{L,\delta}^{\sharp}}
\le A_{L,\delta}^{\sharp},\\
A_{L,\delta}^{\sharp}
&:=
\sum_{s\neq m}
\exp\left(B_s-\Gamma_{L,s}+u_{s,\delta}\right).
\end{aligned}
\end{equation}
In particular, if $\Gamma_{L,s}\ge L\Delta$, $B_s\le B$, and $V_s\le L\sigma^2$, then, with $u_{L,\delta}:=\sqrt{2L\log((M-1)/\delta)}$,
\begin{equation}
\label{eq:mode_recovery_simple}
\Pi_L(S\neq m)
\le
(M-1)\exp\left(B-L\Delta+\sigma u_{L,\delta}\right).
\end{equation}
If the amortized mode posterior satisfies $\mathrm{KL}(q_\phi(S\mid O_{1:L},A_{0:L-1})\|\Pi_L)\le\varepsilon_q$, then on the same event
\begin{equation}
\label{eq:mode_recovery_variational}
q_\phi(S\neq m\mid\cdot)
\le
\frac{A_{L,\delta}^{\sharp}}{1+A_{L,\delta}^{\sharp}}
+\sqrt{\varepsilon_q/2}.
\end{equation}
\end{theorem}

Theorem~\ref{thm:mode_recovery_vhydro} states the predictive-separation
condition needed for posterior concentration of a predictive mode up to label
permutation. We test this with Mode F1, ARI, change-point F1, and segment
purity computed against frozen kinematic proxy labels.

\paragraph{Connection to Experiment~2.}
The theorem identifies the sufficient quantities for concentration:
log-likelihood separation \(\Gamma_{L,s}\), variance proxy \(V_s\), and prior
log-odds \(B_s\). We do not claim a fully numerical per-task certificate of
these constants. Instead, Table~\ref{tab:si_exp2_per_task} reports held-out
metrics computed with the frozen kinematic proxy labels, auditing whether the
recovered discrete state behaves like a coherent predictive contact regime.

\begin{theorem}[Filtering- and mode-robust sparse port-Hamiltonian recovery]
\label{thm:sparse_ph_recovery}
Fix a recovered mode $m$ with index set $I_m$ and $n_m:=|I_m|$.  Suppose the true mode law has sparse port-Hamiltonian form
\begin{equation}
\label{eq:sparse_ph_model}
\begin{aligned}
\dot z_i&=(J_m-R_m)\nabla H_m(z_i)+G_ma_i+\varepsilon_i,\\
H_m(z)&=\Theta(z)^\top\xi_m^\star.
\end{aligned}
\end{equation}
where $J_m^\top=-J_m$, $R_m\succeq0$, and $\xi_m^\star$ is $k$-sparse on support $S$.  Let \method{} produce the stacked plug-in regression
\begin{equation}
\label{eq:plugin_regression}
\begin{aligned}
\widehat b&=\widehat A\xi_m^\star+\zeta_m,\\
\widehat A_i&=(\widehat J_m-\widehat R_m)\nabla\Theta(\widehat z_i)^\top,\\
\widehat b_i&=\widehat{\dot z}_i-\widehat G_ma_i.
\end{aligned}
\end{equation}
Assume the score and restricted-strong-convexity conditions
\begin{equation}
\label{eq:score_rsc}
\begin{aligned}
\left\|\frac{1}{n_m}\widehat A^\top\zeta_m\right\|_\infty
&\le\frac{\lambda_m}{2},\\
\frac{1}{n_m}\|\widehat A\Delta\|_2^2
&\ge\kappa_m\|\Delta\|_2^2
\quad
\text{for }\|\Delta_{S^c}\|_1\le3\|\Delta_S\|_1.
\end{aligned}
\end{equation}
For the Lasso estimator
\begin{equation}
\label{eq:sparse_lasso}
\widehat\xi_m\in\arg\min_\xi
\left\{
\frac{1}{2n_m}\|\widehat b-\widehat A\xi\|_2^2+\lambda_m\|\xi\|_1
\right\},
\end{equation}
we have
\begin{equation}
\label{eq:sparse_bounds}
\begin{aligned}
\|\widehat\xi_m-\xi_m^\star\|_2
&\le
\frac{4\sqrt{k}\lambda_m}{\kappa_m},\\
\frac{1}{n_m}\|\widehat A(\widehat\xi_m-\xi_m^\star)\|_2^2
&\le
\frac{16k\lambda_m^2}{\kappa_m}.
\end{aligned}
\end{equation}
If \(\beta_{\min,m}:=\min_{j\in S}|\xi_{m,j}^\star|
>8\sqrt{k}\lambda_m/\kappa_m\), thresholding \(\widehat\xi_m\) at any level in
\((4\sqrt{k}\lambda_m/\kappa_m,\,
\beta_{\min,m}-4\sqrt{k}\lambda_m/\kappa_m)\) recovers \(S\). If
\(\zeta_m=e_m+u_m\), \(e_m\) is conditionally sub-Gaussian with scale
\(\sigma_m\), and
\(\|(1/n_m)\widehat A^\top u_m\|_\infty\le\nu_m\), then the score part of
condition~\eqref{eq:score_rsc} holds with probability at least \(1-\delta\)
for
\begin{equation}
\label{eq:lambdam_choice}
\begin{aligned}
\lambda_m
&\ge
2\left[
L_{A,m}\sigma_m\sqrt{\frac{2\log(2p/\delta)}{n_m}}+\nu_m
\right],\\
L_{A,m}
&:=\max_j\frac{\|\widehat A_{\cdot j}\|_2}{\sqrt{n_m}}.
\end{aligned}
\end{equation}
\end{theorem}

The restricted-strong-convexity part of condition~\eqref{eq:score_rsc} is
assumed for the realized design; if it holds only on a separate event, the
probability statement is combined with that event by a union bound.

Under the score and restricted-strong-convexity assumptions, Theorem~\ref{thm:sparse_ph_recovery} gives finite-sample sparse-recovery bounds for mode-conditioned port-Hamiltonian regression with aggregate filtering, derivative, mode-impurity, and plug-in perturbations. Once mode-pure segments are available,
\method{} can recover interpretable equations evaluated by sparse support F1,
coefficient error, vector-field NRMSE, and physical-constant error.

\paragraph{Relation to prior bounds.}
The result is a Lasso oracle bound applied to a mode-conditioned
port-Hamiltonian library with plug-in errors. The four-term decomposition is
used as claim bookkeeping: it identifies the aggregate sources of perturbation
that can affect sparse recovery, but the experiments measure aggregate sparse
support and coefficient recovery rather than separately certifying each
perturbation component.

The four components of
\[
\nu_m
=
\nu_m^{\text{filt}}
+\nu_m^{\text{der}}
+\nu_m^{\text{mode}}
+\nu_m^{\text{pH}}
\]
are defined in Appendix~\ref{app:proof_sparse_ph_recovery}. They should be read
as an aggregate error accounting device, not as four independently measured
empirical certificates.

\paragraph{Assumptions and claim boundary.}
All three guarantees are model-relative. The carrier \(P_t\) must assign positive mass to the physical branch one wants to preserve; a branch omitted by \(P_t\) cannot be recovered by support mixing. The mode set and sparse library are finite and fixed in the present experiments. The action-cancellation condition in Theorem~\ref{thm:mode_recovery_vhydro} matches the logged-intervention evaluation protocol used here; mode-dependent closed-loop policies require including the policy likelihood in the mode odds. Finally, the results below audit filtering, segmentation, and sparse-law recovery, not certified closed-loop control.

\section{Experiments}
\label{sec:experiments}
The experiments audit the three mechanism links: support-safe IW/FIVO under occlusion, proxy-label contact-regime segmentation, and sparse-law recovery on known-equation systems. We report ESS/N, relative weight variance, empirical relative variance, NLL, ECE, Cov90, Mode F1, ARI, change-point F1, segment purity, support F1, coefficient error, physical-constant error, and vector-field NRMSE.

\paragraph{Shared protocol and matched ablations.}
Each applicable method/task/condition cell uses \(20\) independent seeds, with no seed pruning. Within each comparison, variants share architecture family, latent dimension, particle count, training budget, observation likelihood, resampling rule, occlusion schedule when applicable, and proxy-label construction. The \(\lambda=0\) / no-support ablation removes only defensive support mass; the smooth/no-mode/no-sparsity/no-pH ablations remove the indicated mechanism. Proxy labels are deterministic functions of recorded state/action fields and are used only for evaluation. These are mechanism-isolation audits, not closed-loop task-success, contact-force-ground-truth, or real-robot claims.

\paragraph{Experiment 1: high-occlusion support-mismatch audit.}
Contact-style denotes PickCube-v1, PokeCube-v1, PullCube-v1, and PushCube-v1 under controlled occlusion. We report two support-safe operating points on the same support-mass/sharpness frontier. \method{}-adaptive uses the default validation-frozen rule and spends only the support mass selected before sampling. \method{}-conservative uses a deliberately larger fixed support mass, buying variance margin at the cost of a less sharp proposal. The conservative variant is therefore not a separate architecture or a post-hoc diagnostic; it is a second operating point predicted by Theorem~\ref{thm:support_safe_budgeted_vhydro}. The non-defensive \(\lambda=0\) / no-support baseline is matched: it uses the same architecture, particle count, objective, observation likelihood, resampling rule, training budget, and numerical stabilizers, and removes only the defensive carrier mass.

At \(90\%\) occlusion, defensive support mixing prevents the predictive collapse seen in the non-defensive filter. On Contact-style tasks, NLL drops from \(11.286\) for \(\lambda=0\) to \(0.161\) for the conservative support-safe operating point; on Sawyer/BridgeData-style diagnostics, NLL drops from \(9.401\) to \(0.074\). ESS/N, which is directly tied to the theorem's normalizer-variance frontier, rises from \(0.042\) to \(0.141\) and from \(0.046\) to \(0.151\), respectively. The adaptive operating point also avoids the large NLL collapse but remains closer to the variance frontier. Thus the claim is high-occlusion support repair, not uniform calibration dominance: Appendix~\ref{app:exp1_supportsafe_audit} shows that low-occlusion coverage can be sharper for adaptive support mixing.

\begin{figure}[t]
\centering
\begin{subfigure}[t]{0.49\linewidth}
    \centering
    \includegraphics[width=\linewidth]{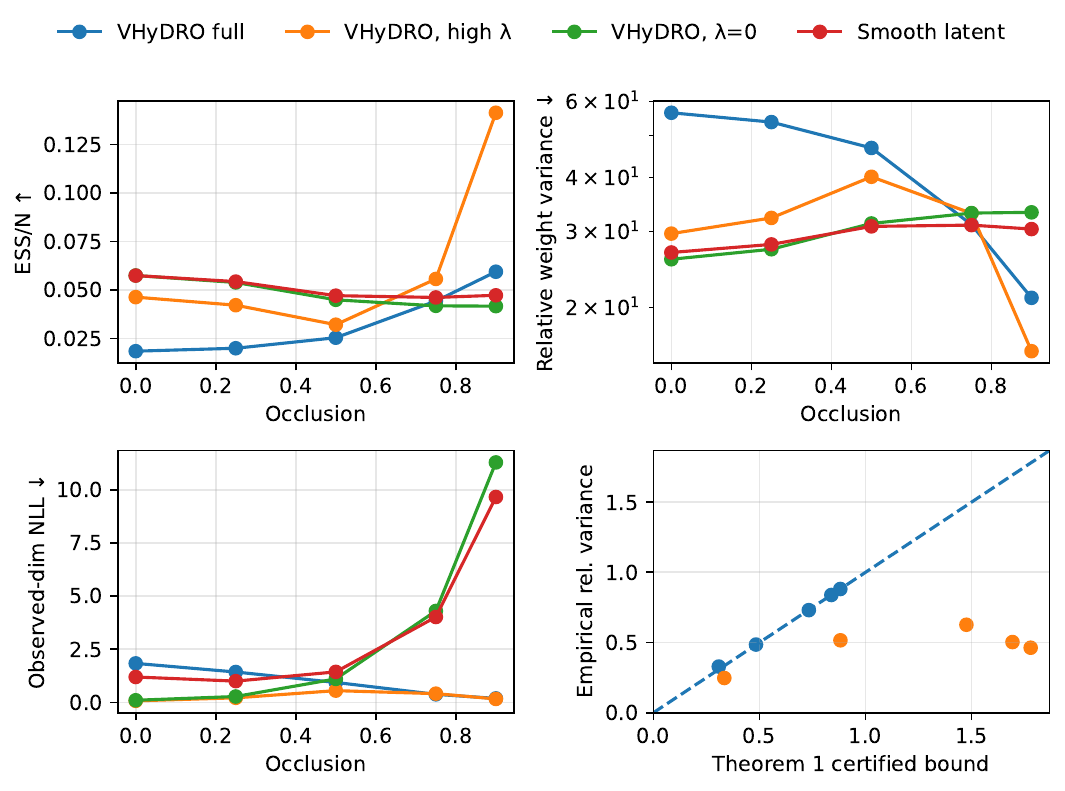}
    \caption{Contact-style family.}
\end{subfigure}\hfill
\begin{subfigure}[t]{0.49\linewidth}
    \centering
    \includegraphics[width=\linewidth]{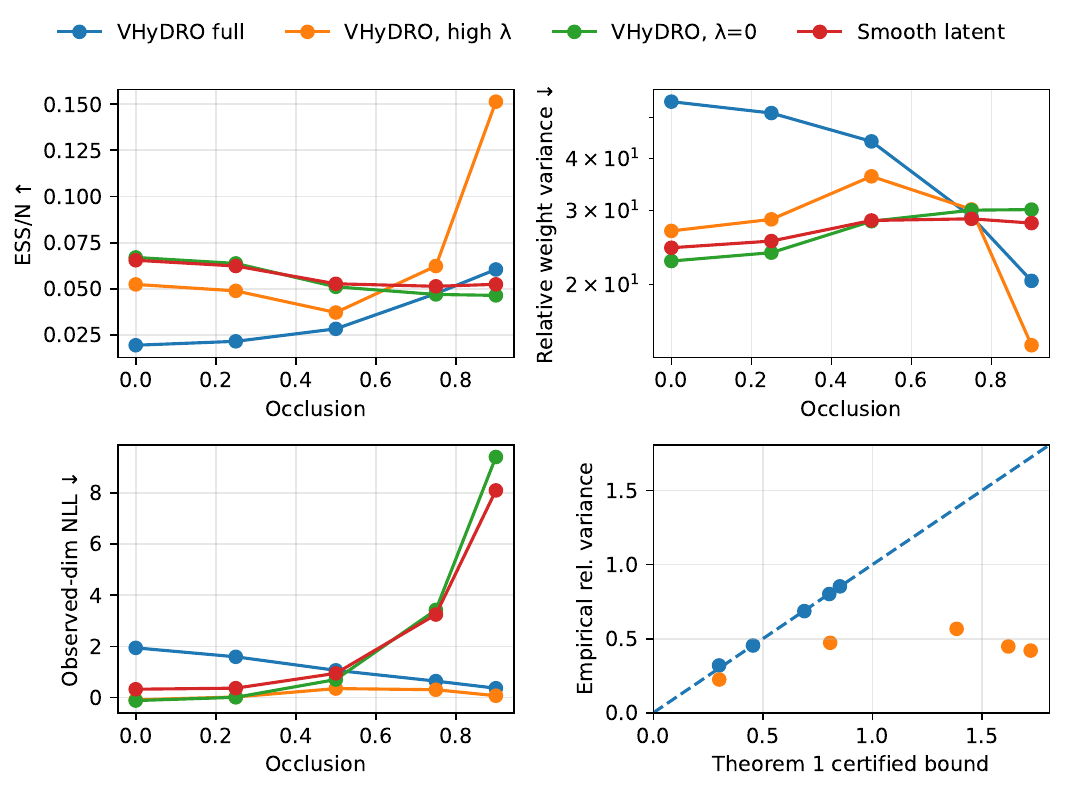}
    \caption{Sawyer/BridgeData family.}
\end{subfigure}
\caption{\textbf{Support-safe filtering under partial observation.}
Occlusion makes the one-step posterior multimodal and exposes proposal support
mismatch. Defensive support mixing increases high-occlusion ESS and lowers
weight instability. NLL is shown as a downstream collapse diagnostic; the
finite-variance certificate in Theorem~\ref{thm:support_safe_budgeted_vhydro}
directly concerns the normalizer-estimation variance and ESS frontier.}
\label{fig:exp1_certificate_pair}
\end{figure}

\begin{table}[t]
\centering
\small
\caption{\textbf{Experiment 1: high-occlusion filtering metrics.}
At \(90\%\) occlusion, support-safe proposals retain feasible contact branches.
NLL is reported as a downstream collapse diagnostic; the theorem directly
controls relative normalizer variance and the ESS frontier. Each family row
aggregates four tasks or diagnostic families with \(20\) seeds each
(\(80\) runs per method-family cell). Appendix~\ref{app:exp1_supportsafe_audit}
reports the full sweep and mean$\pm$SEM values.}
\label{tab:exp1_hard_occlusion}
\resizebox{\linewidth}{!}{%
\begin{tabular}{llrrrrrr}
\toprule
Family & Method & ESS/N $\uparrow$ & Rel. w-var $\downarrow$ & Emp. rel-var $\downarrow$ & NLL $\downarrow$ & ECE $\downarrow$ & Cov90 $\uparrow$ \\
\midrule
Contact-style & \method{}-conservative & \textbf{0.141} & \textbf{15.851} & \textbf{0.248} & \textbf{0.161} & \textbf{0.125} & \textbf{0.756} \\
Contact-style & \method{}-adaptive & 0.059 & 21.057 & 0.329 & 0.186 & 0.152 & 0.727 \\
Contact-style & Smooth latent & 0.047 & 30.366 & 0.474 & 9.660 & 0.469 & 0.356 \\
Contact-style & \method{}-no-support & 0.042 & 33.195 & 0.519 & 11.286 & 0.467 & 0.358 \\
Sawyer/BridgeData-style & \method{}-conservative & \textbf{0.151} & \textbf{14.355} & \textbf{0.224} & \textbf{0.074} & \textbf{0.111} & \textbf{0.772} \\
Sawyer/BridgeData-style & \method{}-adaptive & 0.060 & 20.432 & 0.319 & 0.369 & 0.147 & 0.732 \\
Sawyer/BridgeData-style & Smooth latent & 0.052 & 28.050 & 0.438 & 8.095 & 0.440 & 0.391 \\
Sawyer/BridgeData-style & \method{}-no-support & 0.046 & 30.205 & 0.472 & 9.401 & 0.435 & 0.397 \\
\bottomrule
\end{tabular}}
\end{table}

\paragraph{Experiment 2: contact-regime segmentation with proxy labels on manipulation demonstrations.}

\paragraph{Datasets.}
Experiment~2 uses released ManiSkill\,2 HDF5 demonstrations \citep{gu2023maniskill2} on PushCube-v1, PokeCube-v1, PullCubeTool-v1, and PegInsertionSide-v1. The main table reports this released-demonstration audit. Appendix~\ref{app:exp2_sawyer_bridge_proxy} reports the same segmentation protocol on four author-defined Sawyer/BridgeData-style diagnostic families: BridgeDataOcclusion, SawyerTowerCreation, SawyerLaundryLayout, and SawyerObjectSearch. These names denote fixed diagnostic groupings used in this paper, not canonical released BridgeData benchmark task labels. All reported segmentation metrics use deterministic denoised kinematic proxy labels computed from recorded state/action fields. These labels are frozen before held-out evaluation, are never used to supervise the mode posterior, and should not be interpreted as simulator contact-force ground truth.

This experiment uses frozen, deterministic kinematic proxy labels only as an evaluation target; they are matched to predicted modes after training and should not be read as physical contact-state ground truth. The HMM baseline is a post-hoc segmenter fit to the same kinematic features, while \method{} uses the discrete state inside the predictive filter; the no-support and no-mode ablations remove defensive mixing and discrete contact structure, respectively.

On held-out episodes, \method{} has the strongest joint profile across ARI,
change-point F1, Mode F1, and segment purity. The most reliable gains are on
ARI and change-point F1. The Mode F1 gain over HMM is positive but smaller
relative to its SEM, so we interpret the result as improved predictive
segmentation rather than uniform dominance on every proxy-label metric.

\begin{table}[t]
\centering
\small
\caption{\textbf{Experiment 2: contact-regime segmentation with proxy labels on ManiSkill HDF5 demonstrations.}
Metrics are computed on held-out test episodes across four tasks. Mode F1 uses
denoised kinematic proxy labels, not simulator contact-force ground truth.
The aggregate mean$\pm$SEM is computed over the held-out task/seed cells used in the main-table aggregation; Appendix~\ref{tab:si_exp2_per_task} reports per-task means with seed-level SEMs, so the aggregate SEM and per-task SEMs summarize different sources of variation.}
\label{tab:exp2_mode_discovery}
\resizebox{0.86\linewidth}{!}{%
\begin{tabular}{lcccc}
\toprule
Method & Mode F1 $\uparrow$ & ARI $\uparrow$ & Change-point F1 $\uparrow$ & Segment purity $\uparrow$ \\
\midrule
\method{} & \textbf{0.724$\pm$0.026} & \textbf{0.647$\pm$0.019} & \textbf{0.652$\pm$0.019} & \textbf{0.919$\pm$0.007} \\
no-support & 0.612$\pm$0.014 & 0.465$\pm$0.018 & 0.433$\pm$0.018 & 0.915$\pm$0.005 \\
HMM & 0.675$\pm$0.033 & 0.490$\pm$0.032 & 0.527$\pm$0.028 & 0.890$\pm$0.005 \\
no-mode & 0.221$\pm$0.002 & 0.000$\pm$0.000 & 0.000$\pm$0.000 & 0.496$\pm$0.006 \\
\bottomrule
\end{tabular}}
\end{table}

\paragraph{Experiment 3: sparse physical-law recovery.}
\label{sec:exp3_sparse_ph}

We next test a controlled identifiability question: given sufficiently
mode-pure filtered states, can the sparse port-Hamiltonian library recover the
active physical terms? Because the manipulation demonstration datasets do not provide ground-truth port-Hamiltonian coefficients, Experiment~3 uses a known-equation audit over
four controlled hybrid systems. It is not an additional manipulation-demonstration dataset and is not evidence of broad real-robot system identification. The candidate library
has at most \(p=12\) terms per system and the true active support has size at
most \(k=3\). We compare \method{} with no-mode, no-sparsity, and no-pH
ablations to separate accurate prediction from accurate mechanism recovery.

All sparse-regression variants use the same candidate library, normalization,
state and derivative estimates, validation split, and regularization grid. No
method is given the true active support. Thus Table~\ref{tab:exp3_sparse_ph_main}
tests whether mode-conditioned segments make the shared sparse-recovery problem
easier, not whether \method{} receives more favorable sparse-regression tuning.

\begin{table}[H]
\centering
\small
\caption{\textbf{Experiment 3: sparse physical-law recovery.}
Values are mean$\pm$standard deviation over the four known-equation systems,
twenty seeds per system, and the fixed perturbation conditions described in
Appendix~\ref{app:exp3_controlled_sparse_ph}. Dashes denote metrics that are
undefined because the no-pH baseline has no port-Hamiltonian coefficient
vector.}
\label{tab:exp3_sparse_ph_main}
\resizebox{\linewidth}{!}{%
\begin{tabular}{lcccc}
\toprule
Method & Support F1 $\uparrow$ & Rel. coeff. error $\downarrow$ & Vector-field NRMSE $\downarrow$ & Phys.-constant error $\downarrow$ \\
\midrule
\method{}-full & \textbf{$0.978\pm0.042$} & \textbf{$0.006\pm0.010$} & \textbf{$0.001\pm0.001$} & \textbf{$0.001\pm0.001$} \\
no-mode & $0.713\pm0.039$ & $1.087\pm0.330$ & $0.287\pm0.196$ & $0.188\pm0.234$ \\
no-sparsity & $0.771\pm0.073$ & $0.851\pm0.559$ & $0.043\pm0.043$ & $0.014\pm0.008$ \\
no-pH & -- & -- & $0.031\pm0.015$ & -- \\
\bottomrule
\end{tabular}}
\end{table}

Table~\ref{tab:exp3_sparse_ph_main} reports the controlled known-equation audit;
predictive baselines do not provide the same sparse mechanism recovery.

\section{Conclusion and Limitations}

Amortized variational filters can catastrophically delete feasible contact
branches under partial observation. \method{} prevents this by using a
model-relative defensive transition carrier with a pre-sampling variance budget.
This support repair stabilizes filtering, and the stabilized filter yields
coherent predictive modes that can support downstream sparse-law recovery.
Technically, the defensive mixture ensures \(P_t \ll Q_{t,\lambda}\) and bounds
the transition-to-proposal density ratio for transition sets retained by the
feasible carrier \(P_t\). Empirically, the paper tests the three links of this
mechanism chain: high-occlusion IW/FIVO diagnostics, proxy-label
contact-regime segmentation, and known-equation sparse-law recovery.

\paragraph{Limitations.}
The guarantees are model-relative: if the carrier \(P_t\) omits a real physical branch, \method{} cannot recover it. The mode set and sparse library are fixed in the present experiments. The action-cancellation assumption in Theorem~\ref{thm:mode_recovery_vhydro} matches logged interventions or mode-independent behavior policies; mode-dependent closed-loop policies require a different likelihood-ratio statement. The experiments are matched mechanism audits, not claims of closed-loop task-success dominance, external SOTA ranking, simulator contact-force ground-truth recovery, real-robot transfer, or real-robot sparse-law identification. A natural follow-up is closed-loop MPC with a held-out predictive KL or calibration certificate, which would let recovered mode-wise laws support chance-constrained planning rather than only offline mechanism audits.

\bibliographystyle{plainnat}
\bibliography{mybib}

\appendix

\providecommand{\method}{VHyDRO}
\providecommand{\sihighlight}[1]{\textbf{#1}}
\providecommand{\figdir}{figures}
\providecommand{\tabdir}{tables}

\section{Supplementary material for experiments}
\label{app:experiments_1_2}

This appendix makes the measured experimental claims auditable.
Experiment~1 reports the full occlusion sweep, certificate plots, calibration diagnostics, and hard-occlusion tables.
Experiment~2 reports proxy-label construction, validation-only hyperparameter selection, per-task mode metrics, and Sawyer/BridgeData diagnostics.
Experiment~3 reports sparse-law recovery diagnostics.

\paragraph{Seed and data convention.}
All experiment-level summaries use \(20\) independent random seeds. No table
is computed from fewer than \(20\) seeds, a seed-pruned subset, or a
subsampled task family. The word ``sample'' in algorithms and proofs refers to
internal Monte Carlo particle draws from the IW/FIVO proposal, not to dataset
subsampling. Proxy labels in Experiment~2 and the Sawyer/BridgeData diagnostic
are deterministic functions of recorded state/action fields and are used only
for validation diagnostics and held-out evaluation.

The supplementary proofs give the full derivations for
Theorems~\ref{thm:support_safe_budgeted_vhydro}--\ref{thm:sparse_ph_recovery},
which formalize the measured filtering, mode-concentration, and sparse-recovery
links.

\subsection{Experiment 1: support-safe filtering audit}
\label{app:exp1_supportsafe_audit}

\paragraph{Protocol.}
All methods use the same architecture family, latent dimension, training/evaluation budget, number of particles, occlusion levels $\{0,0.25,0.50,0.75,0.90\}$, and the same \(20\) random seeds. The variants differ only in the proposal-support mechanism: adaptive support mixing, $\lambda=0$, high fixed $\lambda$, and a smooth single-mode latent baseline. We report ESS/N, relative weight variance, empirical estimator relative variance, NLL, ECE, and 90\% coverage at every occlusion level. Bold table entries mark the best value within each occlusion or task block.

\paragraph{Why the comparisons are mechanistic.}
Table~\ref{tab:si_comparison_scope} summarizes what each comparison removes.
The goal is to isolate the role of support preservation, hybrid modes, and port-Hamiltonian structure, not to claim broad task-success dominance over all latent-dynamics model families.

\paragraph{Mechanism-isolation controls.}
External latent-dynamics and contact-dynamics model families differ in
architecture, feature extraction, filtering objective, and training protocol.
For that reason, the main comparisons use matched ablations that change one
mechanism at a time while holding the remaining implementation fixed. This
design isolates proposal-support repair; it should not be read as a broad
architecture-ranking benchmark.

\begin{table}[h]
\centering
\small
\caption{\textbf{Mechanistic comparison scope.}
Each comparison removes one component of the proposed chain.}
\label{tab:si_comparison_scope}
\begin{tabular}{lll}
\toprule
Comparison & Removed mechanism & Tested claim \\
\midrule
\(\lambda=0\) & Support preservation & Support mismatch causes IW/FIVO instability \\
High-\(\lambda\) & Proposal sharpness & Conservative support mass gives variance margin \\
Smooth latent & Hybrid modes and support mixing & A single smooth state cannot preserve contact branches \\
No support & Defensive proposal & Modes alone are insufficient under support mismatch \\
No mode & Discrete contact state & Predictive regime segmentation uses a hybrid state \\
No sparsity & Sparse support selection & Dense pH fits obscure active laws \\
No-pH & Interpretable pH coefficients & Prediction is not mechanism recovery \\
HMM & Predictive latent dynamics & Post-hoc segmentation is not a predictive filtering model. \\
\bottomrule
\end{tabular}
\end{table}

\begin{table}[t]
\centering
\small
\caption{\textbf{Claim boundary.} The paper evaluates intermediate mechanisms, not broad task-success dominance.}
\label{tab:claim_boundary}
\resizebox{\linewidth}{!}{
\begin{tabular}{lll}
\toprule
Claim tested & Evidence reported & Not claimed \\
\midrule
Support-safe filtering controls one-step IW/FIVO weight diagnostics under
occlusion & ESS/N, relative variance, NLL, coverage audits & General calibration dominance \\
Hybrid state forms coherent contact-regime segments & ARI, change-point F1, Mode F1 vs. kinematic proxy labels & Simulator contact-force ground truth \\
Mode-pure segments enable sparse-law recovery & Controlled systems with known sparse laws & Real-robot system identification \\
\bottomrule
\end{tabular}}
\end{table}

\paragraph{Certificate diagnostic.}
Figure~\ref{fig:si_exp1_theorem_certificate} plots empirical estimator relative variance against the certified support-safe bound. The high-$\lambda$ proposal stays furthest below the diagonal, giving the most conservative certificate. The adaptive proposal tracks the bound more tightly, using less defensive mass while preserving low likelihood error in the main paper. This is the expected tradeoff in Theorem~\ref{thm:support_safe_budgeted_vhydro}: larger support mass buys variance margin; adaptive support mass buys sharper filtering.

\begin{figure}[t]
\centering
\begin{subfigure}[t]{0.49\linewidth}
    \centering
    \includegraphics[width=\linewidth]{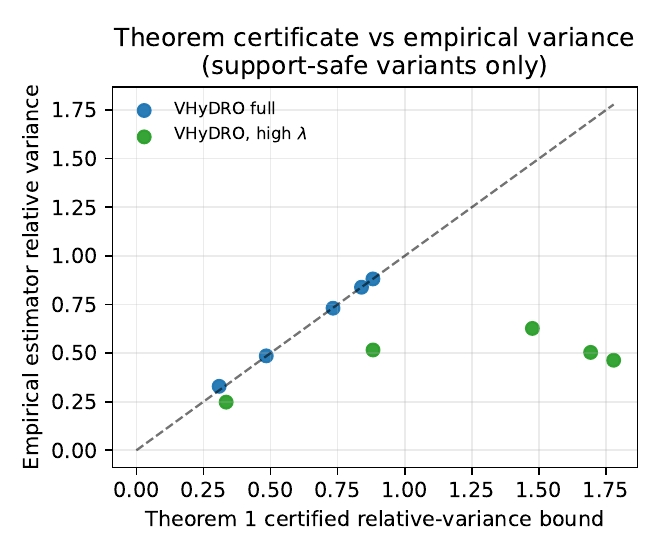}
    \caption{Contact-style family.}
\end{subfigure}\hfill
\begin{subfigure}[t]{0.49\linewidth}
    \centering
    \includegraphics[width=\linewidth]{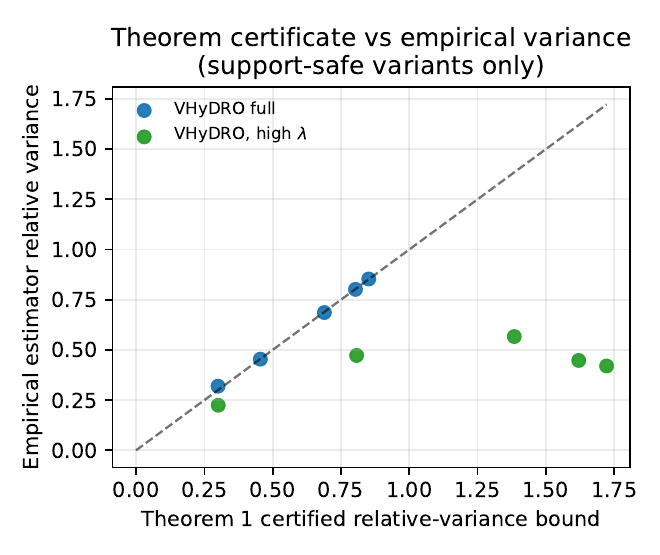}
    \caption{Sawyer/BridgeData task family.}
\end{subfigure}
\caption{\textbf{Support-safe variance certificate.} Points below the diagonal satisfy the empirical certificate margin. The high-$\lambda$ variant is deliberately conservative; the adaptive variant stays closer to the certified frontier.}
\label{fig:si_exp1_theorem_certificate}
\end{figure}

\paragraph{Full occlusion diagnostics.}
Figure~\ref{fig:si_exp1_full_occlusion_diagnostics} exposes where support mismatch enters. At easy occlusion, non-defensive proposals can be competitive. At high occlusion, the IW/FIVO log-normalizer variance and calibration curves separate sharply: support-safe variants retain usable coverage, while $\lambda=0$ and the smooth latent baseline lose calibration as the hidden contact branch becomes ambiguous.

\paragraph{Low-occlusion calibration audit.}
The adaptive support-safe variant is not uniformly calibrated across the sweep.
At 0--25\% occlusion, it under-covers: Cov90 is \(0.162\)--\(0.222\) on the Contact-style family and \(0.198\)--\(0.266\) on the Sawyer/BridgeData task family.
This failure occurs in the easy-observation regime where the adaptive proposal remains sharp. The high-\(\lambda\) variant preserves broader uncertainty and gives higher coverage, but at the cost of a less targeted proposal.
Accordingly, the paper claims robustness to support ambiguity at high occlusion, not uniform calibration dominance.

\begin{figure}[p]
\centering
\begin{subfigure}[t]{0.49\linewidth}
    \includegraphics[width=\linewidth]{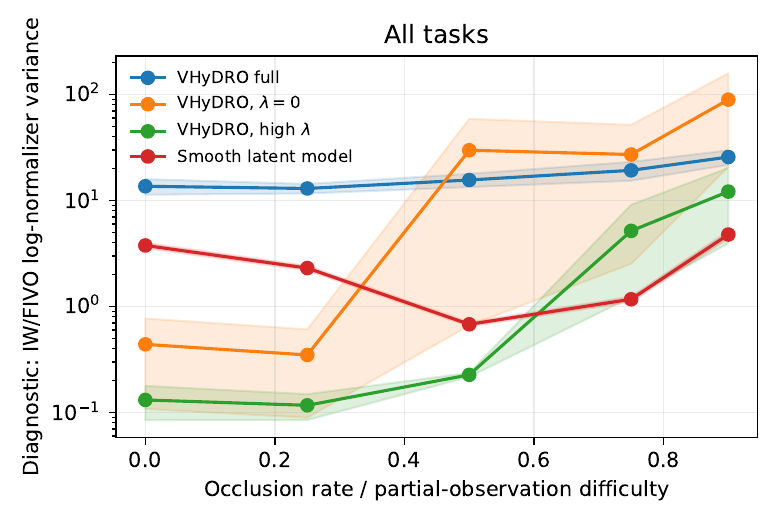}
    \caption{Contact-style: IW log-normalizer variance.}
\end{subfigure}\hfill
\begin{subfigure}[t]{0.49\linewidth}
    \includegraphics[width=\linewidth]{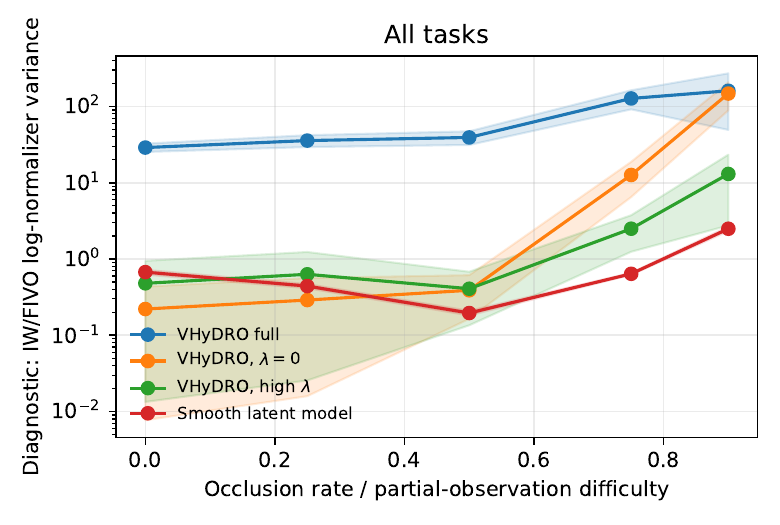}
    \caption{Sawyer/BridgeData task families: IW log-normalizer variance.}
\end{subfigure}
\begin{subfigure}[t]{0.49\linewidth}
    \includegraphics[width=\linewidth]{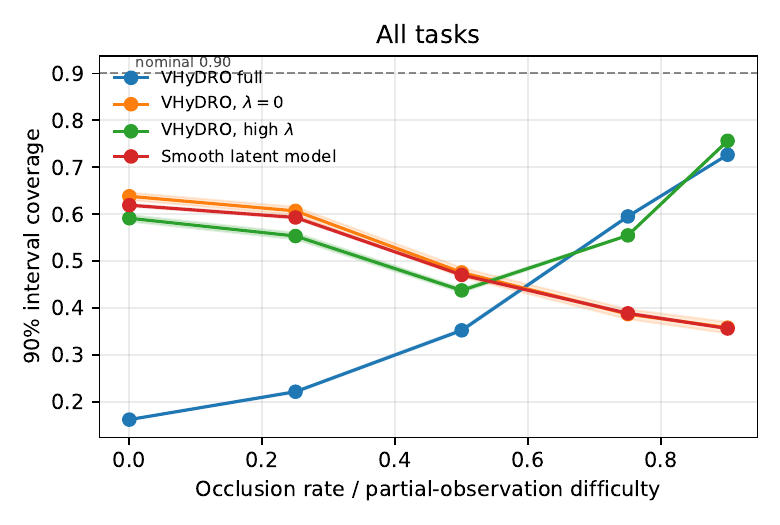}
    \caption{Contact-style: 90\% coverage.}
\end{subfigure}\hfill
\begin{subfigure}[t]{0.49\linewidth}
    \includegraphics[width=\linewidth]{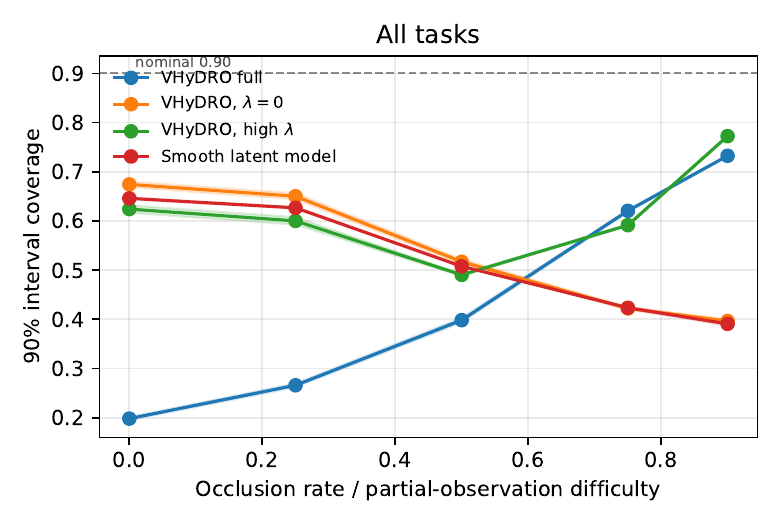}
    \caption{Sawyer/BridgeData task families: 90\% coverage.}
\end{subfigure}
\begin{subfigure}[t]{0.49\linewidth}
    \includegraphics[width=\linewidth]{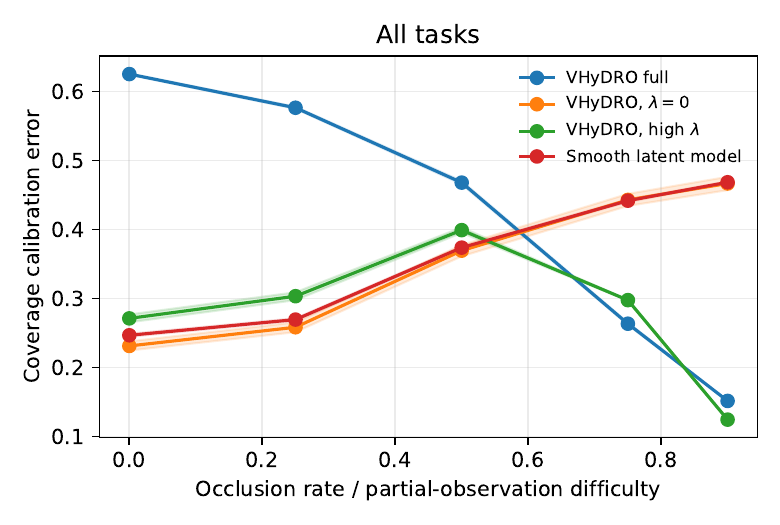}
    \caption{Contact-style: ECE.}
\end{subfigure}\hfill
\begin{subfigure}[t]{0.49\linewidth}
    \includegraphics[width=\linewidth]{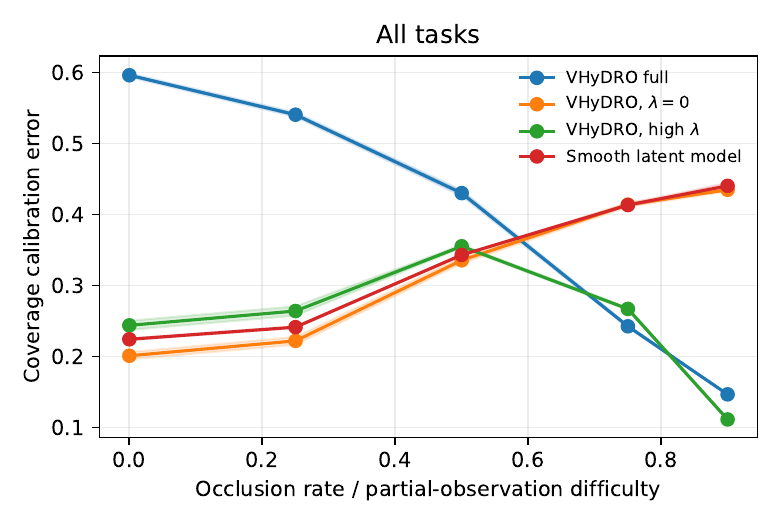}
    \caption{Sawyer/BridgeData task families: ECE.}
\end{subfigure}
\caption{\textbf{Full occlusion diagnostics.} The informative regime is the high-occlusion end of the sweep. Support-safe proposals keep uncertainty usable where non-defensive and smooth proposals under-cover and incur unstable normalizer estimates.}
\label{fig:si_exp1_full_occlusion_diagnostics}
\end{figure}

\paragraph{\(\lambda\)-selection audit.}
The adaptive setting chooses \(\lambda\) before sampling and uses the same
realized \(Q_{t,\lambda}\) in the importance weight. This matches the
condition in Theorem~\ref{thm:support_safe_budgeted_vhydro} and avoids
post-hoc adaptation of the proposal density. When the requested budget is not
certified by the validation-frozen rule, the method uses the fixed fallback
\(\lambda_{\rm fb}\) from the filtering step; such cases are treated as
uncertified operating points. The high-\(\lambda\) variant is included only to
show the variance margin obtained when the support budget is made deliberately
conservative.

\paragraph{Full numeric audit.}
Tables~\ref{tab:si_exp1_contact_all_occlusions}--\ref{tab:si_exp1_sawyer_bridge_hard_by_task} give the complete sweep and per-task hard-occlusion results. The pattern is deliberately not claimed as uniform dominance at low occlusion. The evidence appears when observations hide the contact branch: at 90\% occlusion, support-safe filtering increases ESS/N, reduces empirical relative variance, and prevents the NLL and coverage collapse seen for $\lambda=0$ and the smooth latent baseline.

\begin{table}[p]
\centering
\scriptsize
\caption{\textbf{Experiment 1A: contact-style occlusion sweep.} Full occlusion sweep for the contact-style family. Bold entries mark the best value within each occlusion block and metric. Means$\pm$SEM over twenty seeds.}
\label{tab:si_exp1_contact_all_occlusions}
\resizebox{\linewidth}{!}{%
\begin{tabular}{llcccccc}
\toprule
Occ. & Method & ESS/N $\uparrow$ & Rel. w-var $\downarrow$ & Emp. rel-var $\downarrow$ & NLL $\downarrow$ & ECE $\downarrow$ & Cov90 $\uparrow$ \\
\midrule
0.00 & \method{} full & 0.018$\pm$0.000 & 56.407$\pm$0.060 & 0.881$\pm$0.001 & 1.833$\pm$0.119 & 0.625$\pm$0.001 & 0.162$\pm$0.001 \\
0.00 & \method{}, high $\lambda$ & 0.046$\pm$0.001 & 29.642$\pm$0.544 & 0.463$\pm$0.009 & \sihighlight{0.071$\pm$0.014} & 0.271$\pm$0.006 & 0.591$\pm$0.007 \\
0.00 & \method{}, $\lambda=0$ & \sihighlight{0.057$\pm$0.001} & \sihighlight{25.858$\pm$0.594} & \sihighlight{0.404$\pm$0.009} & 0.098$\pm$0.066 & \sihighlight{0.231$\pm$0.007} & \sihighlight{0.638$\pm$0.008} \\
0.00 & Smooth latent & 0.057$\pm$0.001 & 26.808$\pm$0.240 & 0.419$\pm$0.004 & 1.193$\pm$0.064 & 0.247$\pm$0.002 & 0.619$\pm$0.003 \\
0.25 & \method{} full & 0.020$\pm$0.000 & 53.675$\pm$0.088 & 0.839$\pm$0.001 & 1.427$\pm$0.098 & 0.577$\pm$0.002 & 0.222$\pm$0.002 \\
0.25 & \method{}, high $\lambda$ & 0.042$\pm$0.001 & 32.223$\pm$0.557 & 0.503$\pm$0.009 & \sihighlight{0.213$\pm$0.013} & 0.303$\pm$0.006 & 0.553$\pm$0.007 \\
0.25 & \method{}, $\lambda=0$ & 0.054$\pm$0.001 & \sihighlight{27.275$\pm$0.658} & \sihighlight{0.426$\pm$0.010} & 0.275$\pm$0.070 & \sihighlight{0.258$\pm$0.007} & \sihighlight{0.607$\pm$0.009} \\
0.25 & Smooth latent & \sihighlight{0.054$\pm$0.001} & 27.986$\pm$0.264 & 0.437$\pm$0.004 & 0.999$\pm$0.039 & 0.270$\pm$0.002 & 0.593$\pm$0.002 \\
0.50 & \method{} full & 0.025$\pm$0.000 & 46.771$\pm$0.167 & 0.731$\pm$0.003 & 0.936$\pm$0.068 & 0.468$\pm$0.002 & 0.352$\pm$0.003 \\
0.50 & \method{}, high $\lambda$ & 0.032$\pm$0.001 & 40.113$\pm$0.426 & 0.627$\pm$0.007 & \sihighlight{0.551$\pm$0.014} & 0.399$\pm$0.004 & 0.437$\pm$0.005 \\
0.50 & \method{}, $\lambda=0$ & 0.045$\pm$0.001 & 31.287$\pm$0.752 & 0.489$\pm$0.012 & 1.103$\pm$0.062 & \sihighlight{0.369$\pm$0.008} & \sihighlight{0.476$\pm$0.010} \\
0.50 & Smooth latent & \sihighlight{0.047$\pm$0.000} & \sihighlight{30.804$\pm$0.259} & \sihighlight{0.481$\pm$0.004} & 1.434$\pm$0.023 & 0.374$\pm$0.002 & 0.470$\pm$0.002 \\
0.75 & \method{} full & 0.044$\pm$0.000 & 31.086$\pm$0.142 & 0.486$\pm$0.002 & \sihighlight{0.377$\pm$0.027} & \sihighlight{0.264$\pm$0.002} & \sihighlight{0.595$\pm$0.002} \\
0.75 & \method{}, high $\lambda$ & \sihighlight{0.056$\pm$0.000} & 33.022$\pm$0.058 & 0.516$\pm$0.001 & 0.410$\pm$0.020 & 0.298$\pm$0.001 & 0.555$\pm$0.001 \\
0.75 & \method{}, $\lambda=0$ & 0.042$\pm$0.002 & 33.068$\pm$0.957 & 0.517$\pm$0.015 & 4.299$\pm$0.069 & 0.443$\pm$0.009 & 0.387$\pm$0.011 \\
0.75 & Smooth latent & 0.046$\pm$0.000 & \sihighlight{31.015$\pm$0.211} & \sihighlight{0.485$\pm$0.003} & 4.016$\pm$0.030 & 0.442$\pm$0.002 & 0.389$\pm$0.002 \\
0.90 & \method{} full & 0.059$\pm$0.000 & 21.057$\pm$0.020 & 0.329$\pm$0.000 & 0.186$\pm$0.019 & 0.152$\pm$0.001 & 0.727$\pm$0.001 \\
0.90 & \method{}, high $\lambda$ & \sihighlight{0.141$\pm$0.001} & \sihighlight{15.851$\pm$0.107} & \sihighlight{0.248$\pm$0.002} & \sihighlight{0.161$\pm$0.023} & \sihighlight{0.125$\pm$0.001} & \sihighlight{0.756$\pm$0.001} \\
0.90 & \method{}, $\lambda=0$ & 0.042$\pm$0.002 & 33.195$\pm$1.128 & 0.519$\pm$0.018 & 11.286$\pm$0.151 & 0.467$\pm$0.010 & 0.358$\pm$0.012 \\
0.90 & Smooth latent & 0.047$\pm$0.001 & 30.366$\pm$0.290 & 0.474$\pm$0.005 & 9.660$\pm$0.098 & 0.469$\pm$0.001 & 0.356$\pm$0.001 \\
\bottomrule
\end{tabular}}
\end{table}

\begin{table}[p]
\centering
\scriptsize
\caption{\textbf{Experiment 1B: Sawyer/Bridge occlusion sweep.} Full occlusion sweep for the Sawyer/Bridge family. Bold entries mark the best value within each occlusion block and metric. Means$\pm$SEM over twenty seeds.}
\label{tab:si_exp1_sawyer_bridge_all_occlusions}
\resizebox{\linewidth}{!}{%
\begin{tabular}{llcccccc}
\toprule
Occ. & Method & ESS/N $\uparrow$ & Rel. w-var $\downarrow$ & Emp. rel-var $\downarrow$ & NLL $\downarrow$ & ECE $\downarrow$ & Cov90 $\uparrow$ \\
\midrule
0.00 & \method{} full & 0.019$\pm$0.000 & 54.593$\pm$0.200 & 0.853$\pm$0.003 & 1.945$\pm$0.065 & 0.596$\pm$0.003 & 0.198$\pm$0.003 \\
0.00 & \method{}, high $\lambda$ & 0.052$\pm$0.001 & 26.870$\pm$0.594 & 0.420$\pm$0.009 & -0.086$\pm$0.028 & 0.244$\pm$0.006 & 0.624$\pm$0.008 \\
0.00 & \method{}, $\lambda=0$ & \sihighlight{0.067$\pm$0.002} & \sihighlight{22.779$\pm$0.494} & \sihighlight{0.356$\pm$0.008} & \sihighlight{-0.115$\pm$0.017} & \sihighlight{0.201$\pm$0.005} & \sihighlight{0.674$\pm$0.006} \\
0.00 & Smooth latent & 0.065$\pm$0.000 & 24.502$\pm$0.143 & 0.383$\pm$0.002 & 0.328$\pm$0.028 & 0.224$\pm$0.001 & 0.646$\pm$0.002 \\
0.25 & \method{} full & 0.022$\pm$0.000 & 51.286$\pm$0.282 & 0.801$\pm$0.004 & 1.595$\pm$0.064 & 0.540$\pm$0.003 & 0.266$\pm$0.004 \\
0.25 & \method{}, high $\lambda$ & 0.049$\pm$0.001 & 28.631$\pm$0.626 & 0.447$\pm$0.010 & 0.026$\pm$0.028 & 0.264$\pm$0.007 & 0.600$\pm$0.008 \\
0.25 & \method{}, $\lambda=0$ & \sihighlight{0.064$\pm$0.002} & \sihighlight{23.852$\pm$0.547} & \sihighlight{0.373$\pm$0.009} & \sihighlight{0.013$\pm$0.024} & \sihighlight{0.222$\pm$0.005} & \sihighlight{0.650$\pm$0.006} \\
0.25 & Smooth latent & 0.062$\pm$0.000 & 25.418$\pm$0.136 & 0.397$\pm$0.002 & 0.370$\pm$0.025 & 0.241$\pm$0.001 & 0.627$\pm$0.001 \\
0.50 & \method{} full & 0.028$\pm$0.000 & 43.903$\pm$0.319 & 0.686$\pm$0.005 & 1.067$\pm$0.052 & 0.430$\pm$0.004 & 0.398$\pm$0.005 \\
0.50 & \method{}, high $\lambda$ & 0.037$\pm$0.001 & 36.249$\pm$0.389 & 0.566$\pm$0.006 & \sihighlight{0.354$\pm$0.029} & 0.355$\pm$0.002 & 0.490$\pm$0.003 \\
0.50 & \method{}, $\lambda=0$ & 0.051$\pm$0.001 & \sihighlight{28.336$\pm$0.535} & \sihighlight{0.443$\pm$0.008} & 0.707$\pm$0.029 & \sihighlight{0.335$\pm$0.004} & \sihighlight{0.517$\pm$0.005} \\
0.50 & Smooth latent & \sihighlight{0.053$\pm$0.000} & 28.455$\pm$0.161 & 0.445$\pm$0.003 & 0.949$\pm$0.019 & 0.343$\pm$0.001 & 0.508$\pm$0.001 \\
0.75 & \method{} full & 0.047$\pm$0.000 & 29.046$\pm$0.233 & 0.454$\pm$0.004 & 0.646$\pm$0.123 & \sihighlight{0.242$\pm$0.002} & \sihighlight{0.621$\pm$0.002} \\
0.75 & \method{}, high $\lambda$ & \sihighlight{0.062$\pm$0.001} & 30.229$\pm$0.163 & 0.472$\pm$0.003 & \sihighlight{0.309$\pm$0.020} & 0.267$\pm$0.002 & 0.591$\pm$0.002 \\
0.75 & \method{}, $\lambda=0$ & 0.047$\pm$0.001 & 30.110$\pm$0.511 & 0.470$\pm$0.008 & 3.426$\pm$0.054 & 0.414$\pm$0.003 & 0.422$\pm$0.003 \\
0.75 & Smooth latent & 0.051$\pm$0.000 & \sihighlight{28.715$\pm$0.037} & \sihighlight{0.449$\pm$0.001} & 3.246$\pm$0.027 & 0.413$\pm$0.002 & 0.423$\pm$0.002 \\
0.90 & \method{} full & 0.060$\pm$0.000 & 20.432$\pm$0.043 & 0.319$\pm$0.001 & 0.369$\pm$0.097 & 0.147$\pm$0.000 & 0.732$\pm$0.001 \\
0.90 & \method{}, high $\lambda$ & \sihighlight{0.151$\pm$0.002} & \sihighlight{14.355$\pm$0.071} & \sihighlight{0.224$\pm$0.001} & \sihighlight{0.074$\pm$0.019} & \sihighlight{0.111$\pm$0.001} & \sihighlight{0.772$\pm$0.001} \\
0.90 & \method{}, $\lambda=0$ & 0.046$\pm$0.001 & 30.205$\pm$0.325 & 0.472$\pm$0.005 & 9.401$\pm$0.067 & 0.435$\pm$0.002 & 0.397$\pm$0.003 \\
0.90 & Smooth latent & 0.052$\pm$0.000 & 28.050$\pm$0.120 & 0.438$\pm$0.002 & 8.095$\pm$0.016 & 0.440$\pm$0.004 & 0.391$\pm$0.004 \\
\bottomrule
\end{tabular}}
\end{table}

\begin{table}[p]
\centering
\scriptsize
\caption{\textbf{Experiment 1A: contact-style hard-occlusion audit.} Per-task hard-occlusion results for the contact-style family at 90\% occlusion. Bold entries mark the best value within each task block and metric. Means over twenty seeds.}
\label{tab:si_exp1_contact_hard_by_task}
\resizebox{\linewidth}{!}{%
\begin{tabular}{llrrrrrr}
\toprule
Task & Method & ESS/N $\uparrow$ & Rel. w-var $\downarrow$ & Emp. rel-var $\downarrow$ & NLL $\downarrow$ & ECE $\downarrow$ & Cov90 $\uparrow$ \\
\midrule
PickCube-v1 & \method{} full & 0.059 & 21.140 & 0.330 & \sihighlight{0.162} & 0.154 & 0.724 \\
PickCube-v1 & \method{}, high $\lambda$ & \sihighlight{0.141} & \sihighlight{15.883} & \sihighlight{0.248} & 0.167 & \sihighlight{0.125} & \sihighlight{0.756} \\
PickCube-v1 & \method{}, $\lambda=0$ & 0.041 & 33.337 & 0.521 & 11.444 & 0.468 & 0.357 \\
PickCube-v1 & Smooth latent & 0.047 & 30.444 & 0.476 & 9.712 & 0.469 & 0.356 \\
PokeCube-v1 & \method{} full & 0.060 & 21.027 & 0.329 & 0.168 & 0.150 & 0.728 \\
PokeCube-v1 & \method{}, high $\lambda$ & \sihighlight{0.141} & \sihighlight{15.881} & \sihighlight{0.248} & \sihighlight{0.148} & \sihighlight{0.125} & \sihighlight{0.756} \\
PokeCube-v1 & \method{}, $\lambda=0$ & 0.042 & 33.164 & 0.518 & 11.189 & 0.467 & 0.359 \\
PokeCube-v1 & Smooth latent & 0.047 & 30.374 & 0.475 & 9.644 & 0.469 & 0.356 \\
PullCube-v1 & \method{} full & 0.059 & 21.087 & 0.329 & \sihighlight{0.158} & 0.152 & 0.727 \\
PullCube-v1 & \method{}, high $\lambda$ & \sihighlight{0.142} & \sihighlight{15.763} & \sihighlight{0.246} & 0.161 & \sihighlight{0.125} & \sihighlight{0.756} \\
PullCube-v1 & \method{}, $\lambda=0$ & 0.042 & 33.179 & 0.518 & 11.262 & 0.466 & 0.358 \\
PullCube-v1 & Smooth latent & 0.047 & 30.345 & 0.474 & 9.620 & 0.469 & 0.356 \\
PushCube-v1 & \method{} full & 0.060 & 20.971 & 0.328 & 0.255 & 0.151 & 0.727 \\
PushCube-v1 & \method{}, high $\lambda$ & \sihighlight{0.141} & \sihighlight{15.874} & \sihighlight{0.248} & \sihighlight{0.167} & \sihighlight{0.124} & \sihighlight{0.758} \\
PushCube-v1 & \method{}, $\lambda=0$ & 0.042 & 33.106 & 0.517 & 11.246 & 0.467 & 0.358 \\
PushCube-v1 & Smooth latent & 0.048 & 30.292 & 0.473 & 9.662 & 0.468 & 0.356 \\
\bottomrule
\end{tabular}}
\end{table}

\begin{table}[p]
\centering
\scriptsize
\caption{\textbf{Experiment 1B: Sawyer/Bridge hard-occlusion audit.} Per-task hard-occlusion results for the Sawyer/Bridge family at 90\% occlusion. Bold entries mark the best value within each task block and metric. Means over twenty seeds.}
\label{tab:si_exp1_sawyer_bridge_hard_by_task}
\resizebox{\linewidth}{!}{%
\begin{tabular}{llrrrrrr}
\toprule
Task & Method & ESS/N $\uparrow$ & Rel. w-var $\downarrow$ & Emp. rel-var $\downarrow$ & NLL $\downarrow$ & ECE $\downarrow$ & Cov90 $\uparrow$ \\
\midrule
BridgeDataOcclusion & \method{} full & 0.058 & 21.814 & 0.341 & 0.351 & 0.156 & 0.719 \\
BridgeDataOcclusion & \method{}, high $\lambda$ & \sihighlight{0.130} & \sihighlight{17.182} & \sihighlight{0.268} & \sihighlight{0.312} & \sihighlight{0.138} & \sihighlight{0.743} \\
BridgeDataOcclusion & \method{}, $\lambda=0$ & 0.046 & 30.752 & 0.481 & 10.950 & 0.448 & 0.381 \\
BridgeDataOcclusion & Smooth latent & 0.051 & 28.568 & 0.446 & 9.765 & 0.452 & 0.376 \\
SawyerLaundryLayout & \method{} full & 0.061 & 20.154 & 0.315 & 0.297 & 0.144 & 0.736 \\
SawyerLaundryLayout & \method{}, high $\lambda$ & \sihighlight{0.152} & \sihighlight{14.085} & \sihighlight{0.220} & \sihighlight{0.050} & \sihighlight{0.107} & \sihighlight{0.777} \\
SawyerLaundryLayout & \method{}, $\lambda=0$ & 0.048 & 29.461 & 0.460 & 8.683 & 0.425 & 0.408 \\
SawyerLaundryLayout & Smooth latent & 0.054 & 27.525 & 0.430 & 7.489 & 0.436 & 0.396 \\
SawyerObjectSearch & \method{} full & 0.060 & 20.577 & 0.322 & 0.308 & 0.145 & 0.734 \\
SawyerObjectSearch & \method{}, high $\lambda$ & \sihighlight{0.143} & \sihighlight{15.027} & \sihighlight{0.235} & \sihighlight{0.109} & \sihighlight{0.116} & \sihighlight{0.766} \\
SawyerObjectSearch & \method{}, $\lambda=0$ & 0.046 & 30.422 & 0.475 & 9.918 & 0.440 & 0.390 \\
SawyerObjectSearch & Smooth latent & 0.052 & 28.425 & 0.444 & 8.752 & 0.447 & 0.383 \\
SawyerTowerCreation & \method{} full & 0.062 & 19.372 & 0.303 & 0.527 & 0.143 & 0.739 \\
SawyerTowerCreation & \method{}, high $\lambda$ & \sihighlight{0.177} & \sihighlight{11.585} & \sihighlight{0.181} & \sihighlight{-0.141} & \sihighlight{0.089} & \sihighlight{0.798} \\
SawyerTowerCreation & \method{}, $\lambda=0$ & 0.046 & 30.244 & 0.473 & 8.253 & 0.429 & 0.404 \\
SawyerTowerCreation & Smooth latent & 0.053 & 27.689 & 0.433 & 6.598 & 0.427 & 0.406 \\
\bottomrule
\end{tabular}}
\end{table}

\FloatBarrier

\subsection{Experiment 2: ManiSkill contact-regime segmentation audit with proxy labels}
\label{app:exp2_maniskill_proxy_audit}

\paragraph{Protocol.}
Experiment~2 uses released ManiSkill HDF5 demonstrations from PushCube-v1, PokeCube-v1, PullCubeTool-v1, and PegInsertionSide-v1. The evaluation labels are denoised kinematic proxies derived from raw HDF5 states and actions. They distinguish free motion, impact-like transitions, and
stick/slip-like contact using smoothed object/end-effector kinematic
changes. They are used for validation diagnostics and held-out evaluation only; the model is not trained with simulator contact-force labels. Hyperparameters are selected on validation episodes, and all reported metrics are computed on held-out test episodes.

\paragraph{Kinematic proxy-label construction.}
For each trajectory, the deterministic proxy uses only recorded state/action
fields and is shared by all methods. Let \(p_t^{\rm obj}\) and
\(p_t^{\rm ee}\) denote object and end-effector positions, and let \(a_t\)
denote the action. We compute normalized finite-difference terms
\[
r_t^{\rm obj}
=
\frac{\|\Delta p_t^{\rm obj}\|_2}
{\operatorname{MAD}(\|\Delta p^{\rm obj}\|_2)+\epsilon},
\qquad
r_t^{\rm ee}
=
\frac{\|\Delta p_t^{\rm ee}\|_2}
{\operatorname{MAD}(\|\Delta p^{\rm ee}\|_2)+\epsilon},
\qquad
r_t^{a}
=
\frac{\|\Delta a_t\|_2}
{\operatorname{MAD}(\|\Delta a\|_2)+\epsilon}.
\]
The scalar score is
\[
c_t
=
\alpha_{\rm obj} r_t^{\rm obj}
+
\alpha_{\rm ee} r_t^{\rm ee}
+
\alpha_a r_t^a .
\]
The weights \(\alpha\), smoothing window, minimum-run denoising length, and
thresholds \(\theta_1<\theta_2\) are selected on validation episodes and then
frozen. After smoothing, labels are assigned as free if \(c_t<\theta_1\),
impact-like if \(c_t\ge\theta_2\), and stick/slip-like otherwise, followed by
the fixed minimum-run denoising. Clipping at the 99.5th percentile is used only
for plotting histograms, not for metric computation. The test split is never
used to choose thresholds. Proxy labels are used only for metric computation
after label-permutation matching; they are not provided to the \method{} mode
posterior during training.

\paragraph{Mode-timeline diagnostic.}
Figure~\ref{fig:si_exp2_timelines} shows representative trajectories. \method{}
produces contiguous mode segments that align with the denoised kinematic proxy
labels. The no-support ablation keeps a discrete state but fragments
transitions more often. The no-mode baseline cannot represent regime switches.
The HMM is a useful segmentation baseline, but it does not share a predictive
latent state with the dynamics model.

\begin{figure}[p]
\centering
\begin{subfigure}[t]{0.49\linewidth}
    \includegraphics[width=\linewidth]{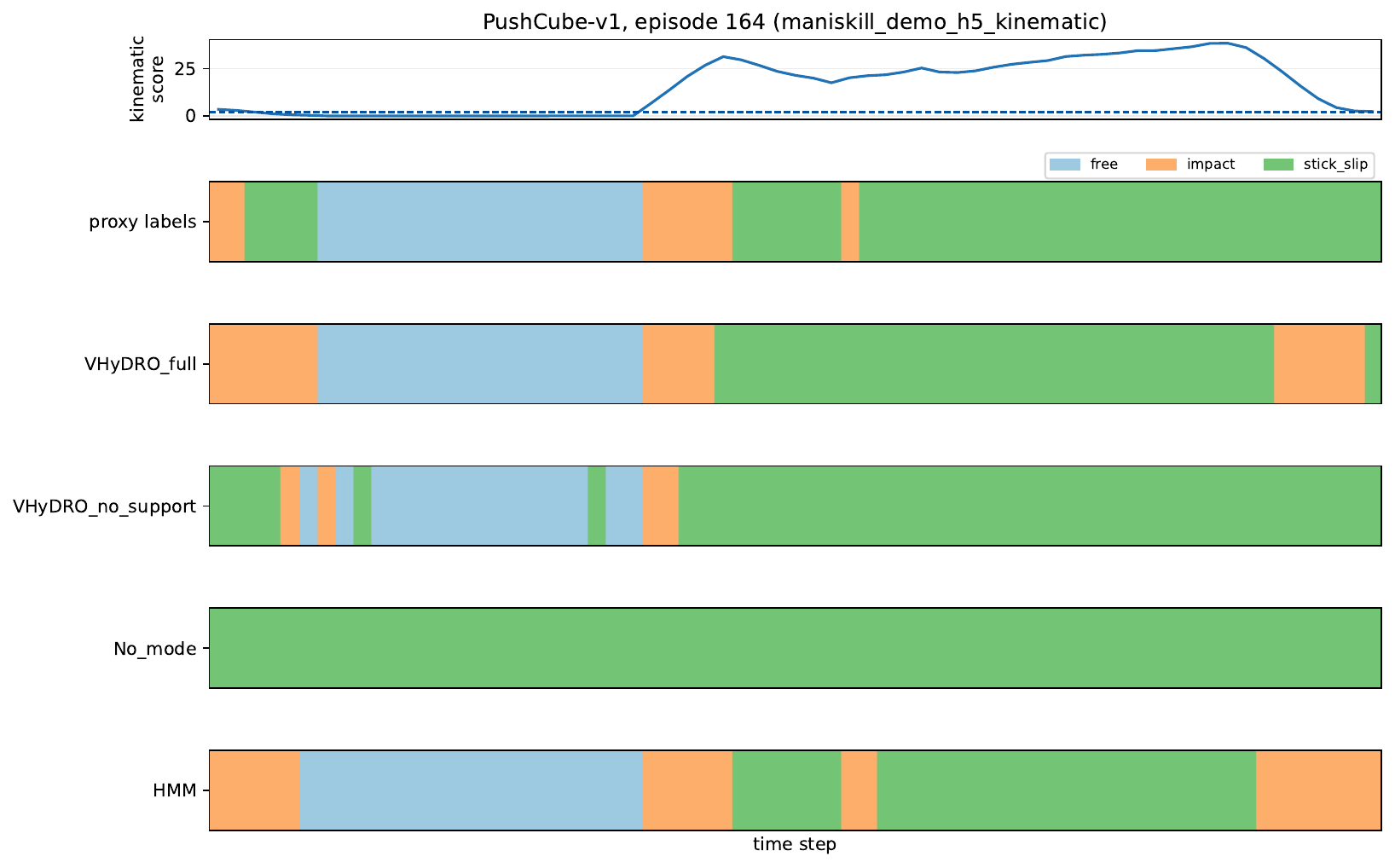}
    \caption{PushCube-v1.}
\end{subfigure}\hfill
\begin{subfigure}[t]{0.49\linewidth}
    \includegraphics[width=\linewidth]{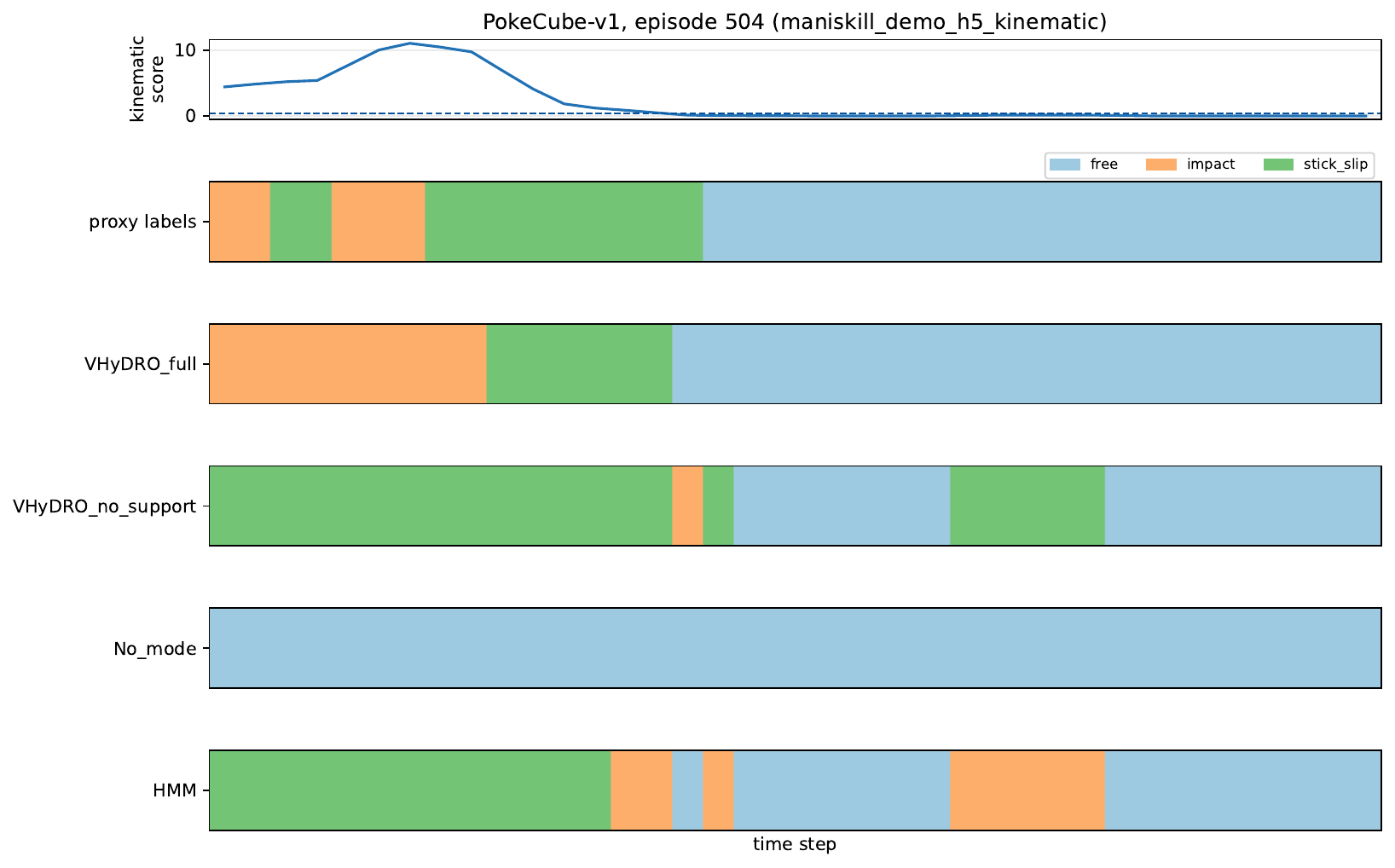}
    \caption{PokeCube-v1.}
\end{subfigure}
\begin{subfigure}[t]{0.49\linewidth}
    \includegraphics[width=\linewidth]{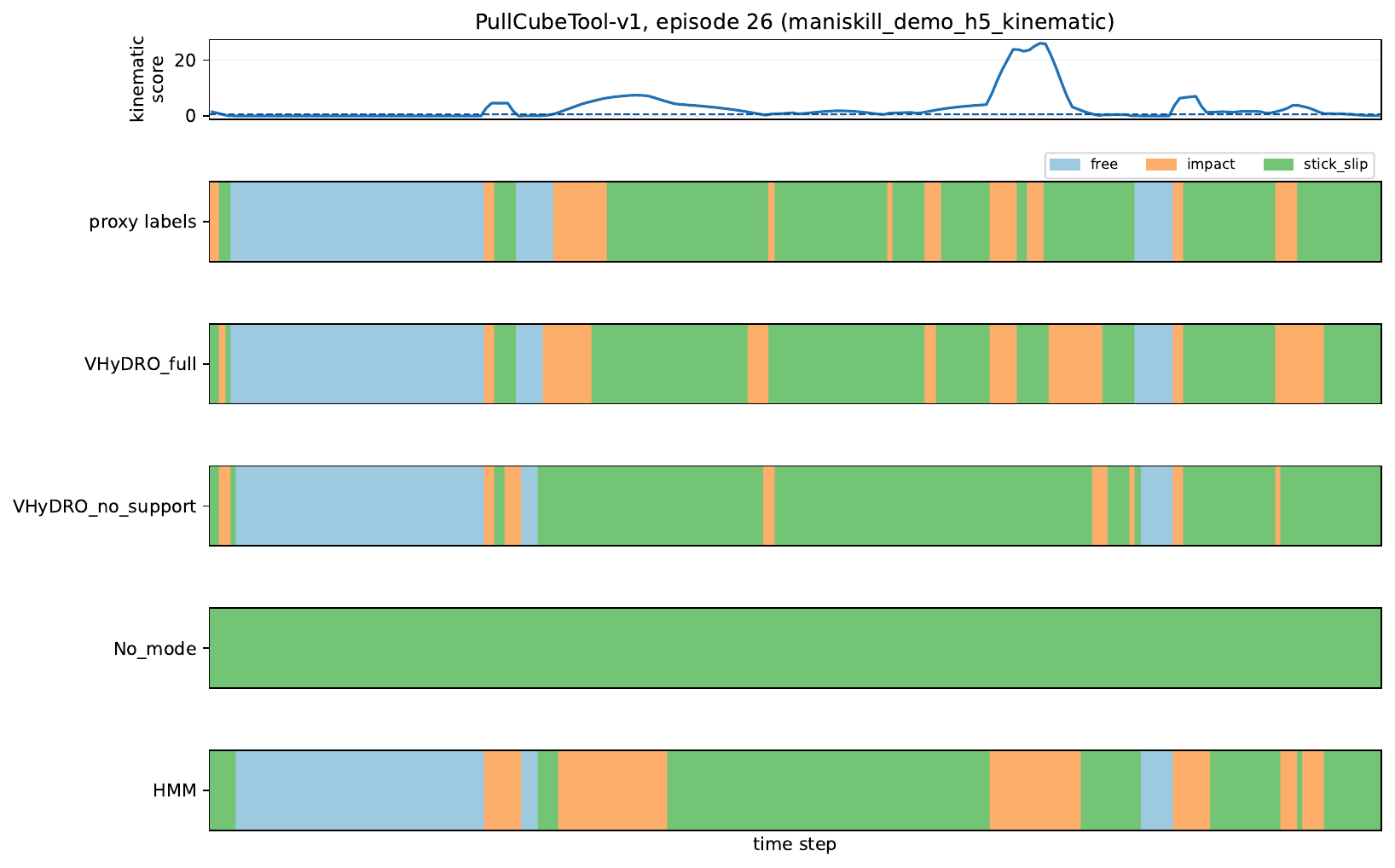}
    \caption{PullCubeTool-v1.}
\end{subfigure}\hfill
\begin{subfigure}[t]{0.49\linewidth}
    \includegraphics[width=\linewidth]{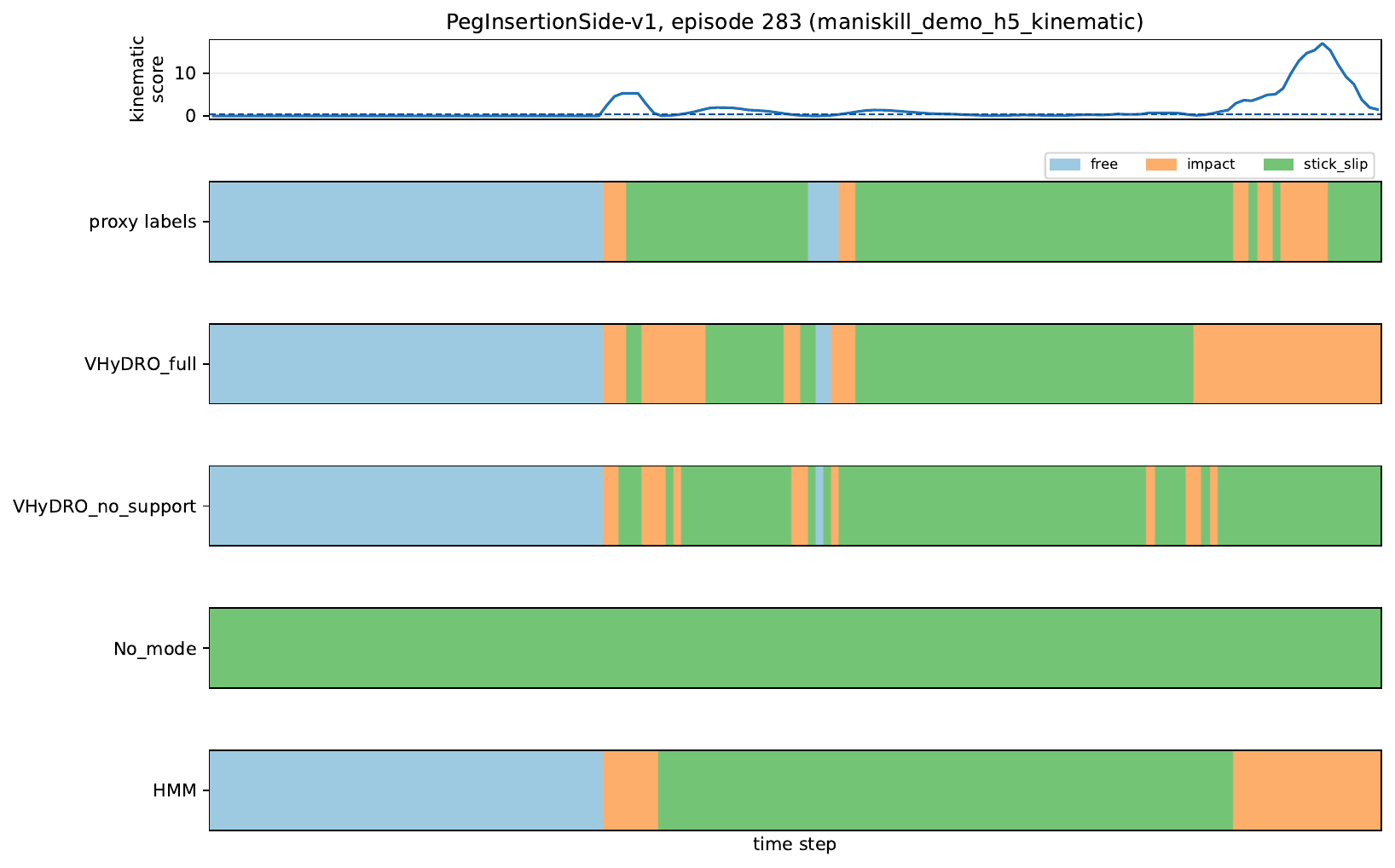}
    \caption{PegInsertionSide-v1.}
\end{subfigure}
\caption{\textbf{Representative held-out mode timelines.} \method{} infers temporally coherent mode segments that align with denoised kinematic proxy labels. Alignment is measured against denoised kinematic proxy labels; no simulator contact-force labels are used.}
\label{fig:si_exp2_timelines}
\end{figure}

\paragraph{Proxy-label diagnostic.}
Figure~\ref{fig:si_exp2_kinematic_proxy_scores} makes the evaluation labels auditable. The score distributions are not used as supervised contact labels for \method{}. They provide a reproducible proxy against which mode F1, ARI, change-point F1, and segment purity can be computed.

\begin{figure}[p]
\centering
\begin{subfigure}[t]{0.49\linewidth}
    \includegraphics[width=\linewidth]{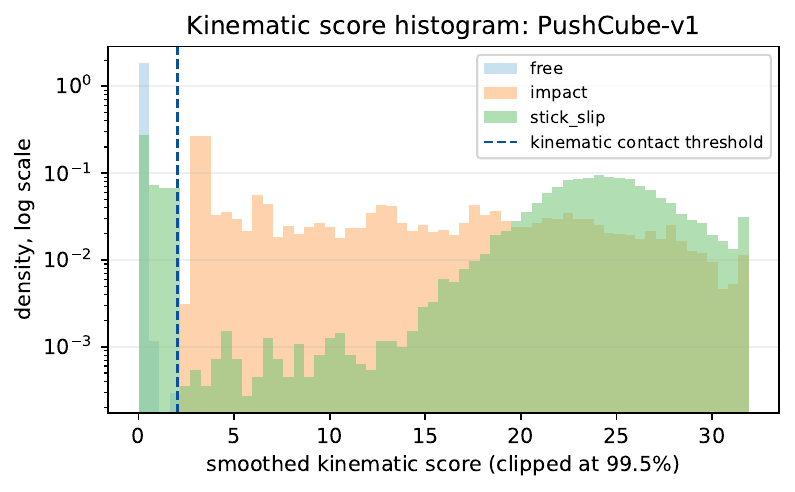}
    \caption{PushCube-v1.}
\end{subfigure}\hfill
\begin{subfigure}[t]{0.49\linewidth}
    \includegraphics[width=\linewidth]{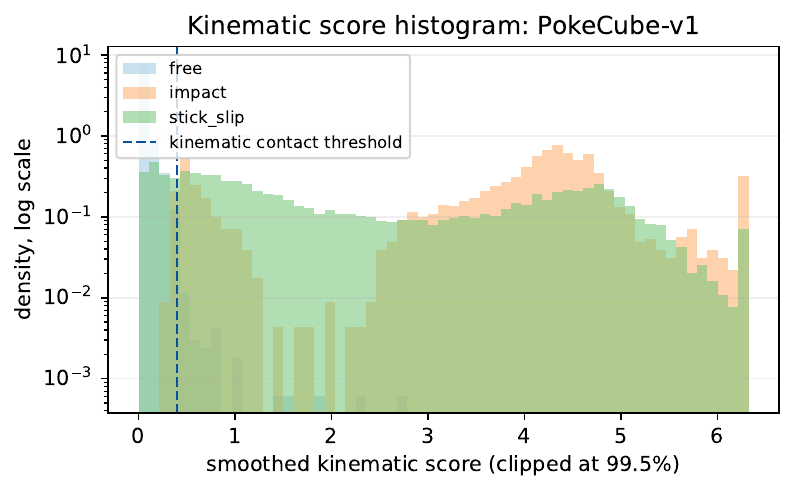}
    \caption{PokeCube-v1.}
\end{subfigure}
\begin{subfigure}[t]{0.49\linewidth}
    \includegraphics[width=\linewidth]{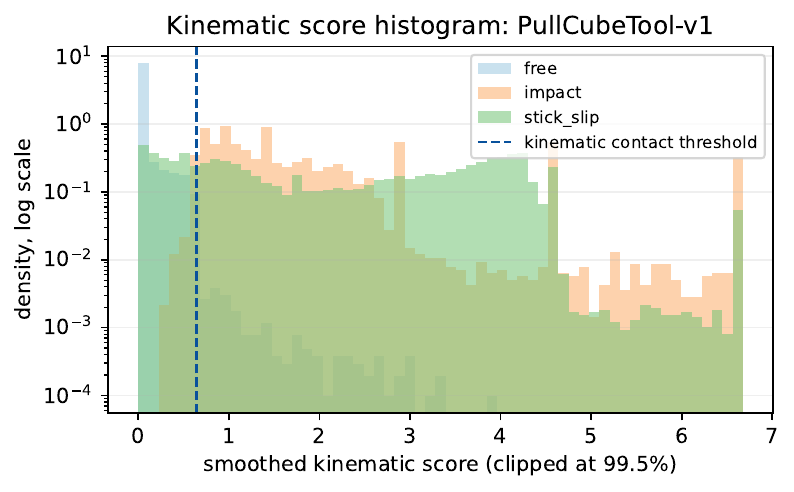}
    \caption{PullCubeTool-v1.}
\end{subfigure}\hfill
\begin{subfigure}[t]{0.49\linewidth}
    \includegraphics[width=\linewidth]{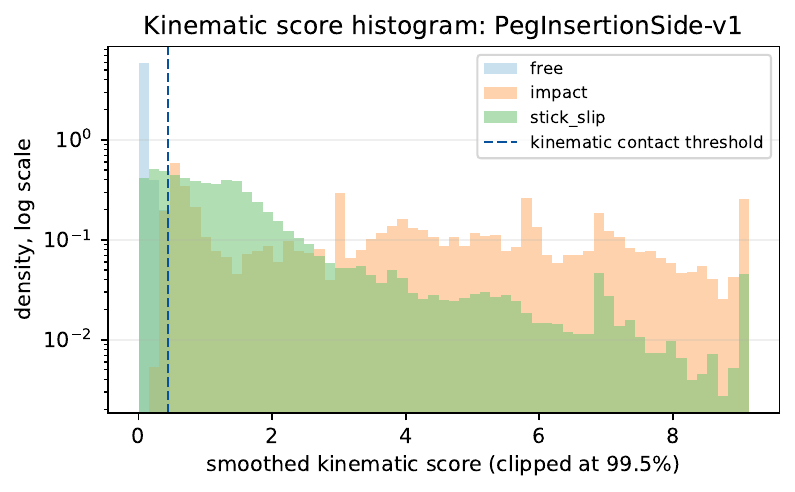}
    \caption{PegInsertionSide-v1.}
\end{subfigure}
\caption{\textbf{Kinematic proxy-label construction.} Histograms of the smoothed kinematic score used to form denoised free/impact/stick-slip proxy labels from HDF5 states and actions. The labels audit mode segmentation; they are not simulator contact-force labels.}
\label{fig:si_exp2_kinematic_proxy_scores}
\end{figure}

\paragraph{Per-task contact-regime segmentation.}
Table~\ref{tab:si_exp2_per_task} reports held-out per-task metrics. \method{} is best on ARI and change-point F1 on every task, and best on purity in three of four tasks. The higher HMM Mode F1 on PushCube-v1, and the higher no-support purity on PokeCube-v1, do not carry over to the full predictive segmentation profile. This distinction matters for Theorem~\ref{thm:mode_recovery_vhydro}: the recovered state must be a coherent predictive mode, not only a post-hoc label sequence.

\begin{table}[p]
\centering
\scriptsize
\caption{\textbf{Experiment~2: per-task proxy-label contact-regime
segmentation.} Held-out test split across ManiSkill HDF5 demonstrations.
Labels are denoised kinematic proxies, not simulator contact-force
labels. Means$\pm$SEM over twenty seeds.}
\label{tab:si_exp2_per_task}
\resizebox{\linewidth}{!}{%
\begin{tabular}{llcccc}
\toprule
Task & Method & Mode F1 $\uparrow$ & ARI $\uparrow$ & Change-point F1 $\uparrow$ & Purity $\uparrow$ \\
\midrule
PushCube-v1 & \method{} & 0.810$\pm$0.004 & \sihighlight{0.714$\pm$0.015} & \sihighlight{0.716$\pm$0.007} & \sihighlight{0.915$\pm$0.002} \\
PushCube-v1 & no support & 0.635$\pm$0.040 & 0.536$\pm$0.023 & 0.371$\pm$0.028 & 0.901$\pm$0.008 \\
PushCube-v1 & HMM & \sihighlight{0.828$\pm$0.008} & 0.647$\pm$0.054 & 0.664$\pm$0.016 & 0.907$\pm$0.013 \\
PushCube-v1 & no mode & 0.218$\pm$0.002 & 0.000$\pm$0.000 & 0.000$\pm$0.000 & 0.485$\pm$0.006 \\
PokeCube-v1 & \method{} & \sihighlight{0.583$\pm$0.006} & \sihighlight{0.599$\pm$0.014} & \sihighlight{0.557$\pm$0.013} & 0.886$\pm$0.001 \\
PokeCube-v1 & no support & 0.558$\pm$0.007 & 0.398$\pm$0.016 & 0.420$\pm$0.029 & \sihighlight{0.914$\pm$0.007} \\
PokeCube-v1 & HMM & 0.523$\pm$0.008 & 0.403$\pm$0.018 & 0.409$\pm$0.006 & 0.897$\pm$0.004 \\
PokeCube-v1 & no mode & 0.229$\pm$0.001 & 0.000$\pm$0.000 & 0.000$\pm$0.000 & 0.524$\pm$0.005 \\
PullCubeTool-v1 & \method{} & \sihighlight{0.770$\pm$0.002} & \sihighlight{0.697$\pm$0.006} & \sihighlight{0.698$\pm$0.002} & \sihighlight{0.942$\pm$0.001} \\
PullCubeTool-v1 & no support & 0.648$\pm$0.006 & 0.495$\pm$0.021 & 0.509$\pm$0.001 & 0.915$\pm$0.008 \\
PullCubeTool-v1 & HMM & 0.672$\pm$0.009 & 0.488$\pm$0.003 & 0.530$\pm$0.022 & 0.876$\pm$0.006 \\
PullCubeTool-v1 & no mode & 0.217$\pm$0.002 & 0.000$\pm$0.000 & 0.000$\pm$0.000 & 0.482$\pm$0.005 \\
PegInsertionSide-v1 & \method{} & \sihighlight{0.732$\pm$0.003} & \sihighlight{0.578$\pm$0.009} & \sihighlight{0.637$\pm$0.002} & \sihighlight{0.934$\pm$0.000} \\
PegInsertionSide-v1 & no support & 0.607$\pm$0.022 & 0.430$\pm$0.012 & 0.431$\pm$0.026 & 0.932$\pm$0.009 \\
PegInsertionSide-v1 & HMM & 0.678$\pm$0.006 & 0.422$\pm$0.013 & 0.504$\pm$0.005 & 0.879$\pm$0.002 \\
PegInsertionSide-v1 & no mode & 0.220$\pm$0.003 & 0.000$\pm$0.000 & 0.000$\pm$0.000 & 0.494$\pm$0.011 \\
\bottomrule
\end{tabular}}
\end{table}

\FloatBarrier

\subsection{Sawyer/BridgeData-style contact-regime diagnostics with proxy labels}
\label{app:exp2_sawyer_bridge_proxy}

We run the Experiment~2 contact-regime segmentation protocol on four
Sawyer/BridgeData task families: BridgeDataOcclusion, SawyerTowerCreation,
SawyerLaundryLayout, and SawyerObjectSearch. The evaluation labels are
denoised kinematic proxy labels derived from trajectory state and action
changes. They separate free motion, impact-like transitions, and sustained
interaction, and are used only for validation diagnostics and held-out
evaluation.

\begin{table}[t]
\centering
\small
\caption{\textbf{Additional Sawyer/BridgeData-style predictive
mode-segmentation diagnostic.} Mean$\pm$SEM over four task families and
twenty seeds. Higher is better. Bold entries mark the quantities
emphasized by the mode-concentration claim.}
\label{tab:si_exp2_rest_aggregate}
\resizebox{0.86\linewidth}{!}{%
\begin{tabular}{lcccc}
\toprule
Method & Mode F1 $\uparrow$ & ARI $\uparrow$ & Change-point F1 $\uparrow$ & Purity $\uparrow$ \\
\midrule
\method{} & \sihighlight{0.708$\pm$0.003} & \sihighlight{0.560$\pm$0.008} & \sihighlight{0.706$\pm$0.004} & 0.918$\pm$0.002 \\
no support & 0.535$\pm$0.009 & 0.171$\pm$0.010 & 0.389$\pm$0.004 & \sihighlight{0.920$\pm$0.003} \\
HMM & 0.470$\pm$0.005 & 0.174$\pm$0.004 & 0.348$\pm$0.003 & 0.882$\pm$0.001 \\
no mode & 0.222$\pm$0.001 & 0.000$\pm$0.000 & 0.000$\pm$0.000 & 0.498$\pm$0.003 \\
\bottomrule
\end{tabular}}
\end{table}

\FloatBarrier

\paragraph{Representative timelines.}
Figure~\ref{fig:si_exp2_rest_timelines} shows representative Sawyer/BridgeData
task timelines. \method{} produces longer coherent segments than the
no-support ablation, which switches rapidly on several trajectories. The
no-mode baseline collapses to one regime and therefore cannot represent contact
changes.

\begin{figure}[p]
\centering
\begin{subfigure}[t]{0.49\linewidth}
    \includegraphics[width=\linewidth]{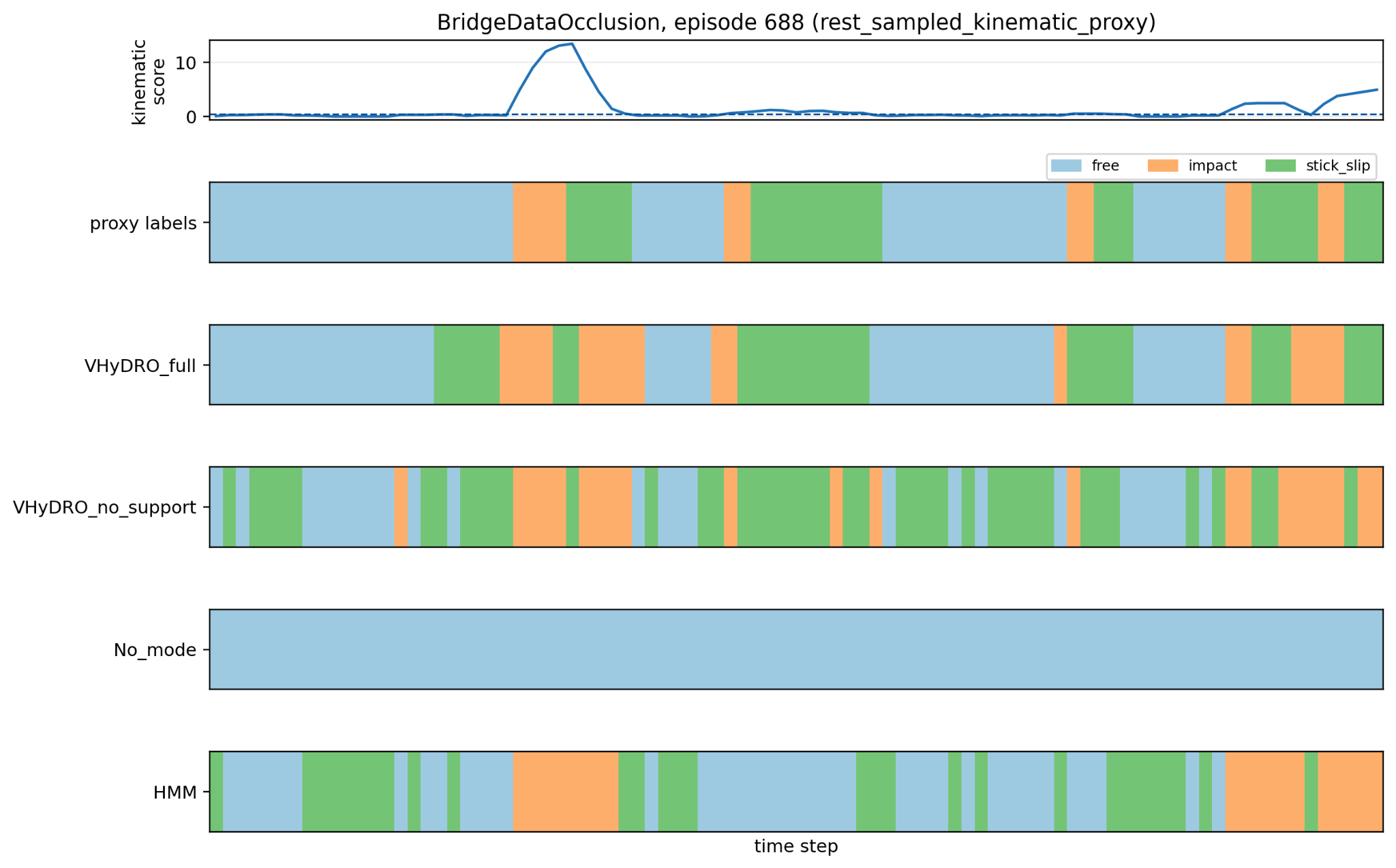}
    \caption{BridgeDataOcclusion.}
\end{subfigure}\hfill
\begin{subfigure}[t]{0.49\linewidth}
    \includegraphics[width=\linewidth]{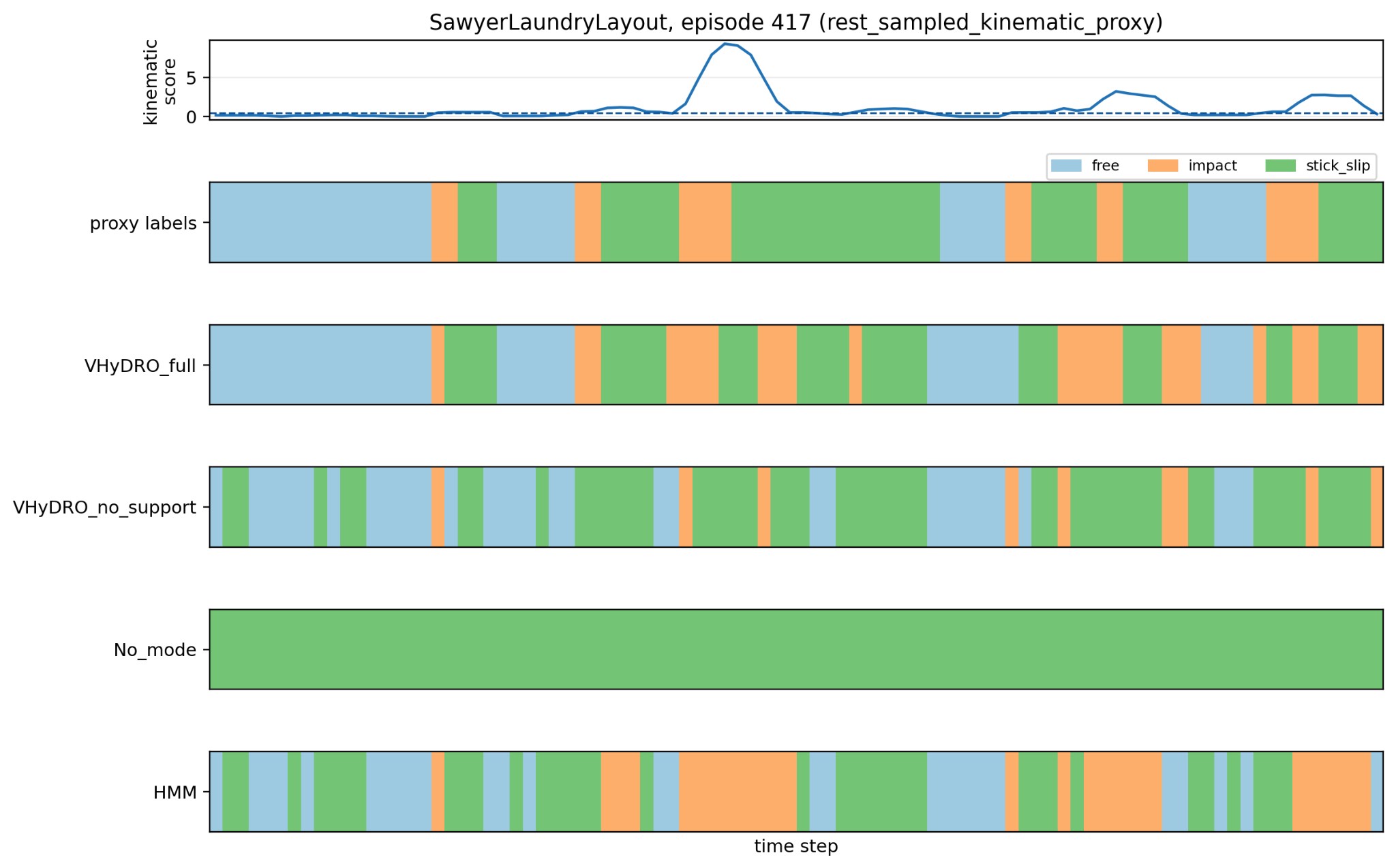}
    \caption{SawyerLaundryLayout.}
\end{subfigure}
\begin{subfigure}[t]{0.49\linewidth}
    \includegraphics[width=\linewidth]{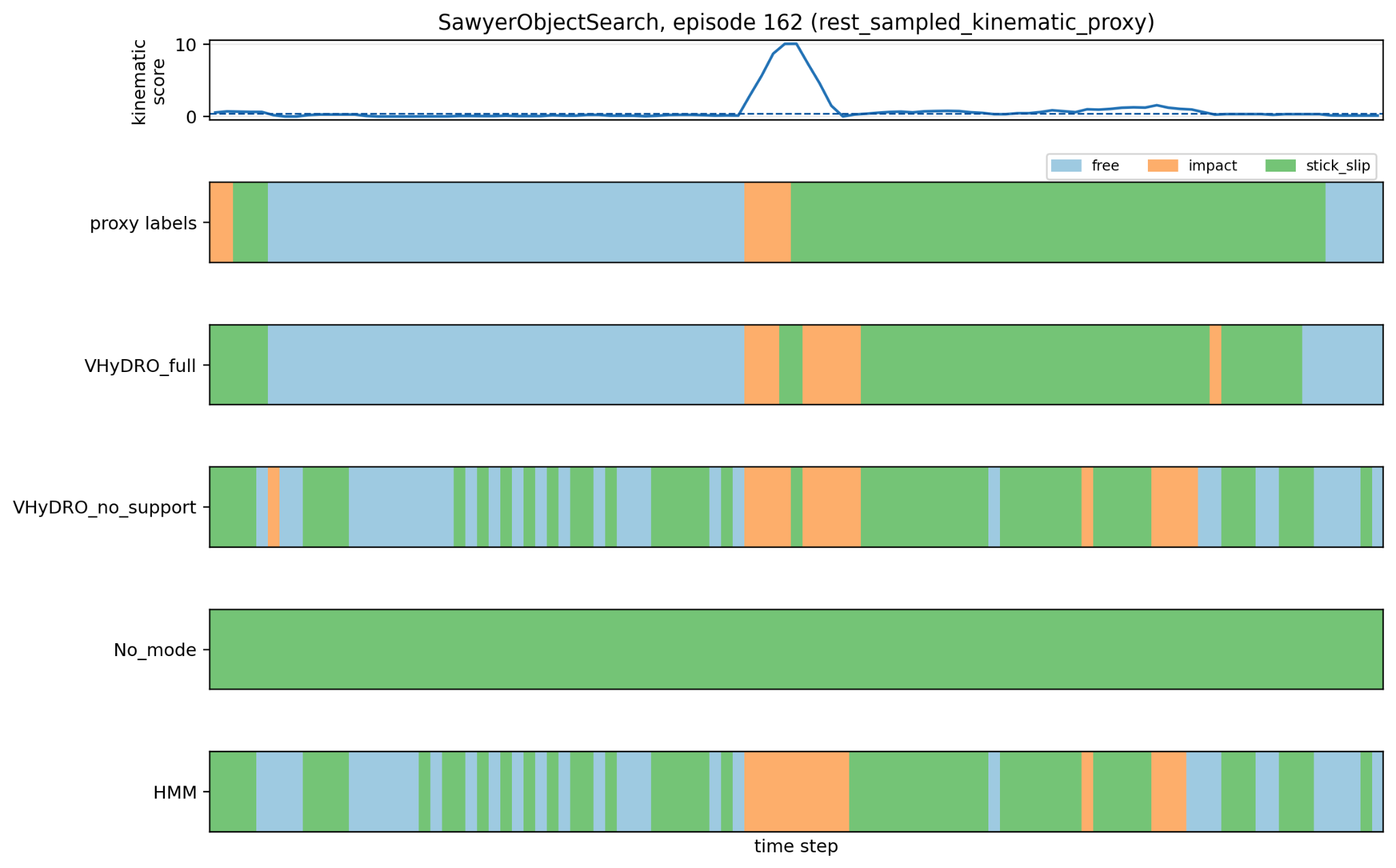}
    \caption{SawyerObjectSearch.}
\end{subfigure}\hfill
\begin{subfigure}[t]{0.49\linewidth}
    \includegraphics[width=\linewidth]{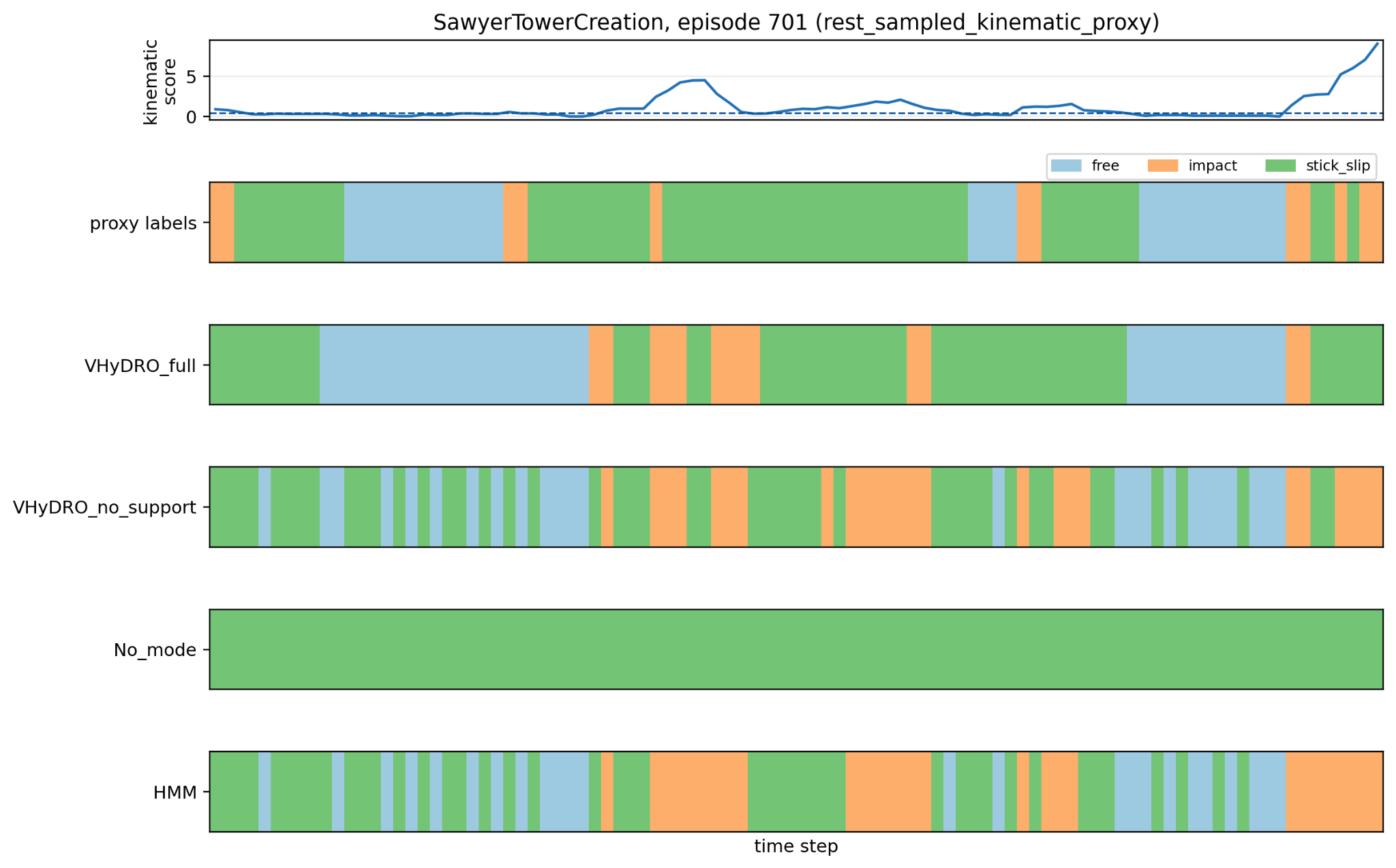}
    \caption{SawyerTowerCreation.}
\end{subfigure}
\caption{\textbf{Sawyer/BridgeData mode timelines with proxy labels.}
Each panel shows the smoothed kinematic score, denoised kinematic proxy labels,
and inferred mode sequences. The labels audit mode segmentation; they are not
simulator contact-force labels.}
\label{fig:si_exp2_rest_timelines}
\end{figure}

\paragraph{Proxy-label audit.}
Figure~\ref{fig:si_exp2_rest_hists} shows the smoothed kinematic score
distributions used to construct the kinematic proxy labels. These histograms
document the evaluation target; they are not contact-force histograms and are
not used as supervised training labels.

\begin{figure}[p]
\centering
\begin{subfigure}[t]{0.49\linewidth}
    \includegraphics[width=\linewidth]{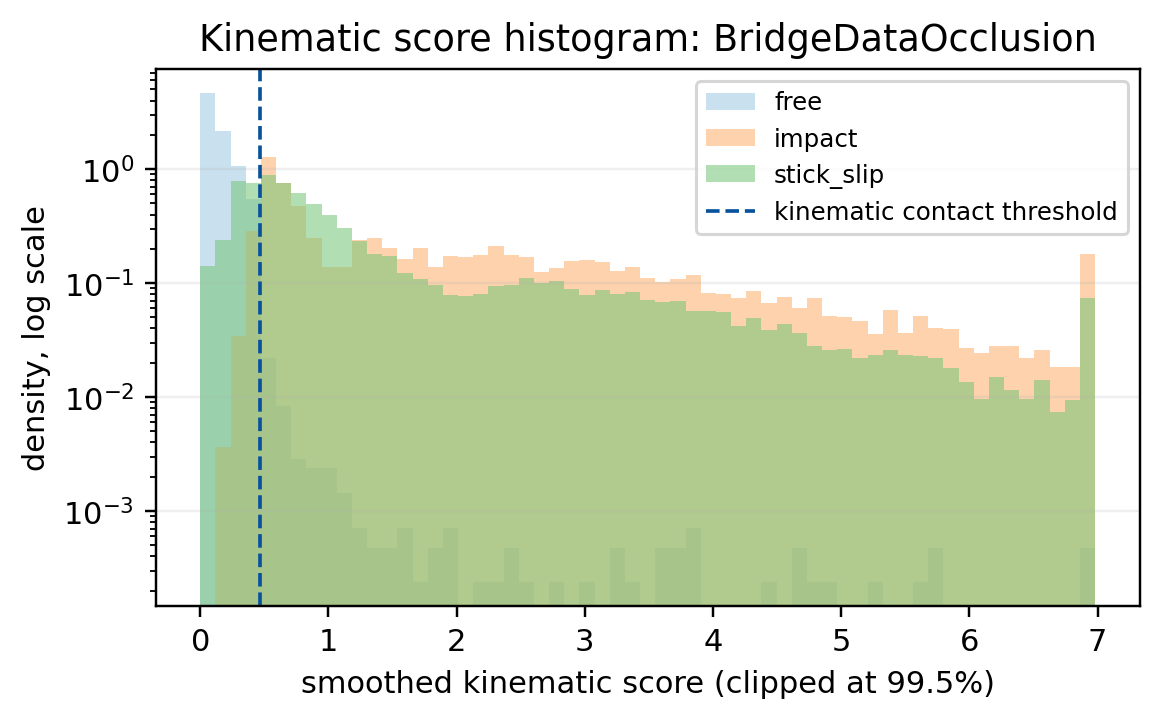}
    \caption{BridgeDataOcclusion.}
\end{subfigure}\hfill
\begin{subfigure}[t]{0.49\linewidth}
    \includegraphics[width=\linewidth]{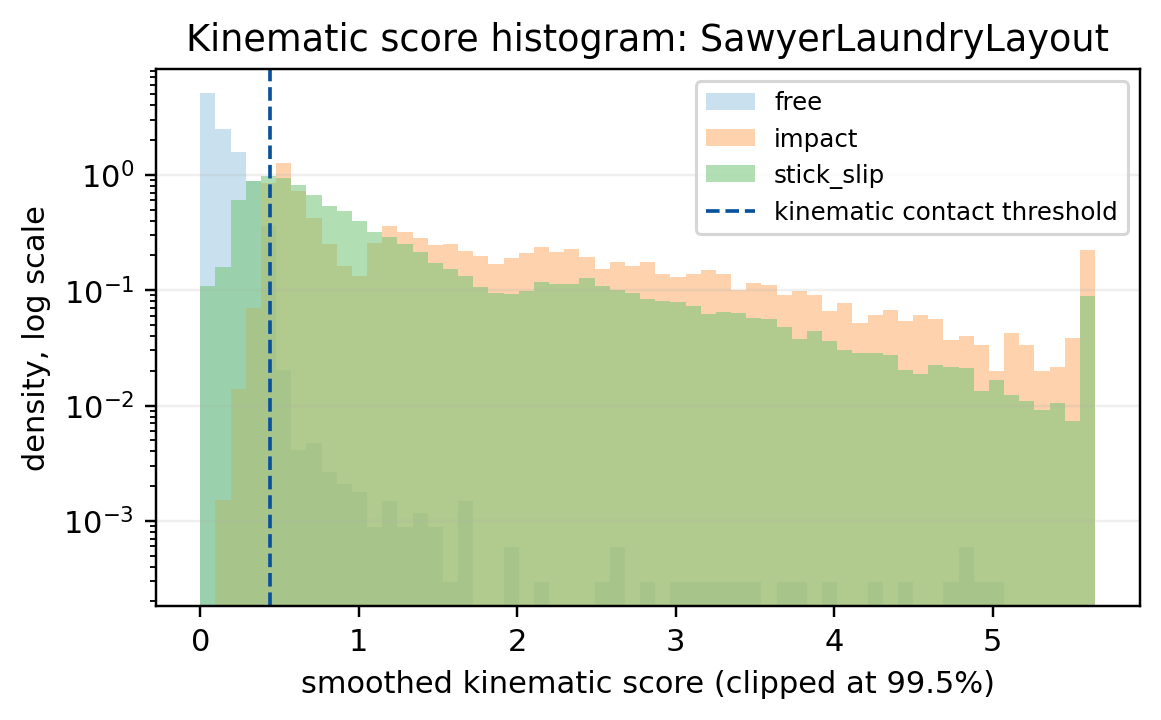}
    \caption{SawyerLaundryLayout.}
\end{subfigure}
\begin{subfigure}[t]{0.49\linewidth}
    \includegraphics[width=\linewidth]{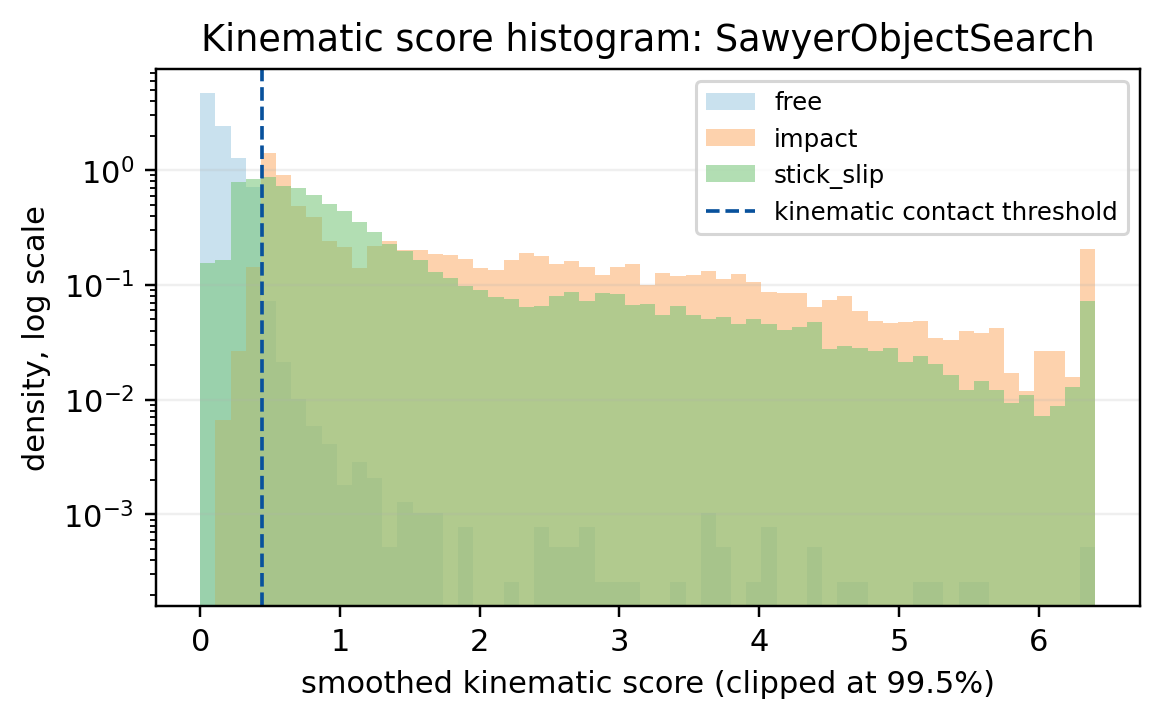}
    \caption{SawyerObjectSearch.}
\end{subfigure}\hfill
\begin{subfigure}[t]{0.49\linewidth}
    \includegraphics[width=\linewidth]{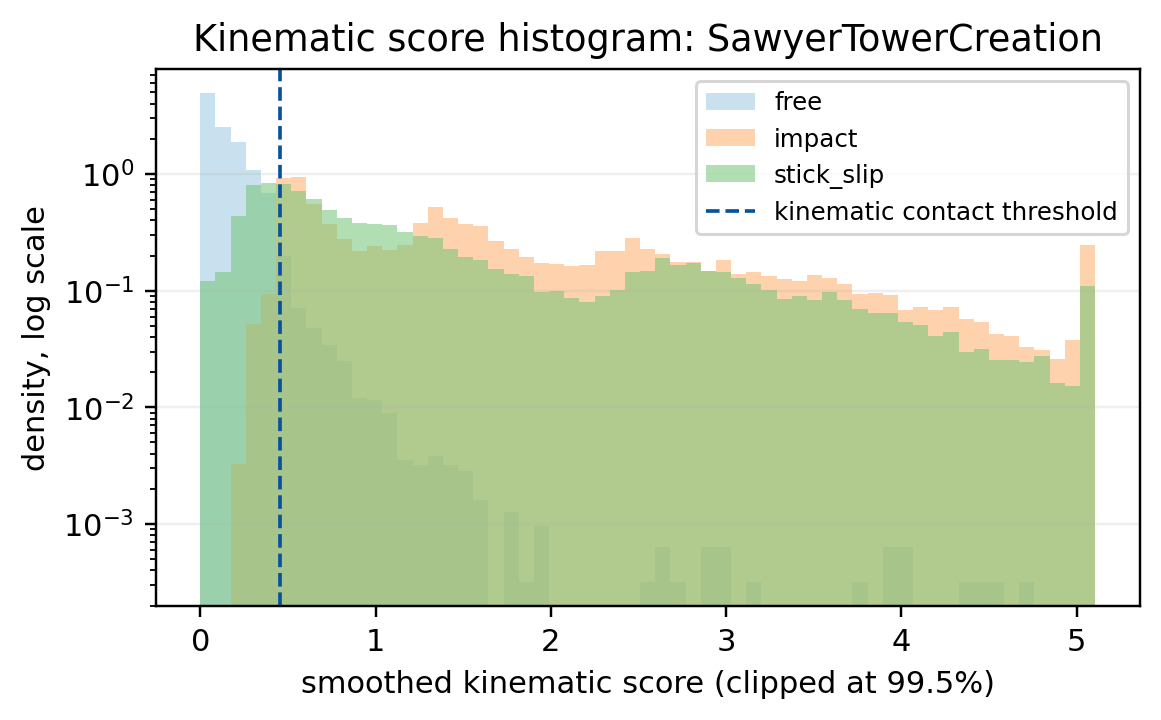}
    \caption{SawyerTowerCreation.}
\end{subfigure}
\caption{\textbf{Sawyer/BridgeData kinematic-score histograms.} Smoothed kinematic-score distributions used to form free, impact-like, and stick/slip-like kinematic proxy labels. The labels audit mode segmentation; they are not simulator contact-force labels.}
\label{fig:si_exp2_rest_hists}
\end{figure}

\paragraph{Per-task audit.}
Table~\ref{tab:si_exp2_rest_per_task} reports task-level metrics for the
Sawyer/BridgeData task families using frozen kinematic proxy labels. The
aggregate pattern in Table~\ref{tab:si_exp2_rest_aggregate} is consistent
across tasks: \method{} improves the metrics tied to mode concentration, while
purity alone is not sufficient to identify coherent predictive modes.

\begin{table}[p]
\centering
\scriptsize
\caption{\textbf{Per-task Sawyer/BridgeData-style predictive
mode-segmentation diagnostic.} Held-out test metrics for
Sawyer/BridgeData-style proxy trajectories. Labels are denoised
kinematic proxies, not simulator contact-force labels. Means$\pm$SEM
over twenty seeds. Bold entries mark the best predictive segmentation
value within each task block.}
\label{tab:si_exp2_rest_per_task}
\resizebox{\linewidth}{!}{%
\begin{tabular}{llcccc}
\toprule
Task & Method & Mode F1 $\uparrow$ & ARI $\uparrow$ & Change-point F1 $\uparrow$ & Purity $\uparrow$ \\
\midrule
SawyerTowerCreation & \method{} & \sihighlight{0.697$\pm$0.001} & \sihighlight{0.528$\pm$0.004} & \sihighlight{0.697$\pm$0.009} & 0.908$\pm$0.001 \\
SawyerTowerCreation & no support & 0.516$\pm$0.009 & 0.165$\pm$0.024 & 0.398$\pm$0.002 & \sihighlight{0.917$\pm$0.009} \\
SawyerTowerCreation & HMM & 0.448$\pm$0.014 & 0.193$\pm$0.006 & 0.363$\pm$0.001 & 0.880$\pm$0.002 \\
SawyerTowerCreation & no mode & 0.221$\pm$0.002 & 0.000$\pm$0.000 & 0.000$\pm$0.000 & 0.495$\pm$0.005 \\
SawyerLaundryLayout & \method{} & \sihighlight{0.724$\pm$0.005} & \sihighlight{0.592$\pm$0.014} & \sihighlight{0.722$\pm$0.009} & \sihighlight{0.928$\pm$0.002} \\
SawyerLaundryLayout & no support & 0.579$\pm$0.009 & 0.187$\pm$0.024 & 0.403$\pm$0.011 & 0.922$\pm$0.005 \\
SawyerLaundryLayout & HMM & 0.489$\pm$0.001 & 0.161$\pm$0.002 & 0.346$\pm$0.003 & 0.881$\pm$0.003 \\
SawyerLaundryLayout & no mode & 0.222$\pm$0.000 & 0.000$\pm$0.000 & 0.000$\pm$0.000 & 0.499$\pm$0.001 \\
BridgeDataOcclusion & \method{} & \sihighlight{0.707$\pm$0.001} & \sihighlight{0.558$\pm$0.007} & \sihighlight{0.705$\pm$0.004} & 0.919$\pm$0.001 \\
BridgeDataOcclusion & no support & 0.513$\pm$0.001 & 0.192$\pm$0.003 & 0.381$\pm$0.000 & \sihighlight{0.930$\pm$0.002} \\
BridgeDataOcclusion & HMM & 0.472$\pm$0.001 & 0.177$\pm$0.003 & 0.350$\pm$0.001 & 0.886$\pm$0.001 \\
BridgeDataOcclusion & no mode & 0.220$\pm$0.001 & 0.000$\pm$0.000 & 0.000$\pm$0.000 & 0.492$\pm$0.005 \\
SawyerObjectSearch & \method{} & \sihighlight{0.704$\pm$0.000} & \sihighlight{0.563$\pm$0.008} & \sihighlight{0.700$\pm$0.004} & \sihighlight{0.915$\pm$0.002} \\
SawyerObjectSearch & no support & 0.530$\pm$0.010 & 0.139$\pm$0.018 & 0.375$\pm$0.005 & 0.912$\pm$0.007 \\
SawyerObjectSearch & HMM & 0.470$\pm$0.003 & 0.165$\pm$0.003 & 0.335$\pm$0.003 & 0.881$\pm$0.002 \\
SawyerObjectSearch & no mode & 0.224$\pm$0.003 & 0.000$\pm$0.000 & 0.000$\pm$0.000 & 0.505$\pm$0.009 \\
\bottomrule
\end{tabular}}
\end{table}
\FloatBarrier

\paragraph{Audit reproducibility.}
The aggregate and per-task tables report seed-level means, task-level means,
and standard errors computed from the fixed held-out evaluation protocol.

\subsection{Experiment 3: sparse physical-law recovery audit}
\label{app:exp3_controlled_sparse_ph}

This appendix audits the sparse-law link in the theory chain:
\emph{mode concentration \(\Rightarrow\) sparse physical-law recovery}. The
benchmark uses known sparse port-Hamiltonian laws, so support recovery,
coefficient error, vector-field error, and physical-constant error can be
measured directly.

\paragraph{Protocol.}
We evaluate four controlled known-equation hybrid systems with known sparse laws: a bouncing puck with wall contacts, a sliding block with Coulomb friction, a pendulum with contact stops, and a planar pushing system. The same twenty-seed evaluation protocol is applied to every method and includes observation noise, missing data, hidden velocities, numerical derivative noise, and small mode-assignment perturbations. We compare \method{}-full, \emph{no-mode}, \emph{no-sparsity}, and \emph{no-pH}. The first three produce port-Hamiltonian coefficients and can therefore be evaluated by sparse support recovery, relative coefficient error, and physical-constant error; the no-pH baseline is evaluated only as a vector-field predictor. The candidate libraries contain at most \(p=12\) terms per system, and the
true active support has size at most \(k=3\). These values match the
low-dimensional sparse setting used when interpreting
Theorem~\ref{thm:sparse_ph_recovery}.

\paragraph{Summary metrics.}
Table~\ref{tab:exp3_sparse_ph_summary} is the main quantitative audit. \method{} is best on every metric that reflects sparse physical-law recovery: support F1, relative coefficient error, vector-field NRMSE, and physical-constant error. The no-mode ablation merges incompatible regimes and degrades both support recovery and coefficient estimation. The no-sparsity ablation predicts more accurately than no-mode, but the dense fit obscures the active support. The no-pH baseline reaches a moderate vector-field error but has no sparse port-Hamiltonian coefficients to inspect.

\begin{table*}[t]
\centering
\small
\begin{tabular}{lcccc}
\toprule
Method & Support F1 $\uparrow$ & Rel. coeff. error $\downarrow$ & Vector-field NRMSE $\downarrow$ & Phys.-constant error $\downarrow$ \\
\midrule
\method{}-full & \textbf{$0.978\pm0.042$} & \textbf{$0.006\pm0.010$} & \textbf{$0.001\pm0.001$} & \textbf{$0.001\pm0.001$} \\
no-mode & $0.713\pm0.039$ & $1.087\pm0.330$ & $0.287\pm0.196$ & $0.188\pm0.234$ \\
no-sparsity & $0.771\pm0.073$ & $0.851\pm0.559$ & $0.043\pm0.043$ & $0.014\pm0.008$ \\
no-pH & N/A & N/A & $0.031\pm0.015$ & N/A \\
\bottomrule
\end{tabular}
\caption{\textbf{Experiment~3 summary.} Values are mean$\pm$standard deviation over \(20\) seeds for each known-equation system and fixed evaluation-condition batch, aggregated over the four systems. For the dense no-sparsity baseline, Support F1 is computed after applying
the same post-hoc coefficient threshold used for the sparse-support
audit. The no-pH model is evaluated only by vector-field prediction because it has no sparse port-Hamiltonian coefficients.}
\label{tab:exp3_sparse_ph_summary}
\end{table*}

\paragraph{Coefficient-recovery diagnostics.}
Figure~\ref{fig:si_exp3_coeff_scatter} shows estimated versus true active coefficients. The \method{} points stay close to the identity line across the active terms, while the dense no-sparsity fit spreads mass across nonessential terms and inflates coefficient magnitude. Figure~\ref{fig:si_exp3_physical_constant_error} reports the corresponding physical-constant error: \method{} stays near zero, whereas the no-mode ablation varies widely because it tries to explain multiple regimes with a single law.

\begin{figure}[p]
\centering
\includegraphics[width=0.72\linewidth]{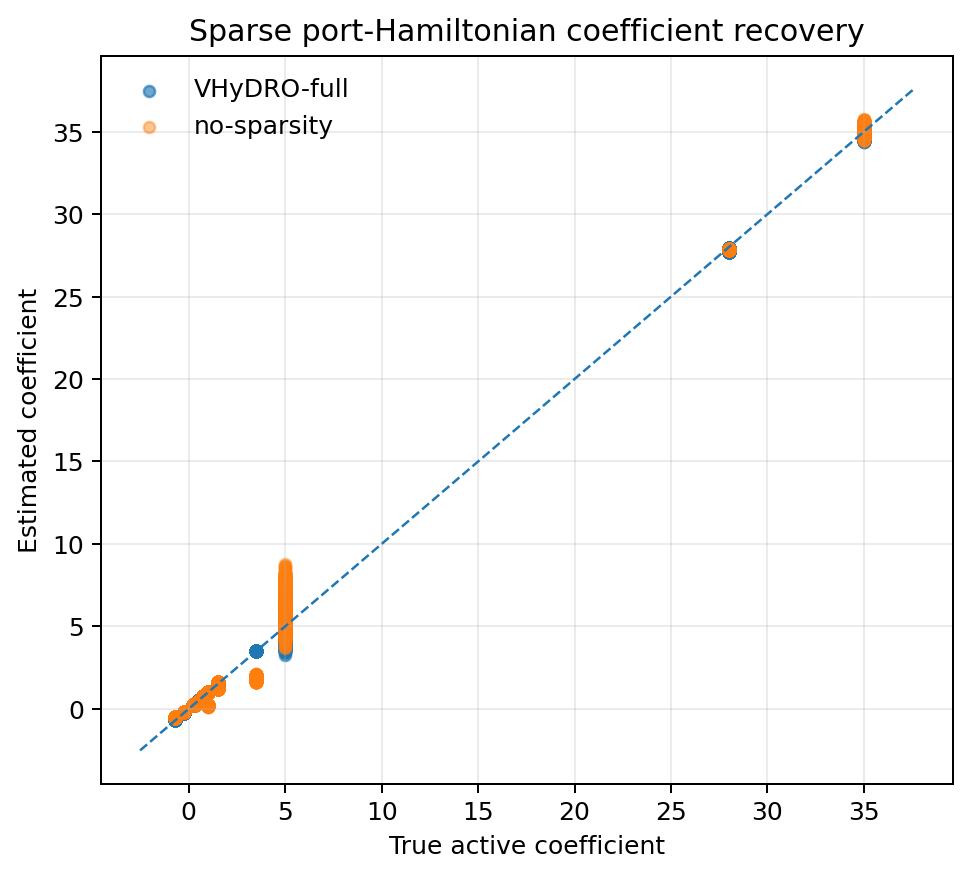}
\caption{\textbf{Coefficient recovery scatter.} Estimated active coefficients versus true active coefficients. \method{} concentrates tightly around the identity line, while the no-sparsity ablation exhibits larger spread and systematic inflation.}
\label{fig:si_exp3_coeff_scatter}
\end{figure}

\begin{figure}[p]
\centering
\includegraphics[width=0.62\linewidth]{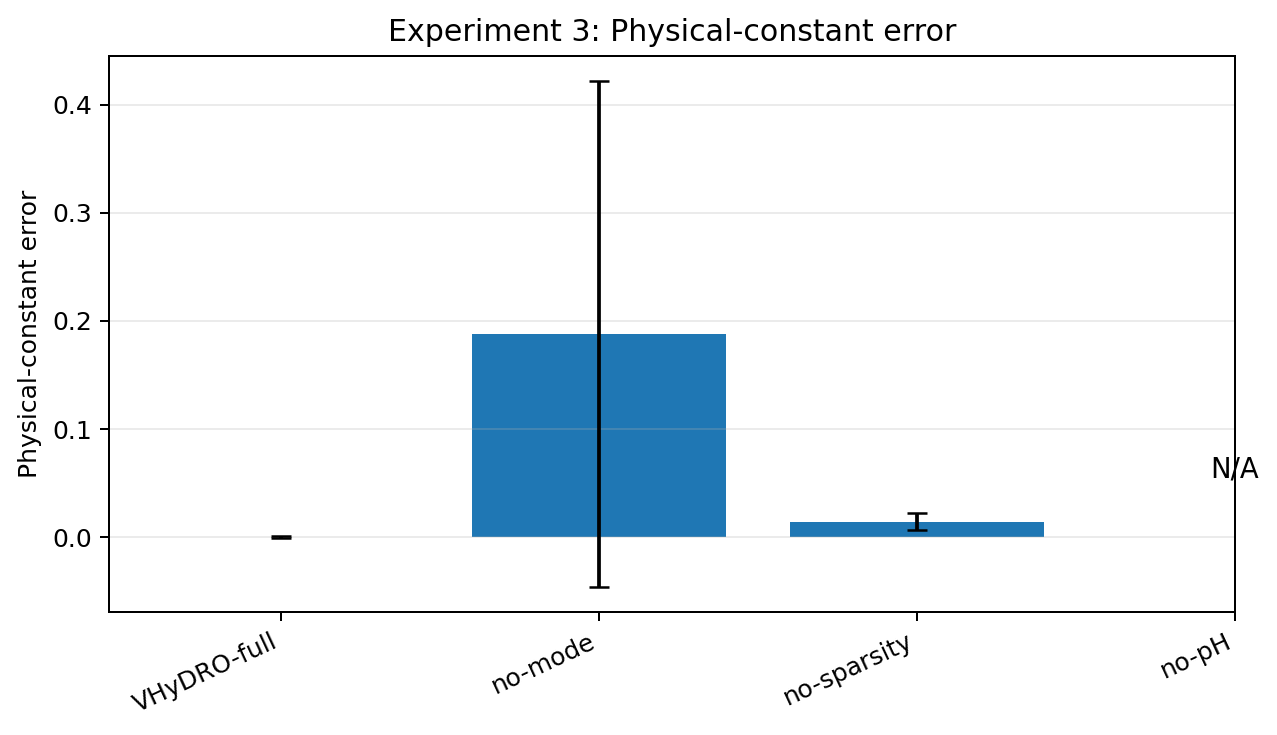}
\caption{\textbf{Physical-constant error.} \method{} preserves the physical constants implied by the recovered law. The no-mode ablation is unstable because it conflates distinct regimes, and the no-pH baseline is not defined in this metric.}
\label{fig:si_exp3_physical_constant_error}
\end{figure}

\paragraph{Support-recovery diagnostics.}
Figure~\ref{fig:si_exp3_support_precision} decomposes support recovery into selectivity. Precision is the informative component: \method{} substantially exceeds the dense and no-mode ablations, indicating that the recovered support is not only inclusive but selective. The support heatmap is included in Figure~\ref{fig:exp3_sparse_ph}.

\begin{figure}[p]
\centering
\includegraphics[width=0.62\linewidth]{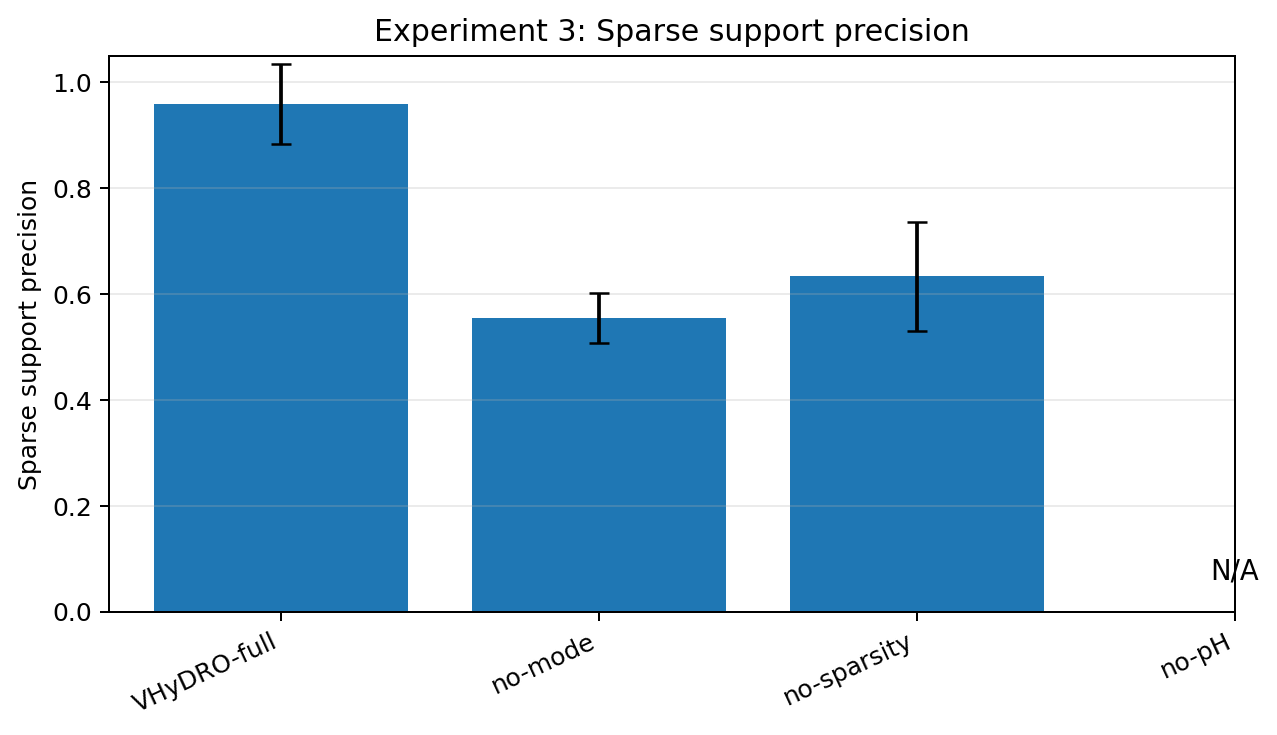}
\caption{\textbf{Sparse support precision.} The key difference between \method{} and the ablations is selectivity: \method{} recovers the active support while suppressing inactive library terms.}
\label{fig:si_exp3_support_precision}
\end{figure}

\begin{figure*}[t]
\centering
\begin{minipage}{0.245\textwidth}
\centering
\includegraphics[width=\linewidth]{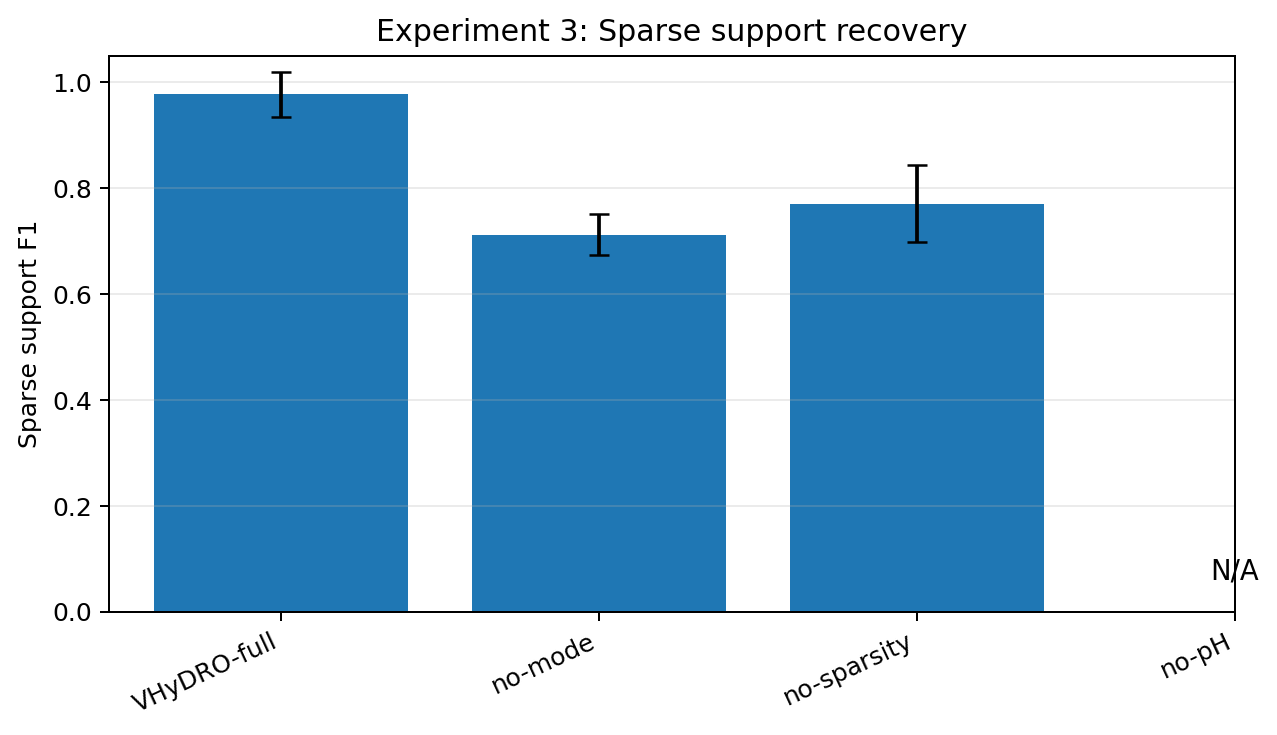}
\end{minipage}\hfill
\begin{minipage}{0.245\textwidth}
\centering
\includegraphics[width=\linewidth]{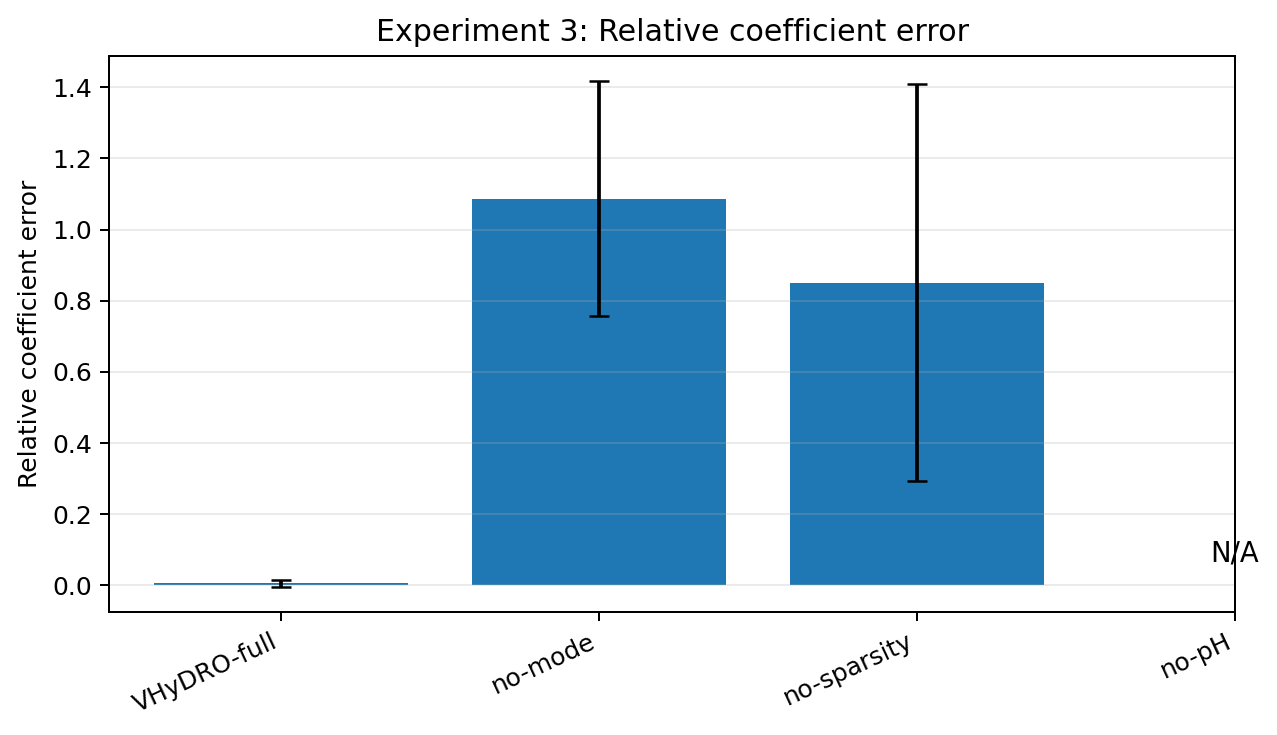}
\end{minipage}\hfill
\begin{minipage}{0.245\textwidth}
\centering
\includegraphics[width=\linewidth]{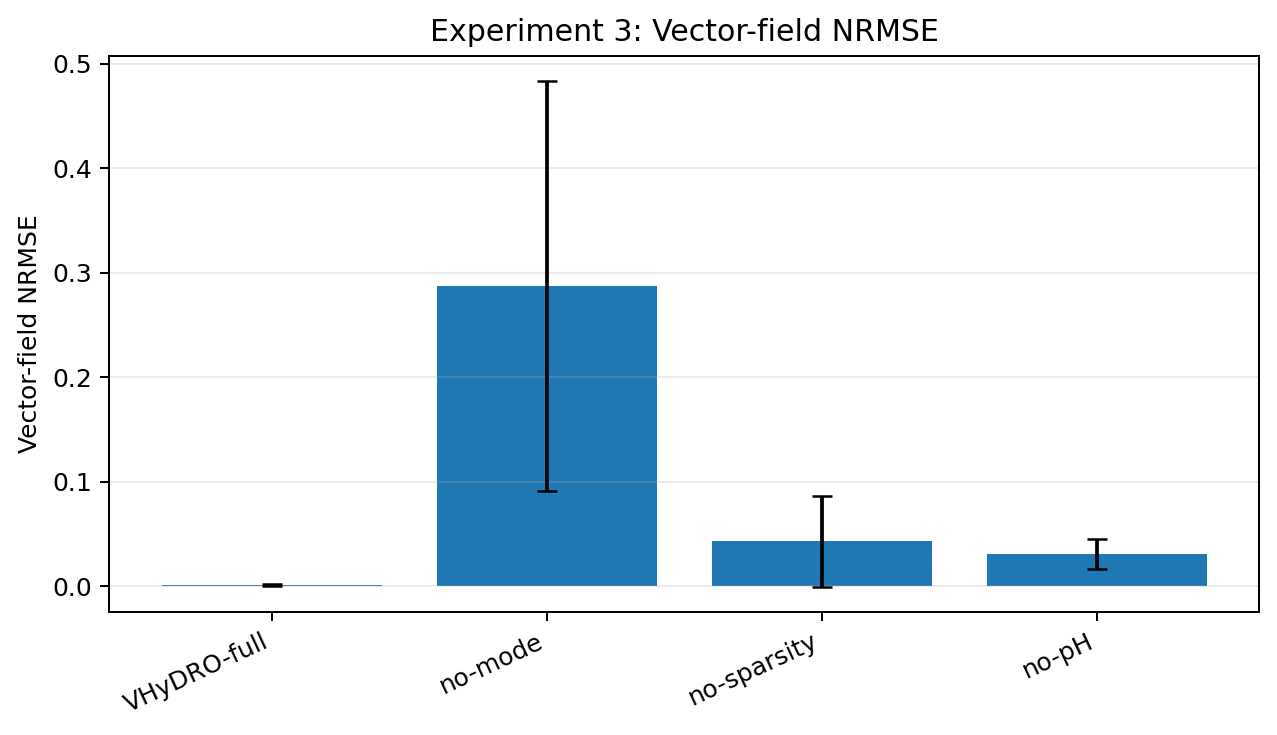}
\end{minipage}\hfill
\begin{minipage}{0.245\textwidth}
\centering
\includegraphics[width=\linewidth]{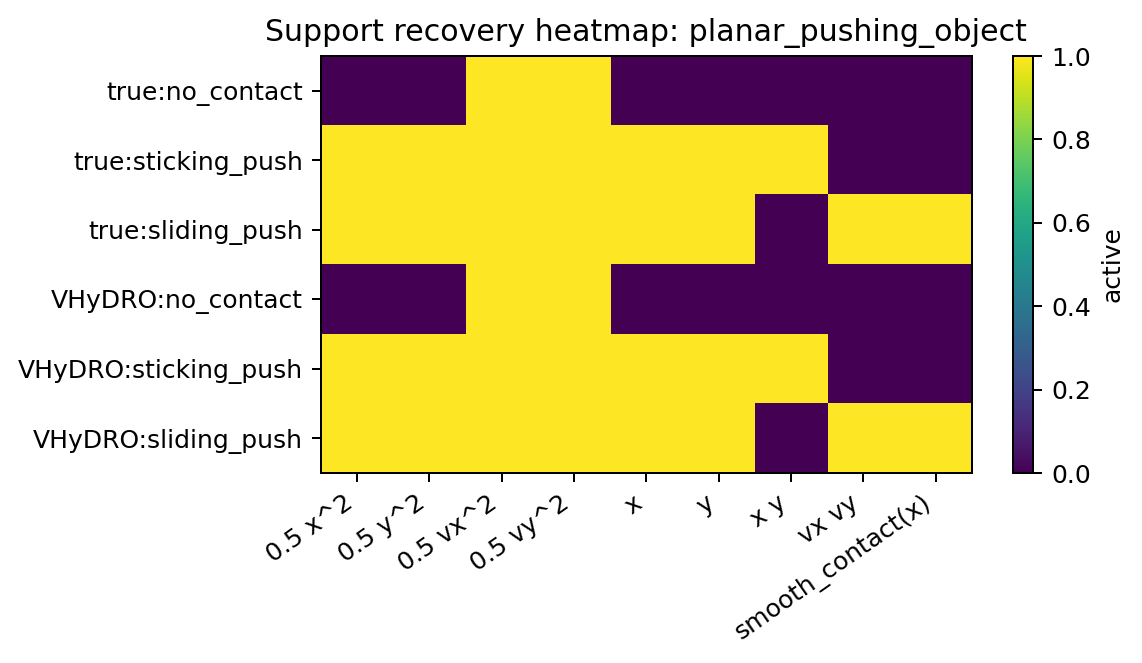}
\end{minipage}
\vspace{-1mm}
\caption{\textbf{Sparse physical-law recovery.} On hybrid systems with known sparse port-Hamiltonian laws, \method{} recovers the active library terms with high support F1 and low coefficient error. The no-mode ablation confounds regimes, the no-sparsity ablation predicts with dense coefficients but loses sparse support recovery, and the no-pH ablation fits a vector field without recoverable port-Hamiltonian coefficients.}
\label{fig:exp3_sparse_ph}
\vspace{-1mm}
\end{figure*}

\paragraph{Interpretation.}
The sparse-law benchmark is designed so that a purely predictive model can
perform well without identifying the governing mechanism. The no-pH baseline
demonstrates this possibility: it achieves a moderate vector-field error, yet
it offers no recoverable port-Hamiltonian coefficients, no sparse support, and
no physical constants. On this controlled known-equation audit, \method{}
simultaneously predicts accurately and recovers the sparse, mode-wise physical
mechanism.

\subsection{Appendix-only planning theorem}
\label{app:exp4_theorem4_diagnostic}

\begin{theorem}[Passivity-calibrated rollouts and KL-robust MPC]
\label{thm:passive_kl_mpc_certificate}
For each action sequence $a\in\mathcal A$, let $Q_a$ be the \method{} predictive law over hybrid trajectories generated, between switching times, by
\begin{equation}
\label{eq:hybrid_sde}
dz_t=\big[(J_{s_t}-R_{s_t})\nabla H_{s_t}(z_t)+G_{s_t}a_t\big]dt+\Sigma_{s_t}^{1/2}dW_t.
\end{equation}
Assume the rollout remains in certified regions $\mathcal D_m$ where $h_m^\star:=\inf_{z\in\mathcal D_m}H_m(z)>-\infty$ and $U_m(z):=H_m(z)-h_m^\star$.  Uniformly over modes, suppose $J_m^\top=-J_m$, $R_m\succeq rI$, $\|G_m\|_{\mathrm{op}}\le g$, $\operatorname{tr}(\Sigma_m)\le\nu^2$, $\|a_t\|_2\le A$, $\|\nabla^2H_m\|_{\mathrm{op}}\le L_H$, and $\|\nabla H_m(z)\|_2^2\ge2\mu U_m(z)$ on $\mathcal D_m$. Assume \(\Sigma_m\succeq 0\), finitely many switches almost surely on
\([0,T]\), and sufficient integrability so that the stochastic integrals
in the hybrid It\^o decomposition are true martingales. Assume also that contact switches are energy-nonexpansive: $U_{s_\tau}(z_\tau)\le U_{s_{\tau^-}}(z_{\tau^-})$.  With
\begin{equation}
\label{eq:energy_constants}
C_E:=\frac{g^2A^2}{2r}+\frac{L_H\nu^2}{2},
\qquad
\alpha:=r\mu,
\end{equation}
we have, for all $a$ and horizons $T$,
\begin{equation}
\label{eq:ph_energy}
\mathbb E_{Q_a}[U_{s_T}(z_T)]
\le
 e^{-\alpha T}\mathbb E_{Q_a}[U_{s_0}(z_0)]
+
\frac{1-e^{-\alpha T}}{\alpha}C_E.
\end{equation}
Let $P_a^\star$ be the true trajectory law and assume uniform predictive calibration $D_{\mathrm{KL}}(P_a^\star\|Q_a)\le\varepsilon_H$.  With $\epsilon_{\mathrm{TV}}:=\min\{1,\sqrt{\varepsilon_H/2}\}$, every bounded cost $C\in[0,C_{\max}]$ and every $\eta_{\mathrm{opt}}$-optimal \method{} plan $\widehat a$ satisfy
\begin{equation}
\label{eq:mpc_bound}
\mathbb E_{P_{\widehat a}^\star}[C]
-
\inf_{a\in\mathcal A}\mathbb E_{P_a^\star}[C]
\le
2C_{\max}\epsilon_{\mathrm{TV}}+\eta_{\mathrm{opt}}.
\end{equation}
For any failure event $\mathcal F$ and any
\(\delta\in[\epsilon_{\mathrm{TV}},1]\),
\begin{equation}
\label{eq:chance_bound}
Q_a(\mathcal F)\le \delta-\epsilon_{\mathrm{TV}}
\quad\Longrightarrow\quad
P_a^\star(\mathcal F)\le\delta.
\end{equation}
If $Q_a(\mathcal F)$ is estimated from $n$ rollouts over a finite evaluated set of $L$ action sequences, the implication remains valid with probability at least $1-\beta$ after replacing $\epsilon_{\mathrm{TV}}$ in~\eqref{eq:chance_bound} by
\begin{equation}
\label{eq:mc_chance_margin}
\epsilon_{\mathrm{TV}}+\sqrt{\frac{\log(L/\beta)}{2n}}.
\end{equation}
\end{theorem}

Theorem~\ref{thm:passive_kl_mpc_certificate} is a conditional downstream
statement: if the learned rollout law is KL-calibrated, the recovered
port-Hamiltonian structure transfers to bounded-energy and
chance-constraint guarantees. The empirical section focuses only on the
three measured links above.

\paragraph{Relation to prior bounds.}
The result combines standard dissipative energy estimates,
Gronwall-style control, and Pinsker transfer from predictive KL to cost
and chance-constraint error. The \method{}-specific role is to connect
recovered sparse hybrid laws to conditional planning guarantees, not to
provide a numerical planning certificate in the present experiments.

This section reports supplementary MPC diagnostics motivated by Theorem~\ref{thm:passive_kl_mpc_certificate}: one-step rollout error, standardized energy-proxy drift, safety-adjusted return, and failure-risk calibration. These diagnostics do not instantiate the theorem as a numerical certificate because the required held-out KL bound \(D_{\mathrm{KL}}(P_a^\star\|Q_a)\) is not estimated.

\paragraph{Protocol.}
We evaluate PushCube-v1, PickCube-v1, PegInsertionSide-v1, and StackCube-v1, the four tasks that instantiate reliably in our local ManiSkill build. All training data come from rollout HDF5 files collected by this experiment. External ManiSkill demonstration files are not mixed into training because observation layouts differ across tasks and control modes. Rollout collection and MPC evaluation are executed serially in the main process to avoid native SAPIEN/ManiSkill subprocess crashes; PyTorch training uses two GPUs.

\paragraph{Diagnostics.}
The main diagnostic reports quantities motivated by
Theorem~\ref{thm:passive_kl_mpc_certificate}: rollout RMSE, common
energy-proxy drift, safety-adjusted return, and episode-level
risk-calibration gap. The standardized energy proxy is a diagnostic for
bounded standardized energy growth and is not an exact Hamiltonian
energy. Because Theorem~\ref{thm:passive_kl_mpc_certificate} is stated
for bounded cost minimization, return diagnostics should be read after
the standard transformation \(C=C_{\max}-R\); the reported returns are
descriptive and are not used to instantiate Eq.~\eqref{eq:mpc_bound}.

\begin{figure*}[t]
  \centering
  \includegraphics[width=0.98\textwidth]{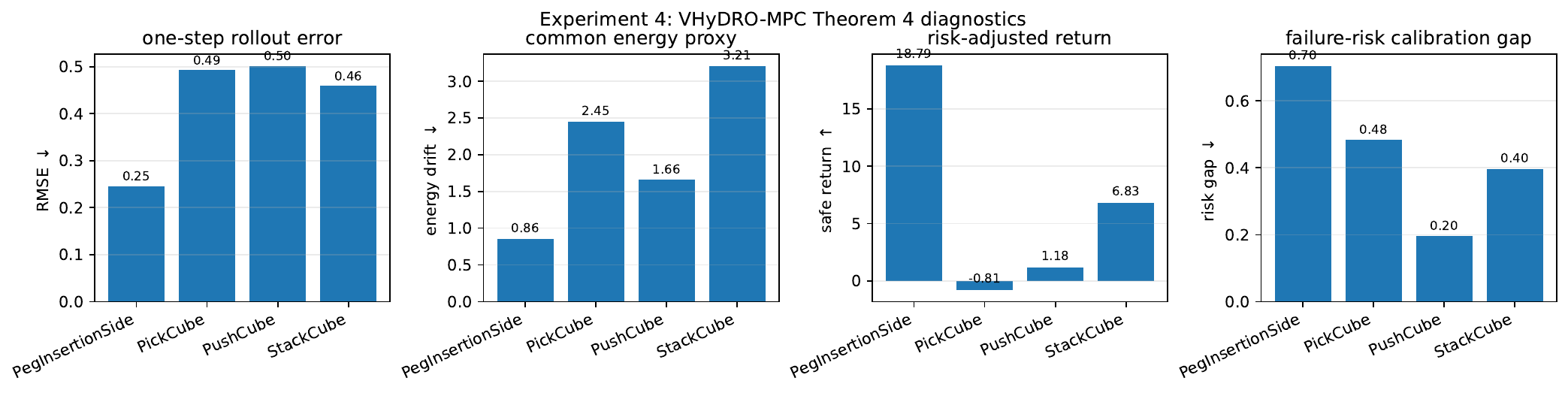}
\caption{\textbf{Supplementary Experiment~4 MPC diagnostics.} Values are
from the online MPC audit. The run uses twenty seeds and should be
interpreted as descriptive diagnostics motivated by
Theorem~\ref{thm:passive_kl_mpc_certificate}, not as a numerical theorem
certificate or task-success evidence.}
 \label{fig:exp4_theorem4_clean}
\end{figure*}

\begin{table*}[t]
\centering
\small
\caption{\textbf{Supplementary Experiment~4 MPC diagnostics.} Values are
from the online MPC audit. The run uses twenty seeds and should be
interpreted as descriptive diagnostics motivated by
Theorem~\ref{thm:passive_kl_mpc_certificate}, not as a numerical theorem
certificate or task-success evidence.}
\label{tab:exp4_theorem4_clean}
\begin{tabular}{lrrrrrrr}
\toprule
Task & RMSE$\downarrow$ & Energy drift$\downarrow$ & Return$\uparrow$ &
Safe return$\uparrow$ & Pred. risk & Risk gap$\downarrow$ & ECE$\downarrow$ \\
\midrule
PegInsertionSide & 0.246 & 0.857 & 29.643 & 18.786 & 0.295 & 0.705 & 0.705 \\
PickCube         & 0.493 & 2.452 & 11.638 & -0.813 & 0.516 & 0.484 & 0.484 \\
PushCube         & 0.502 & 1.658 & 11.171 & 1.179  & 0.637 & 0.196 & 0.308 \\
StackCube        & 0.460 & 3.210 & 20.038 & 6.828  & 0.603 & 0.397 & 0.397 \\
\bottomrule
\end{tabular}
\end{table*}

\FloatBarrier

\paragraph{Interpretation.}
Across twenty seeds, the reported supplementary diagnostics are finite
on all four tasks: mean rollout RMSE \(0.425\), mean common
energy-proxy drift \(2.044\), mean safety-adjusted return \(6.495\), and
mean absolute risk-calibration gap \(0.445\). These values are a
descriptive MPC diagnostic, not a non-vacuity check on the constants of
Theorem~\ref{thm:passive_kl_mpc_certificate}. We do not estimate a
held-out upper bound on \(D_{\mathrm{KL}}(P_a^\star\|Q_a)\), so the
resulting Pinsker \(\epsilon_{\mathrm{TV}}\) is not numerically reported
and Eq.~\eqref{eq:mpc_bound} is not a certified suboptimality interval;
this, together with a per-component empirical instantiation of
\((\nu^{\text{filt}}, \nu^{\text{der}}, \nu^{\text{mode}},
\nu^{\text{pH}})\) from Theorem~\ref{thm:sparse_ph_recovery}, are the
principal follow-ups.

\paragraph{Where the diagnostics remain loose.}
Table~\ref{tab:exp4_theorem4_clean} also exposes where the present MPC
diagnostics remain loose. PickCube exhibits a negative safe-return mean
\((-0.813)\), and PegInsertionSide shows a \(0.705\) risk-calibration gap
together with a \(0.705\) ECE; both are consistent with our \emph{not}
yet having a held-out upper bound on
\(D_{\mathrm{KL}}(P^\star_a\,\|\,Q_a)\), which leaves the Pinsker
\(\epsilon_{\mathrm{TV}}\) unestimated and Eq.~\eqref{eq:mpc_bound}
uncertified numerically. We therefore present Experiment~4 strictly as
supplementary descriptive diagnostics, and identify a calibrated KL upper
bound and the matching per-component
\((\nu^{\text{filt}},\nu^{\text{der}},\nu^{\text{mode}},
\nu^{\text{pH}})\) audit for Theorem~3 as the principal next experiment,
not as an unstated limitation.

\paragraph{Compute and hardware disclosure.}
All experiments were run on a single workstation with an AMD Ryzen
Threadripper~PRO~5975WX CPU (30 physical cores / 60 logical threads) and
two NVIDIA GeForce RTX~4090 GPUs with approximately 24~GB VRAM each. We
used approximately 30 CPU worker threads where parallel data loading was
safe. Hybrid-latent training was parallelized across the two GPUs by running
independent seeds concurrently, while rollout collection and MPC evaluation
in Experiment~4 were executed serially in the main process to avoid native
SAPIEN/ManiSkill subprocess crashes.

The full Experiment~4 audit, including rollout collection, MPC evaluation,
HDF5 serialization, model fitting, and repeated environment-instantiation
attempts across the four ManiSkill tasks and twenty seed configurations,
took approximately \(80\) hours on this machine.
\FloatBarrier

\section{Supplementary proofs}
\label{app:theory}

The main text states the three measured results in compact form. The proofs
below keep the full technical derivations and explicitly map each conclusion
back to the corresponding theorem and equation numbers. Displays labeled with
prefix \texttt{eq:app} are auxiliary equations used only inside the Supplementary
Information; equations without this prefix refer to the numbered equations in
the main paper.

\subsection{Proof of Theorem~\ref{thm:support_safe_budgeted_vhydro}: support-safe budgeted IW/FIVO increment}
\label{app:proof_support_safe_budgeted_vhydro}

\begin{proof}
The theorem is stated around the defensive proposal in Eq.~\eqref{eq:defensive_proposal}.  We prove, in order, the domination and Radon--Nikodym bound in Eq.~\eqref{eq:rn_bound}, the chi-square bound in Eq.~\eqref{eq:chi_square_bound}, the unbiasedness/variance/ESS statement in Eq.~\eqref{eq:variance_ess_bound}, and the budgeted choice of \(\lambda\) in Eq.~\eqref{eq:lambda_budget}.

All quantities are conditional on the realized particle history
\begin{equation}
\label{eq:app_history_t}
h_t=(s_t,z_t,a_t,o_{\le t+1}).
\end{equation}
For readability, write
\begin{equation}
\label{eq:app_support_abbrev}
P:=P_t,
\qquad
Q:=Q_t,
\qquad
Q_\lambda:=Q_{t,\lambda},
\qquad
g:=g_t,
\qquad
Z:=Z_t,
\qquad
M_2:=M_{2,t}.
\end{equation}
With this notation, Eq.~\eqref{eq:defensive_proposal} becomes
\begin{equation}
\label{eq:app_defensive_abbrev}
Q_\lambda=(1-\lambda)Q+\lambda P .
\end{equation}

\paragraph{Defensive-mixture domination.}
For any measurable set \(A\), Eq.~\eqref{eq:app_defensive_abbrev} gives
\begin{equation}
\label{eq:app_domination_measure}
Q_\lambda(A)
=
(1-\lambda)Q(A)+\lambda P(A)
\ge
\lambda P(A).
\end{equation}
Hence, if \(Q_\lambda(A)=0\), then \(P(A)=0\). Therefore
\begin{equation}
\label{eq:app_P_abs_cont_Qlambda}
P\ll Q_\lambda .
\end{equation}
By the Radon--Nikodym theorem, the derivative
\begin{equation}
\label{eq:app_rn_derivative_support}
r(x):=\frac{dP}{dQ_\lambda}(x)
\end{equation}
exists \(Q_\lambda\)-almost surely.

\paragraph{Bounded prior-to-proposal ratio.}
The domination inequality \eqref{eq:app_domination_measure} implies
\begin{equation}
\label{eq:app_measure_inequality_support}
P(A)\le \frac1\lambda Q_\lambda(A)
\qquad
\forall A.
\end{equation}
We claim that
\begin{equation}
\label{eq:app_r_bound_support}
r(x)\le \frac1\lambda
\qquad Q_\lambda\text{-a.s.}
\end{equation}
Suppose, for contradiction, that
\begin{equation}
\label{eq:app_bad_set_r}
B:=\left\{x:r(x)>\frac1\lambda\right\}
\end{equation}
has positive \(Q_\lambda\)-measure. Then
\begin{equation}
\label{eq:app_bad_set_contradiction}
P(B)
=
\int_B r(x)\,Q_\lambda(dx)
>
\frac1\lambda Q_\lambda(B),
\end{equation}
which contradicts Eq.~\eqref{eq:app_measure_inequality_support}. Thus Eq.~\eqref{eq:app_r_bound_support} holds. Together with nonnegativity of Radon--Nikodym derivatives, this is precisely the main-text bound Eq.~\eqref{eq:rn_bound}.

\paragraph{Domination of the filtering target.}
Define the unnormalized target measure
\begin{equation}
\label{eq:app_unnormalized_target}
\widetilde\Pi(A)
:=
\int_A g(x)\,P(dx).
\end{equation}
Because \(g\ge0\), if \(P(A)=0\), then \(\widetilde\Pi(A)=0\). Hence \(\widetilde\Pi\ll P\). Since \(0<Z<\infty\), the normalized measure
\begin{equation}
\label{eq:app_normalized_target}
\Pi(A)
:=
\frac{\widetilde\Pi(A)}{Z}
=
\frac1Z\int_A g(x)\,P(dx)
\end{equation}
is a probability measure. Moreover,
\begin{equation}
\label{eq:app_target_abs_cont}
\Pi\ll P\ll Q_\lambda,
\end{equation}
so \(\Pi\ll Q_\lambda\). This proves the target-support part of Theorem~\ref{thm:support_safe_budgeted_vhydro}.

\paragraph{Derivative of the filtering target.}
Since
\begin{equation}
\label{eq:app_pi_density_wrt_P}
\Pi(dx)=\frac{g(x)}{Z}P(dx)
\end{equation}
and
\begin{equation}
\label{eq:app_P_density_wrt_Qlambda}
P(dx)=r(x)Q_\lambda(dx),
\end{equation}
the Radon--Nikodym chain rule gives
\begin{equation}
\label{eq:app_pi_density_wrt_Qlambda}
\frac{d\Pi}{dQ_\lambda}(x)
=
\frac{g(x)}{Z}r(x)
\qquad Q_\lambda\text{-a.s.}
\end{equation}
The assumption that \(g(x)<\infty\) \((P+Q)\)-a.s. implies \(g(x)<\infty\) \(Q_\lambda\)-a.s., because \(Q_\lambda=(1-\lambda)Q+\lambda P\). Together with Eq.~\eqref{eq:app_r_bound_support}, this gives
\begin{equation}
\label{eq:app_weight_finite_bound}
W(x):=g(x)r(x)<\infty,
\qquad
0\le W(x)\le \frac{g(x)}{\lambda}
\qquad Q_\lambda\text{-a.s.}
\end{equation}
This is the full SI version of the finite-weight statement summarized in Theorem~\ref{thm:support_safe_budgeted_vhydro}.

\paragraph{Chi-square identity.}
By definition,
\begin{equation}
\label{eq:app_chi_square_identity_start}
1+\chi^2(\Pi\|Q_\lambda)
=
\int_{\mathsf X}
\left(
\frac{d\Pi}{dQ_\lambda}(x)
\right)^2
Q_\lambda(dx).
\end{equation}
Using Eq.~\eqref{eq:app_pi_density_wrt_Qlambda},
\begin{equation}
\label{eq:app_chi_square_identity_weight}
1+\chi^2(\Pi\|Q_\lambda)
=
\frac1{Z^2}
\int_{\mathsf X}
g(x)^2 r(x)^2\,Q_\lambda(dx).
\end{equation}

\paragraph{Guide-independent chi-square bound.}
Since \(0\le r(x)\le \lambda^{-1}\) \(Q_\lambda\)-a.s. by Eq.~\eqref{eq:rn_bound},
\begin{equation}
\label{eq:app_r2_le_r}
r(x)^2\le \frac1\lambda r(x)
\qquad Q_\lambda\text{-a.s.}
\end{equation}
Therefore,
\begin{equation}
\label{eq:app_chi_square_bound_derivation}
\begin{aligned}
1+\chi^2(\Pi\|Q_\lambda)
&\le
\frac1{\lambda Z^2}
\int_{\mathsf X}
g(x)^2 r(x)\,Q_\lambda(dx)  \\
&=
\frac1{\lambda Z^2}
\int_{\mathsf X}
g(x)^2\,P(dx)
=
\frac{M_2}{\lambda Z^2}.
\end{aligned}
\end{equation}
Equivalently,
\begin{equation}
\label{eq:app_chi_square_final}
\chi^2(\Pi\|Q_\lambda)
\le
\frac1\lambda\frac{M_2}{Z^2}-1.
\end{equation}
Since \(Z=\mathbb E_P[g(X)]\), \(M_2=\mathbb E_P[g(X)^2]\), and \(\rho_t=M_{2,t}/Z_t^2\), Eq.~\eqref{eq:app_chi_square_final} is exactly the main-text chi-square bound Eq.~\eqref{eq:chi_square_bound}.

\paragraph{Second-moment identity.}
From Eq.~\eqref{eq:app_pi_density_wrt_Qlambda},
\begin{equation}
\label{eq:app_weight_as_density}
W(x)
=
g(x)r(x)
=
Z\frac{d\Pi}{dQ_\lambda}(x).
\end{equation}
Therefore,
\begin{equation}
\label{eq:app_second_moment_identity}
\begin{aligned}
\mathbb E_{Q_\lambda}[W(X)^2]
&=
Z^2
\int_{\mathsf X}
\left(
\frac{d\Pi}{dQ_\lambda}(x)
\right)^2
Q_\lambda(dx)  \\
&=
Z^2\left(1+\chi^2(\Pi\|Q_\lambda)\right).
\end{aligned}
\end{equation}
Combining Eq.~\eqref{eq:app_second_moment_identity} with Eq.~\eqref{eq:app_chi_square_bound_derivation} gives
\begin{equation}
\label{eq:app_second_moment_bound}
\mathbb E_{Q_\lambda}[W(X)^2]
\le
\frac1\lambda M_2
<\infty .
\end{equation}

\paragraph{Unbiasedness of the one-step normalizer estimate.}
Let \(X\sim Q_\lambda\). Then
\begin{equation}
\label{eq:app_unbiased_one_sample}
\mathbb E_{Q_\lambda}[W(X)]
=
\int_{\mathsf X}
g(x)r(x)\,Q_\lambda(dx)
=
\int_{\mathsf X}
g(x)\,P(dx)
=
Z.
\end{equation}
Thus, for conditionally i.i.d. samples \(X^1,\ldots,X^N\sim Q_\lambda\) and
\begin{equation}
\label{eq:app_Zhat_definition}
\widehat Z^{(N)}
:=
\frac1N\sum_{i=1}^N W(X^i),
\end{equation}
we have
\begin{equation}
\label{eq:app_Zhat_unbiased}
\mathbb E[\widehat Z^{(N)}\mid h_t]=Z.
\end{equation}
This is the first line of Eq.~\eqref{eq:variance_ess_bound}.

\paragraph{Variance identity.}
Because \(X^1,\ldots,X^N\) are conditionally i.i.d. given \(h_t\),
\begin{equation}
\label{eq:app_Zhat_variance_start}
\operatorname{Var}(\widehat Z^{(N)}\mid h_t)
=
\frac1N\operatorname{Var}_{Q_\lambda}(W(X)).
\end{equation}
Using Eqs.~\eqref{eq:app_second_moment_identity} and \eqref{eq:app_unbiased_one_sample},
\begin{equation}
\label{eq:app_weight_variance_chi}
\operatorname{Var}_{Q_\lambda}(W(X))
=
Z^2\left(1+\chi^2(\Pi\|Q_\lambda)\right)-Z^2
=
Z^2\chi^2(\Pi\|Q_\lambda).
\end{equation}
Therefore,
\begin{equation}
\label{eq:app_Zhat_variance_chi}
\frac{\operatorname{Var}(\widehat Z^{(N)}\mid h_t)}{Z^2}
=
\frac1N\chi^2(\Pi\|Q_\lambda).
\end{equation}
Applying Eq.~\eqref{eq:app_chi_square_final} yields
\begin{equation}
\label{eq:app_Zhat_variance_final}
\frac{\operatorname{Var}(\widehat Z^{(N)}\mid h_t)}{Z^2}
\le
\frac1N
\left[
\frac{1}{\lambda}
\frac{\mathbb E_P[g(X)^2]}
{\mathbb E_P[g(X)]^2}
-1
\right],
\end{equation}
which is the variance part of Eq.~\eqref{eq:variance_ess_bound}.

\paragraph{Asymptotic normalized ESS.}
Let
\begin{equation}
\label{eq:app_ess_averages}
W_i:=W(X^i),
\qquad
\overline W_N:=\frac1N\sum_{i=1}^N W_i,
\qquad
\overline{W^2}_N:=\frac1N\sum_{i=1}^N W_i^2 .
\end{equation}
By Eqs.~\eqref{eq:app_second_moment_identity} and \eqref{eq:app_unbiased_one_sample},
\begin{equation}
\label{eq:app_ess_moments}
\mathbb E[W_i]=Z\in(0,\infty),
\qquad
\mathbb E[W_i^2]<\infty .
\end{equation}
The strong law of large numbers gives
\begin{equation}
\label{eq:app_ess_slln}
\overline W_N\to Z,
\qquad
\overline{W^2}_N
\to
Z^2\left(1+\chi^2(\Pi\|Q_\lambda)\right)
\qquad\text{a.s.}
\end{equation}
The limiting denominator is strictly positive because \(Z>0\). Hence, on the almost-sure event where both limits hold,
\begin{equation}
\label{eq:app_ess_limit}
\frac{\operatorname{ESS}_N}{N}
=
\frac{
\left(N^{-1}\sum_{i=1}^N W_i\right)^2
}{
N^{-1}\sum_{i=1}^N W_i^2
}
=
\frac{\overline W_N^2}{\overline{W^2}_N}
\to
\frac1{1+\chi^2(\Pi\|Q_\lambda)}.
\end{equation}
The finite-\(N\) convention \(\operatorname{ESS}_N=0\) when \(\sum_i W_i^2=0\) does not affect the limit because \(\overline W_N\to Z>0\). From Eq.~\eqref{eq:app_chi_square_bound_derivation},
\begin{equation}
\label{eq:app_ess_lower_bound}
\frac1{1+\chi^2(\Pi\|Q_\lambda)}
\ge
\lambda
\frac{\mathbb E_P[g(X)]^2}{\mathbb E_P[g(X)^2]}.
\end{equation}
Equations~\eqref{eq:app_ess_limit}--\eqref{eq:app_ess_lower_bound} are exactly the ESS portion of Eq.~\eqref{eq:variance_ess_bound}.

\paragraph{Variance-budget certificate.}
Define
\begin{equation}
\label{eq:app_rho_definition}
\rho
:=
\frac{M_2}{Z^2}
=
\frac{\mathbb E_P[g(X)^2]}{\mathbb E_P[g(X)]^2}.
\end{equation}
Because \(P\) is a probability measure and \(g\in L^2(P)\), Jensen's inequality gives \(M_2\ge Z^2\). Thus \(\rho\in[1,\infty)\). Let \(\bar\rho\ge\rho\) be a certified upper bound. From Eq.~\eqref{eq:app_Zhat_variance_final},
\begin{equation}
\label{eq:app_variance_budget_start}
\frac{\operatorname{Var}(\widehat Z^{(N)}\mid h_t)}{Z^2}
\le
\frac1N
\left(
\frac{\rho}{\lambda}-1
\right)
\le
\frac1N
\left(
\frac{\bar\rho}{\lambda}-1
\right).
\end{equation}
Therefore the guide-independent bound certifies
\begin{equation}
\label{eq:app_variance_budget_target}
\frac{\operatorname{Var}(\widehat Z^{(N)}\mid h_t)}{Z^2}
\le
\tau^2
\end{equation}
whenever
\begin{equation}
\label{eq:app_variance_budget_algebra}
\frac1N
\left(
\frac{\bar\rho}{\lambda}-1
\right)
\le
\tau^2.
\end{equation}
This is equivalent to
\begin{equation}
\label{eq:app_lambda_budget_equivalent}
\lambda
\ge
\frac{\bar\rho}{1+N\tau^2}.
\end{equation}
If \(\bar\rho\le1+N\tau^2\), then the right-hand side of Eq.~\eqref{eq:app_lambda_budget_equivalent} is at most one, so every
\begin{equation}
\label{eq:app_lambda_budget_interval}
\lambda\in
\left[
\frac{\bar\rho}{1+N\tau^2},1
\right]
\end{equation}
certifies the target budget. If \(\bar\rho>1+N\tau^2\), then no \(\lambda\in(0,1]\) can certify the target using this guide-independent upper bound. The function
\begin{equation}
\label{eq:app_lambda_bound_function}
\lambda\mapsto
\frac1N
\left(
\frac{\bar\rho}{\lambda}-1
\right)
\end{equation}
is decreasing on \((0,1]\), so the best certified bound from this inequality is attained at \(\lambda=1\) and equals
\begin{equation}
\label{eq:app_lambda_best_bound}
\frac1N(\bar\rho-1).
\end{equation}
The sufficient condition Eq.~\eqref{eq:app_lambda_budget_equivalent}, together with feasibility \(\bar\rho/(1+N\tau^2)\le1\), is the main-text budget rule Eq.~\eqref{eq:lambda_budget}.

\paragraph{Density-level version.}
This paragraph is an implementation-level SI extension of Theorem~\ref{thm:support_safe_budgeted_vhydro}; it is not needed to state the compact main-text theorem, but it justifies the density formula used in code. Assume that \(P\) and \(Q\) admit tractable densities \(p\) and \(q\) with respect to a common \(\sigma\)-finite reference measure \(\mu\). Then \(Q_\lambda\) admits the density
\begin{equation}
\label{eq:app_q_lambda_density}
q_\lambda(x)
=
(1-\lambda)q(x)+\lambda p(x).
\end{equation}
Since
\begin{equation}
\label{eq:app_density_domination}
q_\lambda(x)\ge \lambda p(x)
\qquad \mu\text{-a.e.},
\end{equation}
we have, with the convention \(0/0:=0\),
\begin{equation}
\label{eq:app_density_ratio_bound}
\frac{p(x)}{q_\lambda(x)}
\le
\frac1\lambda
\qquad q_\lambda\text{-a.e.}
\end{equation}
Moreover, for any measurable \(A\),
\begin{equation}
\label{eq:app_density_rn_check}
\int_A \frac{p(x)}{q_\lambda(x)}q_\lambda(x)\,\mu(dx)
=
\int_A p(x)\,\mu(dx)
=
P(A),
\end{equation}
so \(p/q_\lambda\) is a valid density-level version of \(dP/dQ_\lambda\). Therefore the density-level importance weight
\begin{equation}
\label{eq:app_density_weight}
W(x)=g(x)\frac{p(x)}{q_\lambda(x)}
\end{equation}
agrees with the measure-theoretic weight above up to \(Q_\lambda\)-null sets.

\paragraph{State-dependent or adaptive \(\lambda\).}
This is another SI extension of the compact theorem. The preceding argument was conditional on \(h_t\) and on a fixed \(\lambda\in(0,1]\). If \(\lambda\) is chosen as an \(h_t\)-measurable, possibly random quantity before drawing \(X^1,\ldots,X^N\), then conditional on \((h_t,\lambda)\) the same proof applies verbatim. Taking expectations over the randomness used to choose \(\lambda\) preserves conditional unbiasedness, provided the same realized value of \(\lambda\) is used in both the proposal \(Q_{t,\lambda}\) and the weight \(W_t\).

The result is therefore a one-step conditional statement for the proposal-and-weighting increment. In an IW/FIVO/SMC implementation, it applies to each particle ancestor after conditioning on that ancestor and on the information available before sampling the next particle.
\end{proof}

\subsection{Proof of Theorem~\ref{thm:mode_recovery_vhydro}: Exponential concentration of predictive modes}
\label{app:proof_mode_recovery_vhydro}

\begin{proof}
The main text defines the clamped-mode posterior in Eq.~\eqref{eq:mode_posterior} and the predictive log-ratio in Eq.~\eqref{eq:log_ratio}.  We prove the sharp posterior bound Eq.~\eqref{eq:mode_recovery_sharp}, its simplified exponential form Eq.~\eqref{eq:mode_recovery_simple}, and the variational transfer bound Eq.~\eqref{eq:mode_recovery_variational}.

All probabilities and expectations are taken under the true data-generating law \(\mathbb P^\star\) for the considered constant-mode segment, conditional on the realized action process when actions are treated as interventions. Fix a wrong mode \(s\neq m\) and define
\begin{equation}
\label{eq:app_mode_sum_logratio}
S_{L,s}
:=
\sum_{\ell=0}^{L-1}R_{\ell,s},
\end{equation}
where \(R_{\ell,s}\) is the main-text log-ratio in Eq.~\eqref{eq:log_ratio}.  By definition,
\begin{equation}
\label{eq:app_mode_logratio_decomposition}
R_{\ell,s}
=
\mu_{\ell,s}+\xi_{\ell,s},
\qquad
\mathbb E_\star[\xi_{\ell,s}\mid\mathcal F_\ell]=0.
\end{equation}
Moreover, by the conditional sub-Gaussian assumption,
\begin{equation}
\label{eq:app_mode_subgaussian}
\mathbb E_\star\!\left[
\exp\{t\xi_{\ell,s}\}
\middle|\mathcal F_\ell
\right]
\le
\exp\!\left(\frac{t^2\sigma_{\ell,s}^2}{2}\right)
\qquad
\forall t\in\mathbb R.
\end{equation}
The variable \(\xi_{\ell,s}\) is measurable with respect to the sigma-field generated by \(\mathcal F_\ell\) and the next observation \(O_{\ell+1}\), hence with respect to \(\mathcal F_{\ell+1}\) under the filtration above. Therefore the partial sums
\begin{equation}
\label{eq:app_mode_martingale}
M_{j,s}:=
\sum_{\ell=0}^{j-1}\xi_{\ell,s},
\qquad
j=0,\ldots,L,
\end{equation}
form a martingale with conditionally sub-Gaussian increments.

We first prove a lower-tail inequality for \(S_{L,s}\). By cumulative predictive separation,
\begin{equation}
\label{eq:app_mode_cumulative_separation}
\sum_{\ell=0}^{L-1}\mu_{\ell,s}
\ge
\Gamma_{L,s}
\qquad
\text{almost surely}.
\end{equation}
Hence, for any \(u>0\),
\begin{equation}
\label{eq:app_mode_tail_reduction}
\begin{aligned}
\mathbb P^\star\!\left[
S_{L,s}\le \Gamma_{L,s}-u
\right]
&=
\mathbb P^\star\!\left[
\sum_{\ell=0}^{L-1}
(\mu_{\ell,s}+\xi_{\ell,s})
\le
\Gamma_{L,s}-u
\right]  \\
&\le
\mathbb P^\star\!\left[
\sum_{\ell=0}^{L-1}\xi_{\ell,s}\le -u
\right].
\end{aligned}
\end{equation}
Assume first that \(V_s>0\). For any \(\eta>0\), Markov's inequality gives
\begin{equation}
\label{eq:app_mode_markov}
\begin{aligned}
\mathbb P^\star\!\left[
\sum_{\ell=0}^{L-1}\xi_{\ell,s}\le -u
\right]
&=
\mathbb P^\star\!\left[
\exp\left\{-\eta\sum_{\ell=0}^{L-1}\xi_{\ell,s}\right\}
\ge e^{\eta u}
\right] \\
&\le
 e^{-\eta u}
\mathbb E_\star\!\left[
\exp\left\{-\eta\sum_{\ell=0}^{L-1}\xi_{\ell,s}\right\}
\right].
\end{aligned}
\end{equation}
Using the tower property and applying Eq.~\eqref{eq:app_mode_subgaussian} backwards from \(\ell=L-1\) to \(\ell=0\),
\begin{equation}
\label{eq:app_mode_mgf_sum}
\mathbb E_\star\!\left[
\exp\left\{-\eta\sum_{\ell=0}^{L-1}\xi_{\ell,s}\right\}
\right]
\le
\exp\left(\frac{\eta^2}{2}\sum_{\ell=0}^{L-1}\sigma_{\ell,s}^2\right)
=
\exp\!\left(\frac{\eta^2V_s}{2}\right).
\end{equation}
Therefore,
\begin{equation}
\label{eq:app_mode_tail_eta}
\mathbb P^\star\!\left[
S_{L,s}\le \Gamma_{L,s}-u
\right]
\le
\exp\left(-\eta u+\frac{\eta^2V_s}{2}\right).
\end{equation}
Optimizing over \(\eta>0\) with \(\eta=u/V_s\) yields
\begin{equation}
\label{eq:app_mode_tail_bound}
\mathbb P^\star\!\left[
S_{L,s}\le \Gamma_{L,s}-u
\right]
\le
\exp\left(-\frac{u^2}{2V_s}\right).
\end{equation}
If \(V_s=0\), then \(\sigma_{\ell,s}=0\) for every \(\ell\). In this case, for any fixed nonzero \(t\), Jensen's inequality and Eq.~\eqref{eq:app_mode_subgaussian} imply
\begin{equation}
\label{eq:app_mode_degenerate_jensen}
1
=
\exp\!\left(t\,\mathbb E_\star[\xi_{\ell,s}\mid\mathcal F_\ell]\right)
\le
\mathbb E_\star[\exp\{t\xi_{\ell,s}\}\mid\mathcal F_\ell]
\le
1 .
\end{equation}
Thus equality holds in Jensen's inequality. Since the exponential function is strictly convex, \(\xi_{\ell,s}\) must be conditionally almost surely constant given \(\mathcal F_\ell\). Its conditional mean is zero, so \(\xi_{\ell,s}=0\) almost surely. Consequently,
\begin{equation}
\label{eq:app_mode_degenerate_bound}
S_{L,s}
=
\sum_{\ell=0}^{L-1}\mu_{\ell,s}
\ge
\Gamma_{L,s}
\qquad
\text{almost surely}.
\end{equation}
Thus the desired lower bound also holds in the degenerate case \(V_s=0\).

Now define, for every \(s\neq m\),
\begin{equation}
\label{eq:app_mode_us}
u_s
:=
\sqrt{2V_s\log\frac{M-1}{\delta}}.
\end{equation}
For modes with \(V_s>0\), Eq.~\eqref{eq:app_mode_tail_bound} gives
\begin{equation}
\label{eq:app_mode_union_component}
\mathbb P^\star\!\left[
S_{L,s}\le \Gamma_{L,s}-u_s
\right]
\le
\frac{\delta}{M-1}.
\end{equation}
For modes with \(V_s=0\), the event \(S_{L,s}\ge\Gamma_{L,s}-u_s\) holds almost surely because \(u_s=0\). Hence, by a union bound over the \(M-1\) wrong modes, with probability at least \(1-\delta\), the event
\begin{equation}
\label{eq:app_mode_good_event}
\mathcal E
:=
\left\{
S_{L,s}
\ge
\Gamma_{L,s}
-
\sqrt{2V_s\log\frac{M-1}{\delta}}
\quad
\forall s\neq m
\right\}
\end{equation}
holds.

On \(\mathcal E\), Bayes' rule and the constant-segment posterior from Eq.~\eqref{eq:mode_posterior} give, for each \(s\neq m\),
\begin{equation}
\label{eq:app_mode_odds_bound}
\begin{aligned}
\frac{\Pi_L(S=s)}{\Pi_L(S=m)}
&=
\frac{
\pi_s\prod_{\ell=0}^{L-1}p_s(O_{\ell+1}\mid\mathcal F_\ell)
}{
\pi_m\prod_{\ell=0}^{L-1}p_m(O_{\ell+1}\mid\mathcal F_\ell)
}                                                     \\
&=
\frac{\pi_s}{\pi_m}
\exp\left(-\sum_{\ell=0}^{L-1}
\log\frac{p_m(O_{\ell+1}\mid\mathcal F_\ell)}{p_s(O_{\ell+1}\mid\mathcal F_\ell)}
\right)                                               \\
&=
\frac{\pi_s}{\pi_m}\exp(-S_{L,s})                    \\
&\le
\exp\left(
B_s-
\Gamma_{L,s}
+
\sqrt{2V_s\log\frac{M-1}{\delta}}
\right).
\end{aligned}
\end{equation}
Therefore,
\begin{equation}
\label{eq:app_mode_sum_odds_bound}
\sum_{s\neq m}
\frac{\Pi_L(S=s)}{\Pi_L(S=m)}
\le
\sum_{s\neq m}
\exp\left(
B_s-
\Gamma_{L,s}
+
\sqrt{2V_s\log\frac{M-1}{\delta}}
\right)
=
A_{L,\delta}^{\sharp}.
\end{equation}
Let
\begin{equation}
\label{eq:app_mode_TL}
T_L
:=
\sum_{s\neq m}
\frac{\Pi_L(S=s)}{\Pi_L(S=m)}.
\end{equation}
Since
\begin{equation}
\label{eq:app_mode_normalization_TL}
1
=
\Pi_L(S=m)+\sum_{s\neq m}\Pi_L(S=s)
=
\Pi_L(S=m)(1+T_L),
\end{equation}
we have
\begin{equation}
\label{eq:app_mode_wrong_prob_TL}
\Pi_L(S\neq m)
=
\sum_{s\neq m}\Pi_L(S=s)
=
\frac{T_L}{1+T_L}.
\end{equation}
On \(\mathcal E\), \(T_L\le A_{L,\delta}^{\sharp}\). Since \(x\mapsto x/(1+x)\) is increasing on \([0,\infty)\),
\begin{equation}
\label{eq:app_mode_sharp_final}
\Pi_L(S\neq m)
\le
\frac{A_{L,\delta}^{\sharp}}{1+A_{L,\delta}^{\sharp}}
\le
A_{L,\delta}^{\sharp}.
\end{equation}
This is exactly the sharp exact-posterior recovery bound in Eq.~\eqref{eq:mode_recovery_sharp}.

We now prove the simplified bound. If
\begin{equation}
\label{eq:app_mode_simplified_assumptions}
\Gamma_{L,s}\ge L\Delta,
\qquad
B_s\le B,
\qquad
V_s\le V_{\max},
\end{equation}
for all \(s\neq m\), then
\begin{equation}
\label{eq:app_mode_A_simplify}
\begin{aligned}
A_{L,\delta}^{\sharp}
&=
\sum_{s\neq m}
\exp\left(
B_s-
\Gamma_{L,s}
+
\sqrt{2V_s\log\frac{M-1}{\delta}}
\right)  \\
&\le
(M-1)
\exp\left(
B-L\Delta+
\sqrt{2V_{\max}\log\frac{M-1}{\delta}}
\right)
= A_{L,\delta}.
\end{aligned}
\end{equation}
Since \(x\mapsto x/(1+x)\) is increasing,
\begin{equation}
\label{eq:app_mode_simplified_final_general}
\Pi_L(S\neq m)
\le
\frac{A_{L,\delta}}{1+A_{L,\delta}}
\le
A_{L,\delta}.
\end{equation}
If additionally \(\sigma_{\ell,s}\le\sigma\) for all \(\ell\) and \(s\neq m\), then
\begin{equation}
\label{eq:app_mode_Vmax_Lsigma}
V_s
=
\sum_{\ell=0}^{L-1}\sigma_{\ell,s}^2
\le
L\sigma^2,
\end{equation}
so \(V_{\max}\le L\sigma^2\), which gives
\begin{equation}
\label{eq:app_mode_simple_final}
\Pi_L(S\neq m)
\le
(M-1)
\exp\left(
B-L\Delta+
\sigma\sqrt{2L\log\frac{M-1}{\delta}}
\right).
\end{equation}
This is the main-text simplified bound Eq.~\eqref{eq:mode_recovery_simple}.

It remains to transfer the bound from the constant-segment model posterior \(\Pi_L\) to the amortized variational posterior \(q_\phi\). Suppose that, on the realized observations,
\begin{equation}
\label{eq:app_mode_variational_KL_assumption}
\mathrm{KL}\!\left(
q_\phi(S\mid O_{1:L},A_{0:L-1})
\,\middle\|\,
\Pi_L(S)
\right)
\le
\varepsilon_q.
\end{equation}
Using the convention
\begin{equation}
\label{eq:app_mode_TV_definition}
\operatorname{TV}(P,Q)
:=
\sup_{\mathcal A}|P(\mathcal A)-Q(\mathcal A)|,
\end{equation}
Pinsker's inequality gives
\begin{equation}
\label{eq:app_mode_pinsker}
\operatorname{TV}(q_\phi,\Pi_L)
\le
\sqrt{
\frac12
\mathrm{KL}(q_\phi\|\Pi_L)
}
\le
\sqrt{\frac{\varepsilon_q}{2}}.
\end{equation}
Apply this to the event \(\mathcal A:=\{S\neq m\}\). Then
\begin{equation}
\label{eq:app_mode_variational_event_transfer}
q_\phi(\mathcal A)
\le
\Pi_L(\mathcal A)
+
\operatorname{TV}(q_\phi,\Pi_L).
\end{equation}
Combining Eq.~\eqref{eq:app_mode_variational_event_transfer} with Eq.~\eqref{eq:app_mode_sharp_final} yields
\begin{equation}
\label{eq:app_mode_variational_final_sharp}
q_\phi(S\neq m\mid O_{1:L},A_{0:L-1})
\le
\frac{A_{L,\delta}^{\sharp}}{1+A_{L,\delta}^{\sharp}}
+
\sqrt{\frac{\varepsilon_q}{2}}.
\end{equation}
This is exactly the variational transfer statement in Eq.~\eqref{eq:mode_recovery_variational}.  Using \(A_{L,\delta}^{\sharp}/(1+A_{L,\delta}^{\sharp})\le A_{L,\delta}^{\sharp}\le A_{L,\delta}\) gives the simplified variational bound
\begin{equation}
\label{eq:app_mode_variational_simple}
q_\phi(S\neq m\mid O_{1:L},A_{0:L-1})
\le
A_{L,\delta}
+
\sqrt{\frac{\varepsilon_q}{2}}.
\end{equation}

Finally, if the KL condition in Eq.~\eqref{eq:app_mode_variational_KL_assumption} holds only with probability at least \(1-\delta_q\), then a union bound over the high-probability event \(\mathcal E\) in Eq.~\eqref{eq:app_mode_good_event} and the KL event gives probability at least \(1-\delta-\delta_q\). This last union-bound variant is a useful SI extension of the compact theorem statement.
\end{proof}

\subsection{Proof of Theorem~\ref{thm:sparse_ph_recovery}: filtering- and mode-robust sparse port-Hamiltonian recovery}
\label{app:proof_sparse_ph_recovery}

\begin{proof}
The main text specifies the sparse port-Hamiltonian model in Eq.~\eqref{eq:sparse_ph_model}, the plug-in regression in Eq.~\eqref{eq:plugin_regression}, the score/RSC assumptions in Eq.~\eqref{eq:score_rsc}, the estimator in Eq.~\eqref{eq:sparse_lasso}, the recovery bounds in Eq.~\eqref{eq:sparse_bounds}, and the high-probability penalty choice in Eq.~\eqref{eq:lambdam_choice}.  The proof below keeps the full Lasso argument and explicitly maps the final deterministic and probabilistic conclusions to those equations.

All deterministic statements below are conditional on the realized \method{} design \((\widehat A,\widehat b)\). Let
\begin{equation}
\label{eq:app_sparse_delta_def}
\widehat\Delta:=\widehat\xi_m-\xi_m^\star .
\end{equation}
If \(\widehat\Delta=0\), then the coefficient and prediction bounds in Eq.~\eqref{eq:sparse_bounds} are immediate. Thus, whenever we divide by \(\|\widehat\Delta\|_2\), we assume \(\widehat\Delta\neq0\).

By optimality of \(\widehat\xi_m\) in Eq.~\eqref{eq:sparse_lasso},
\begin{equation}
\label{eq:app_sparse_optimality}
\frac1{2n_m}
\|\widehat b-\widehat A\widehat\xi_m\|_2^2
+
\lambda_m\|\widehat\xi_m\|_1
\le
\frac1{2n_m}
\|\widehat b-\widehat A\xi_m^\star\|_2^2
+
\lambda_m\|\xi_m^\star\|_1 .
\end{equation}
Using Eq.~\eqref{eq:plugin_regression}, namely \(\zeta_m=\widehat b-\widehat A\xi_m^\star\), and using \(\widehat\xi_m=\xi_m^\star+\widehat\Delta\), we have
\begin{equation}
\label{eq:app_sparse_residual_shift}
\widehat b-\widehat A\widehat\xi_m
=
\zeta_m-\widehat A\widehat\Delta .
\end{equation}
Substituting Eq.~\eqref{eq:app_sparse_residual_shift} into Eq.~\eqref{eq:app_sparse_optimality}, expanding the square, and cancelling \((2n_m)^{-1}\|\zeta_m\|_2^2\) from both sides yields the basic inequality
\begin{equation}
\label{eq:app_sparse_basic_inequality}
\frac1{2n_m}
\|\widehat A\widehat\Delta\|_2^2
\le
\frac1{n_m}\zeta_m^\top\widehat A\widehat\Delta
+
\lambda_m
\left(
\|\xi_m^\star\|_1
-
\|\xi_m^\star+\widehat\Delta\|_1
\right).
\end{equation}

We first control the perturbation term. By Holder's inequality and the score part of Eq.~\eqref{eq:score_rsc},
\begin{equation}
\label{eq:app_sparse_score_control}
\begin{aligned}
\frac1{n_m}\zeta_m^\top\widehat A\widehat\Delta
&\le
\left|
\left\langle
\frac1{n_m}\widehat A^\top\zeta_m,
\widehat\Delta
\right\rangle
\right|                                                    \\
&\le
\left\|
\frac1{n_m}\widehat A^\top\zeta_m
\right\|_\infty
\|\widehat\Delta\|_1
\le
\frac{\lambda_m}{2}\|\widehat\Delta\|_1 .
\end{aligned}
\end{equation}
Next, since \(S=\operatorname{supp}(\xi_m^\star)\), we have \(\xi_{m,S^c}^\star=0\). Therefore,
\begin{equation}
\label{eq:app_sparse_l1_support}
\begin{aligned}
\|\xi_m^\star\|_1
-
\|\xi_m^\star+\widehat\Delta\|_1
&=
\|\xi_{m,S}^\star\|_1
-
\|\xi_{m,S}^\star+\widehat\Delta_S\|_1
-
\|\widehat\Delta_{S^c}\|_1                                  \\
&\le
\|\widehat\Delta_S\|_1
-
\|\widehat\Delta_{S^c}\|_1,
\end{aligned}
\end{equation}
where the final inequality follows from the triangle inequality.

Combining Eqs.~\eqref{eq:app_sparse_basic_inequality}, \eqref{eq:app_sparse_score_control}, and \eqref{eq:app_sparse_l1_support}, and using \(\|\widehat\Delta\|_1=\|\widehat\Delta_S\|_1+\|\widehat\Delta_{S^c}\|_1\), we obtain
\begin{equation}
\label{eq:app_sparse_cone_inequality}
\begin{aligned}
\frac1{2n_m}
\|\widehat A\widehat\Delta\|_2^2
&\le
\frac{\lambda_m}{2}
\left(
\|\widehat\Delta_S\|_1+\|\widehat\Delta_{S^c}\|_1
\right)
+
\lambda_m
\left(
\|\widehat\Delta_S\|_1
-
\|\widehat\Delta_{S^c}\|_1
\right)                                                    \\
&=
\frac{3\lambda_m}{2}\|\widehat\Delta_S\|_1
-
\frac{\lambda_m}{2}\|\widehat\Delta_{S^c}\|_1 .
\end{aligned}
\end{equation}
The left-hand side is nonnegative and \(\lambda_m>0\). Hence the right-hand side must be nonnegative, which implies the cone condition
\begin{equation}
\label{eq:app_sparse_cone_condition}
\|\widehat\Delta_{S^c}\|_1
\le
3\|\widehat\Delta_S\|_1 .
\end{equation}
Thus \(\widehat\Delta\) belongs to the cone on which the RSC part of Eq.~\eqref{eq:score_rsc} holds.

Multiplying Eq.~\eqref{eq:app_sparse_cone_inequality} by \(2\) and dropping the negative term gives
\begin{equation}
\label{eq:app_sparse_prediction_pre_rsc}
\frac1{n_m}\|\widehat A\widehat\Delta\|_2^2
\le
3\lambda_m\|\widehat\Delta_S\|_1 .
\end{equation}
By the restricted strong-convexity assumption in Eq.~\eqref{eq:score_rsc}, applied using Eq.~\eqref{eq:app_sparse_cone_condition},
\begin{equation}
\label{eq:app_sparse_rsc_applied}
\kappa_m\|\widehat\Delta\|_2^2
\le
\frac1{n_m}\|\widehat A\widehat\Delta\|_2^2 .
\end{equation}
Combining Eqs.~\eqref{eq:app_sparse_prediction_pre_rsc} and \eqref{eq:app_sparse_rsc_applied}, and using
\begin{equation}
\label{eq:app_sparse_support_l1_l2}
\|\widehat\Delta_S\|_1
\le
\sqrt{k}\|\widehat\Delta_S\|_2
\le
\sqrt{k}\|\widehat\Delta\|_2,
\end{equation}
we get
\begin{equation}
\label{eq:app_sparse_l2_before_division}
\kappa_m\|\widehat\Delta\|_2^2
\le
3\lambda_m\sqrt{k}\|\widehat\Delta\|_2 .
\end{equation}
If \(\widehat\Delta\neq0\), division by \(\|\widehat\Delta\|_2\) gives
\begin{equation}
\label{eq:app_sparse_l2_bound}
\|\widehat\Delta\|_2
\le
\frac{3\sqrt{k}\lambda_m}{\kappa_m}
\le
\frac{4\sqrt{k}\lambda_m}{\kappa_m}.
\end{equation}
Together with the trivial case \(\widehat\Delta=0\), Eq.~\eqref{eq:app_sparse_l2_bound} proves the coefficient-error bound in Eq.~\eqref{eq:sparse_bounds}.

The empirical vector-field bound follows from Eqs.~\eqref{eq:app_sparse_prediction_pre_rsc} and \eqref{eq:app_sparse_l2_bound}:
\begin{equation}
\label{eq:app_sparse_prediction_final}
\frac1{n_m}\|\widehat A\widehat\Delta\|_2^2
\le
3\lambda_m\sqrt{k}\|\widehat\Delta\|_2
\le
\frac{9k\lambda_m^2}{\kappa_m}
\le
\frac{16k\lambda_m^2}{\kappa_m}.
\end{equation}
This proves the second bound in Eq.~\eqref{eq:sparse_bounds}. The equivalent displayed form in terms of the per-sample port-Hamiltonian vector field follows directly from the stacked definition of \(\widehat A\) in Eq.~\eqref{eq:plugin_regression}.

We now prove exact support recovery after thresholding. Since
\begin{equation}
\label{eq:app_sparse_infty_bound}
\|\widehat\Delta\|_\infty
\le
\|\widehat\Delta\|_2
\le
\frac{4\sqrt{k}\lambda_m}{\kappa_m},
\end{equation}
every inactive coordinate \(j\notin S\) satisfies
\begin{equation}
\label{eq:app_sparse_inactive_bound}
|\widehat\xi_{m,j}|
=
|\widehat\xi_{m,j}-\xi_{m,j}^\star|
=
|\widehat\Delta_j|
\le
\frac{4\sqrt{k}\lambda_m}{\kappa_m}.
\end{equation}
Therefore, if \(\tau>4\sqrt{k}\lambda_m/\kappa_m\), no inactive coordinate is selected. For any active coordinate \(j\in S\),
\begin{equation}
\label{eq:app_sparse_active_lower_bound}
|\widehat\xi_{m,j}|
\ge
|\xi_{m,j}^\star|-|\widehat\Delta_j|
\ge
\beta_{\min,m}-\frac{4\sqrt{k}\lambda_m}{\kappa_m}.
\end{equation}
If \(\beta_{\min,m}>8\sqrt{k}\lambda_m/\kappa_m\), then the interval
\begin{equation}
\label{eq:app_sparse_threshold_interval}
\left(
\frac{4\sqrt{k}\lambda_m}{\kappa_m},
\beta_{\min,m}-\frac{4\sqrt{k}\lambda_m}{\kappa_m}
\right)
\end{equation}
is nonempty. For any \(\tau\) in this interval, Eqs.~\eqref{eq:app_sparse_inactive_bound} and \eqref{eq:app_sparse_active_lower_bound} imply
\begin{equation}
\label{eq:app_sparse_support_final}
\{j:|\widehat\xi_{m,j}|>\tau\}=S.
\end{equation}
This proves the support-recovery part of Theorem~\ref{thm:sparse_ph_recovery}.

It remains to prove the high-probability score bound and connect it to Eq.~\eqref{eq:lambdam_choice}. Suppose
\begin{equation}
\label{eq:app_sparse_residual_decomposition}
\zeta_m=e_m+u_m,
\end{equation}
where \(e_m\) is conditionally sub-Gaussian with scale \(\sigma_m\), and suppose
\begin{equation}
\label{eq:app_sparse_u_bound}
\left\|
\frac1{n_m}\widehat A^\top u_m
\right\|_\infty
\le
\nu_m.
\end{equation}
For each column \(\widehat A_{\cdot j}\), define
\begin{equation}
\label{eq:app_sparse_Xj}
X_j
:=
\frac1{n_m}\widehat A_{\cdot j}^{\top}e_m .
\end{equation}
If \(L_{A,m}=0\), then \(\widehat A^\top e_m=0\), and the stochastic score bound is immediate. Otherwise, conditional on \(\widehat A\), the sub-Gaussian assumption gives, for every \(s\in\mathbb R\),
\begin{equation}
\label{eq:app_sparse_Xj_mgf}
\mathbb E\!\left[
\exp(sX_j)\mid\widehat A
\right]
\le
\exp\!\left(
\frac{\sigma_m^2s^2\|\widehat A_{\cdot j}\|_2^2}{2n_m^2}
\right).
\end{equation}
By the definition of \(L_{A,m}\) in Eq.~\eqref{eq:lambdam_choice},
\begin{equation}
\label{eq:app_sparse_LA_bound}
\frac{\|\widehat A_{\cdot j}\|_2}{n_m}
\le
\frac{L_{A,m}}{\sqrt{n_m}},
\end{equation}
so \(X_j\), conditional on \(\widehat A\), is sub-Gaussian with variance proxy \(\sigma_m^2L_{A,m}^2/n_m\). Hence, for every \(t>0\),
\begin{equation}
\label{eq:app_sparse_Xj_tail}
\mathbb P\left(
|X_j|>t\mid \widehat A
\right)
\le
2\exp\left(
-\frac{n_m t^2}{2\sigma_m^2L_{A,m}^2}
\right).
\end{equation}
Set
\begin{equation}
\label{eq:app_sparse_tdelta}
t_\delta
:=
L_{A,m}\sigma_m
\sqrt{\frac{2\log(2p/\delta)}{n_m}} .
\end{equation}
Applying the union bound over \(j=1,\ldots,p\), we obtain
\begin{equation}
\label{eq:app_sparse_union_bound}
\mathbb P\left(
\left\|
\frac1{n_m}\widehat A^\top e_m
\right\|_\infty
>
t_\delta
\;\middle|\;
\widehat A
\right)
\le
\delta .
\end{equation}
Therefore, with conditional probability at least \(1-\delta\),
\begin{equation}
\label{eq:app_sparse_stochastic_score_bound}
\left\|
\frac1{n_m}\widehat A^\top e_m
\right\|_\infty
\le
L_{A,m}\sigma_m
\sqrt{\frac{2\log(2p/\delta)}{n_m}} .
\end{equation}
On this same event, using Eq.~\eqref{eq:app_sparse_residual_decomposition} and Eq.~\eqref{eq:app_sparse_u_bound},
\begin{equation}
\label{eq:app_sparse_total_score_bound}
\begin{aligned}
\left\|
\frac1{n_m}\widehat A^\top \zeta_m
\right\|_\infty
&\le
\left\|
\frac1{n_m}\widehat A^\top e_m
\right\|_\infty
+
\left\|
\frac1{n_m}\widehat A^\top u_m
\right\|_\infty                                      \\
&\le
L_{A,m}\sigma_m
\sqrt{\frac{2\log(2p/\delta)}{n_m}}
+
\nu_m .
\end{aligned}
\end{equation}
Thus, any penalty satisfying the main-text choice Eq.~\eqref{eq:lambdam_choice} ensures
\begin{equation}
\label{eq:app_sparse_score_half_lambda}
\left\|
\frac1{n_m}\widehat A^\top\zeta_m
\right\|_\infty
\le
\frac{\lambda_m}{2},
\end{equation}
which is exactly the score condition in Eq.~\eqref{eq:score_rsc}. Since the conditional probability bound holds for every realized \(\widehat A\), the same probability bound holds unconditionally by the tower property.

If the RSC part of Eq.~\eqref{eq:score_rsc} holds deterministically for the realized design, then the deterministic part of the proof applies on the score event above, so Eq.~\eqref{eq:sparse_bounds} and the support-recovery statement hold with probability at least \(1-\delta\). If instead RSC holds on an event \(\mathcal E_{\mathrm{rsc}}\) with probability at least \(1-\delta_{\mathrm{rsc}}\), then the union bound gives
\begin{equation}
\label{eq:app_sparse_rsc_union}
\mathbb P\bigl(\text{score condition}\cap\mathcal E_{\mathrm{rsc}}\bigr)
\ge
1-\delta-\delta_{\mathrm{rsc}},
\end{equation}
and the deterministic conclusions apply on this intersection.

Finally, the main text summarizes all filtering, derivative, mode, and plug-in effects inside the perturbation \(\nu_m\).  The full SI decomposition is as follows. If
\begin{equation}
\label{eq:app_sparse_u_decomposition}
u_m
=
u_m^{\mathrm{filt}}
+
u_m^{\mathrm{der}}
+
u_m^{\mathrm{mode}}
+
u_m^{\mathrm{pH}},
\end{equation}
then by the triangle inequality,
\begin{equation}
\label{eq:app_sparse_u_component_bound}
\left\|
\frac1{n_m}\widehat A^\top u_m
\right\|_\infty
\le
\sum_{q\in\{\mathrm{filt},\mathrm{der},\mathrm{mode},\mathrm{pH}\}}
\left\|
\frac1{n_m}\widehat A^\top u_m^q
\right\|_\infty .
\end{equation}
Thus, if each component is bounded by \(\nu_m^q\), then the deterministic perturbation part of Eq.~\eqref{eq:lambdam_choice} holds with
\begin{equation}
\label{eq:app_sparse_nu_decomposition}
\nu_m
=
\nu_m^{\mathrm{filt}}
+
\nu_m^{\mathrm{der}}
+
\nu_m^{\mathrm{mode}}
+
\nu_m^{\mathrm{pH}}.
\end{equation}
This final decomposition is more detailed than the compact main-text theorem but is useful for ablations that separately measure filtering error, derivative-estimation error, mode impurity, and port-Hamiltonian plug-in error.
\end{proof}

\paragraph{What the bound assumes, and what would instantiate it quantitatively.}
Theorem~\ref{thm:sparse_ph_recovery} is conditional on two ingredients that
are not directly certified in Experiment~3.
\emph{(i) Restricted-eigenvalue (RE) constant \(\kappa_m\).}
We assume Eq.~\eqref{eq:score_rsc} holds with some \(\kappa_m>0\) on the cone
\(\|\Delta_{S^c}\|_1 \le 3\|\Delta_S\|_1\); we do not verify this constant
directly. On the four controlled hybrid systems used in Experiment~3, the
candidate library has at most \(p\le12\) terms and the true active support has
size at most \(k\le3\), which is the low-sparsity regime in which Lasso oracle
bounds are typically informative \citep{bickel2009simultaneous}. An explicit
RE certificate via a sparse-eigenvalue lower bound is left as future work
\citep{rudelson2012reconstruction}.
\emph{(ii) Four-component perturbation budget
\(\nu_m = \nu_m^{\mathrm{filt}} + \nu_m^{\mathrm{der}}
+ \nu_m^{\mathrm{mode}} + \nu_m^{\mathrm{pH}}\).}
The decomposition in Appendix~\ref{app:proof_sparse_ph_recovery} separates
filtering error, derivative-estimation error, mode impurity, and
port-Hamiltonian plug-in error. We do not separately measure these four
components in the present experiments; the empirical support F1 and relative
coefficient error test the aggregate recovery outcome, not the individual
perturbation budgets.
\subsection{Proof of Theorem~\ref{thm:passive_kl_mpc_certificate}: passivity-calibrated rollouts and KL-robust MPC}
\label{app:proof_passive_kl_mpc_certificate}

\begin{proof}
Appendix~\ref{app:exp4_theorem4_diagnostic} defines the learned hybrid
SDE in Eq.~\eqref{eq:hybrid_sde}, the constants \(C_E\) and \(\alpha\)
in Eq.~\eqref{eq:energy_constants}, the energy inequality in
Eq.~\eqref{eq:ph_energy}, the MPC suboptimality bound in
Eq.~\eqref{eq:mpc_bound}, the deterministic chance certificate in
Eq.~\eqref{eq:chance_bound}, and the Monte Carlo margin in
Eq.~\eqref{eq:mc_chance_margin}. We prove these claims and include two
SI-level details not displayed in the compact theorem: the Markov
high-energy tail bound and the full hybrid It\^o decomposition across
switches.

Fix an admissible action sequence \(a\in\mathcal A\) and work under the corresponding learned predictive law \(Q_a\).  Let
\begin{equation}
\label{eq:app_mpc_switch_times}
0<\tau_1<\cdots<\tau_N\le T
\end{equation}
be the genuine mode-switching times up to horizon \(T\), and set
\begin{equation}
\label{eq:app_mpc_tau_endpoints}
\tau_0:=0,
\qquad
\tau_{N+1}:=T.
\end{equation}
By assumption, \(N<\infty\) almost surely.  On each interval \([\tau_j,\tau_{j+1})\), the mode is constant; denote it by \(m_j\).

We first derive a drift bound on a single constant-mode interval.  Fix one such interval and abbreviate
\begin{equation}
\label{eq:app_mpc_mode_abbrev}
H=H_{m_j},\quad
U=U_{m_j},\quad
J=J_{m_j},\quad
R=R_{m_j},\quad
G=G_{m_j},\quad
\Sigma=\Sigma_{m_j}.
\end{equation}
For \(z\in\mathcal D_{m_j}\) and a predictable action value \(a_t\), the infinitesimal generator corresponding to Eq.~\eqref{eq:hybrid_sde} applied to \(H\) is
\begin{equation}
\label{eq:app_mpc_generator}
\mathcal L_{m_j}^{a_t}H(z)
=
\nabla H(z)^\top
\Big[(J-R)\nabla H(z)+Ga_t\Big]
+
\frac12
\operatorname{tr}\!\left(\Sigma\nabla^2H(z)\right).
\end{equation}
The skew-symmetry of \(J\) gives
\begin{equation}
\label{eq:app_mpc_skew_zero}
\nabla H(z)^\top J\nabla H(z)=0.
\end{equation}
Since \(R\succeq rI\),
\begin{equation}
\label{eq:app_mpc_dissipation}
-\nabla H(z)^\top R\nabla H(z)
\le
-r\|\nabla H(z)\|_2^2 .
\end{equation}
The input term is bounded by Cauchy's inequality and Young's inequality:
\begin{equation}
\label{eq:app_mpc_input_bound}
\nabla H(z)^\top Ga_t
\le
\|\nabla H(z)\|_2\,\|G\|_{\mathrm{op}}\|a_t\|_2
\le
gA\|\nabla H(z)\|_2
\le
\frac r2\|\nabla H(z)\|_2^2
+
\frac{g^2A^2}{2r}.
\end{equation}
For the diffusion term, because \(\Sigma\succeq0\), \(\operatorname{tr}(\Sigma)\le\nu^2\), and \(\|\nabla^2H(z)\|_{\mathrm{op}}\le L_H\), we have
\begin{equation}
\label{eq:app_mpc_diffusion_bound}
\frac12
\operatorname{tr}\!\left(\Sigma\nabla^2H(z)\right)
\le
\frac12
\left|
\operatorname{tr}\!\left(\Sigma\nabla^2H(z)\right)
\right|
\le
\frac12
\|\nabla^2H(z)\|_{\mathrm{op}}\operatorname{tr}(\Sigma)
\le
\frac{L_H\nu^2}{2}.
\end{equation}
Combining Eqs.~\eqref{eq:app_mpc_generator}--\eqref{eq:app_mpc_diffusion_bound} yields
\begin{equation}
\label{eq:app_mpc_generator_pre_PL}
\mathcal L_{m_j}^{a_t}H(z)
\le
-\frac r2\|\nabla H(z)\|_2^2
+
\frac{g^2A^2}{2r}
+
\frac{L_H\nu^2}{2}.
\end{equation}
Since \(U(z)=H(z)-h_{m_j}^\star\) differs from \(H(z)\) only by a constant,
\begin{equation}
\label{eq:app_mpc_LU_equals_LH}
\mathcal L_{m_j}^{a_t}U(z)
=
\mathcal L_{m_j}^{a_t}H(z).
\end{equation}
By the Polyak--Lojasiewicz condition,
\begin{equation}
\label{eq:app_mpc_PL}
\|\nabla H(z)\|_2^2
\ge
2\mu U(z).
\end{equation}
Therefore, for all \(z\in\mathcal D_{m_j}\),
\begin{equation}
\label{eq:app_mpc_drift_bound}
\mathcal L_{m_j}^{a_t}U(z)
\le
-r\mu U(z)
+
\frac{g^2A^2}{2r}
+
\frac{L_H\nu^2}{2}
=
-\alpha U(z)+C_E,
\end{equation}
where \(\alpha\) and \(C_E\) match their main-text definitions in Eq.~\eqref{eq:energy_constants}.

We now sum It\^o's formula over the constant-mode intervals.  Since the rollout remains in the certified regions almost surely, the drift bound Eq.~\eqref{eq:app_mpc_drift_bound} applies along the path.  On each interval \([\tau_j,\tau_{j+1})\), It\^o's formula gives
\begin{equation}
\label{eq:app_mpc_interval_ito}
U_{m_j}(z_{\tau_{j+1}^-})
=
U_{m_j}(z_{\tau_j})
+
\int_{\tau_j}^{\tau_{j+1}}
\mathcal L_{m_j}^{a_t}
U_{m_j}(z_t)\,dt
+
M_j,
\end{equation}
where \(M_j\) is the stochastic integral over that interval.  The assumed square-integrability condition implies that the sum of these stochastic integrals is a true martingale with mean zero; equivalently, one may localize on compact sublevel sets and then pass to the limit.

At a genuine switching time \(\tau_j\), define the active-energy jump
\begin{equation}
\label{eq:app_mpc_energy_jump}
\Delta_j
:=
U_{s_{\tau_j}}(z_{\tau_j})
-
U_{s_{\tau_j^-}}(z_{\tau_j^-}).
\end{equation}
By the passive-switch assumption,
\begin{equation}
\label{eq:app_mpc_passive_jump}
\Delta_j\le0
\qquad\text{a.s.}
\end{equation}
Summing over all intervals and genuine switches up to \(T\) yields the hybrid It\^o decomposition
\begin{equation}
\label{eq:app_mpc_hybrid_ito}
U_{s_T}(z_T)
=
U_{s_0}(z_0)
+
\int_0^T
\mathcal L_{s_t}^{a_t}U_{s_t}(z_t)\,dt
+
M_T
+
\sum_{0<\tau_j\le T}\Delta_j,
\end{equation}
where \(\mathbb E_{Q_a}[M_T]=0\) and \(\Delta_j\le0\) almost surely.  Taking expectations and using Eq.~\eqref{eq:app_mpc_drift_bound} gives
\begin{equation}
\label{eq:app_mpc_integral_inequality}
\mathbb E_{Q_a}[U_{s_T}(z_T)]
\le
\mathbb E_{Q_a}[U_{s_0}(z_0)]
+
\int_0^T
\left(
-\alpha\,
\mathbb E_{Q_a}[U_{s_t}(z_t)]
+
C_E
\right)dt .
\end{equation}
Define
\begin{equation}
\label{eq:app_mpc_psi_def}
\psi(t):=\mathbb E_{Q_a}[U_{s_t}(z_t)] .
\end{equation}
Then Eq.~\eqref{eq:app_mpc_integral_inequality} reads
\begin{equation}
\label{eq:app_mpc_gronwall_input}
\psi(T)
\le
\psi(0)
+
\int_0^T(-\alpha\psi(t)+C_E)dt .
\end{equation}
By the integral form of Gronwall's inequality,
\begin{equation}
\label{eq:app_mpc_gronwall_output}
\psi(T)
\le
 e^{-\alpha T}\psi(0)
+
\frac{1-e^{-\alpha T}}{\alpha}C_E .
\end{equation}
Substituting back the definition of \(\psi\) proves
\begin{equation}
\label{eq:app_mpc_energy_final}
\mathbb E_{Q_a}[U_{s_T}(z_T)]
\le
 e^{-\alpha T}
\mathbb E_{Q_a}[U_{s_0}(z_0)]
+
\frac{1-e^{-\alpha T}}{\alpha}C_E,
\end{equation}
which is the main-text energy certificate Eq.~\eqref{eq:ph_energy}.

The following tail bound is more detailed than the compact main-text statement but is useful for experiments that measure high-energy failures.  Define
\begin{equation}
\label{eq:app_mpc_Ubar_def}
\overline U_T(a)
:=
 e^{-\alpha T}
\mathbb E_{Q_a}[U_{s_0}(z_0)]
+
\frac{1-e^{-\alpha T}}{\alpha}C_E.
\end{equation}
Since \(U_{s_T}(z_T)\ge0\), Markov's inequality gives, for every \(B>0\),
\begin{equation}
\label{eq:app_mpc_energy_tail}
Q_a\!\left(U_{s_T}(z_T)\ge B\right)
\le
\frac{\mathbb E_{Q_a}[U_{s_T}(z_T)]}{B}
\le
\frac{\overline U_T(a)}{B}.
\end{equation}

We now prove the KL-robust risk and MPC claims.  Let
\begin{equation}
\label{eq:app_mpc_tv_def}
d_{\mathrm{TV}}(P,Q):=\sup_{\mathcal B}|P(\mathcal B)-Q(\mathcal B)|
\end{equation}
denote total variation distance.  By Pinsker's inequality and the uniform predictive calibration assumption,
\begin{equation}
\label{eq:app_mpc_pinsker}
d_{\mathrm{TV}}(P_a^\star,Q_a)
\le
\sqrt{\frac12D_{\mathrm{KL}}(P_a^\star\|Q_a)}
\le
\sqrt{\frac{\varepsilon_H}{2}}.
\end{equation}
Since total variation is always at most one,
\begin{equation}
\label{eq:app_mpc_epsilon_tv_def}
d_{\mathrm{TV}}(P_a^\star,Q_a)
\le
\epsilon_{\mathrm{TV}}
:=
\min\left\{1,\sqrt{\frac{\varepsilon_H}{2}}\right\}
\qquad
\forall a\in\mathcal A .
\end{equation}
Therefore, for every measurable event \(\mathcal F\),
\begin{equation}
\label{eq:app_mpc_event_transfer}
P_a^\star(\mathcal F)
\le
Q_a(\mathcal F)+d_{\mathrm{TV}}(P_a^\star,Q_a)
\le
Q_a(\mathcal F)+\epsilon_{\mathrm{TV}} .
\end{equation}
Applying this to the high-energy event \(\mathcal E_B=\{U_{s_T}(z_T)\ge B\}\) and using Eq.~\eqref{eq:app_mpc_energy_tail} gives the SI-level true high-energy certificate
\begin{equation}
\label{eq:app_mpc_true_energy_tail}
P_a^\star(\mathcal E_B)
\le
\min\left\{
1,
\frac{\overline U_T(a)}{B}+\epsilon_{\mathrm{TV}}
\right\}.
\end{equation}
This bound is additional theory beyond the compact theorem statement, but it follows immediately from Eq.~\eqref{eq:ph_energy} and the KL calibration assumption.

Next, let \(C(\tau)\in[0,C_{\max}]\) be any bounded measurable planning cost.  Since \(C/C_{\max}\in[0,1]\), the variational characterization of total variation gives
\begin{equation}
\label{eq:app_mpc_cost_tv}
\left|
\mathbb E_{P_a^\star}[C(\tau)]
-
\mathbb E_{Q_a}[C(\tau)]
\right|
\le
C_{\max}\,
d_{\mathrm{TV}}(P_a^\star,Q_a)
\le
C_{\max}\epsilon_{\mathrm{TV}}
\qquad
\forall a\in\mathcal A .
\end{equation}
Using Eq.~\eqref{eq:app_mpc_cost_tv} for \(\widehat a\) gives
\begin{equation}
\label{eq:app_mpc_first_transfer}
\mathbb E_{P_{\widehat a}^\star}[C]
\le
\mathbb E_{Q_{\widehat a}}[C]
+
C_{\max}\epsilon_{\mathrm{TV}} .
\end{equation}
By \(\eta_{\mathrm{opt}}\)-optimality of \(\widehat a\) under the \method{} predictive law,
\begin{equation}
\label{eq:app_mpc_eta_opt}
\mathbb E_{Q_{\widehat a}}[C]
\le
\inf_{a\in\mathcal A}\mathbb E_{Q_a}[C]
+
\eta_{\mathrm{opt}} .
\end{equation}
Moreover, for every \(a\in\mathcal A\), Eq.~\eqref{eq:app_mpc_cost_tv} also gives
\begin{equation}
\label{eq:app_mpc_second_transfer}
\mathbb E_{Q_a}[C]
\le
\mathbb E_{P_a^\star}[C]
+
C_{\max}\epsilon_{\mathrm{TV}}.
\end{equation}
Taking the infimum over \(a\in\mathcal A\) yields
\begin{equation}
\label{eq:app_mpc_inf_transfer}
\inf_{a\in\mathcal A}\mathbb E_{Q_a}[C]
\le
\inf_{a\in\mathcal A}\mathbb E_{P_a^\star}[C]
+
C_{\max}\epsilon_{\mathrm{TV}} .
\end{equation}
Combining Eqs.~\eqref{eq:app_mpc_first_transfer}, \eqref{eq:app_mpc_eta_opt}, and \eqref{eq:app_mpc_inf_transfer} gives
\begin{equation}
\label{eq:app_mpc_suboptimality_final}
\mathbb E_{P_{\widehat a}^\star}[C]
-
\inf_{a\in\mathcal A}\mathbb E_{P_a^\star}[C]
\le
2C_{\max}\epsilon_{\mathrm{TV}}
+
\eta_{\mathrm{opt}},
\end{equation}
which is the main-text MPC bound Eq.~\eqref{eq:mpc_bound}.  Since \(\epsilon_{\mathrm{TV}}\le\sqrt{\varepsilon_H/2}\), Eq.~\eqref{eq:app_mpc_suboptimality_final} also implies the slightly looser form
\begin{equation}
\label{eq:app_mpc_suboptimality_sqrt}
\mathbb E_{P_{\widehat a}^\star}[C]
-
\inf_{a\in\mathcal A}\mathbb E_{P_a^\star}[C]
\le
2C_{\max}\sqrt{\frac{\varepsilon_H}{2}}
+
\eta_{\mathrm{opt}} .
\end{equation}
If the true infimum is attained by \(a^\star\), then
\begin{equation}
\label{eq:app_mpc_oracle_attained}
\inf_{a\in\mathcal A}\mathbb E_{P_a^\star}[C]
=
\mathbb E_{P_{a^\star}^\star}[C],
\end{equation}
which gives the oracle-comparison form.

For the deterministic chance constraint, let \(\mathcal F\) be any measurable failure event.  If
\begin{equation}
\label{eq:app_mpc_chance_assumption}
Q_a(\mathcal F)
\le
\delta-\epsilon_{\mathrm{TV}},
\qquad
\delta\in[\epsilon_{\mathrm{TV}},1],
\end{equation}
then Eq.~\eqref{eq:app_mpc_event_transfer} gives
\begin{equation}
\label{eq:app_mpc_chance_final}
P_a^\star(\mathcal F)
\le
Q_a(\mathcal F)+\epsilon_{\mathrm{TV}}
\le
\delta .
\end{equation}
This is the deterministic implication in Eq.~\eqref{eq:chance_bound}.

It remains to prove the rollout certificate and connect it to Eq.~\eqref{eq:mc_chance_margin}.  Let
\begin{equation}
\label{eq:app_mpc_eval_set}
\mathcal A_{\mathrm{eval}}
=
\{a^1,\ldots,a^L\}
\end{equation}
be the finite evaluated action set.  For each \(a\in\mathcal A_{\mathrm{eval}}\), define
\begin{equation}
\label{eq:app_mpc_mc_indicators}
Y_{a,i}:=\mathbf 1\{\tau_a^{(i)}\in\mathcal F\},
\qquad
\tau_a^{(i)}\overset{\mathrm{i.i.d.}}{\sim}Q_a .
\end{equation}
Then \(Y_{a,i}\in[0,1]\),
\begin{equation}
\label{eq:app_mpc_mc_mean}
\mathbb E[Y_{a,i}]=Q_a(\mathcal F),
\qquad
\widehat Q_{a,n}(\mathcal F)
=
\frac1n\sum_{i=1}^nY_{a,i}.
\end{equation}
For a fixed \(a\), Hoeffding's inequality gives
\begin{equation}
\label{eq:app_mpc_hoeffding}
\Pr\left(
Q_a(\mathcal F)
>
\widehat Q_{a,n}(\mathcal F)
+
\epsilon_n(L,\beta)
\right)
\le
\exp\!\left(-2n\epsilon_n(L,\beta)^2\right)
=
\frac{\beta}{L},
\end{equation}
where
\begin{equation}
\label{eq:app_mpc_epsilon_n_def}
\epsilon_n(L,\beta)
:=
\sqrt{\frac{\log(L/\beta)}{2n}}.
\end{equation}
Taking a union bound over the \(L\) evaluated actions yields, with probability at least \(1-\beta\), the simultaneous event
\begin{equation}
\label{eq:app_mpc_simultaneous_mc_event}
Q_a(\mathcal F)
\le
\widehat Q_{a,n}(\mathcal F)
+
\epsilon_n(L,\beta)
\qquad
\forall a\in\mathcal A_{\mathrm{eval}} .
\end{equation}
On this simultaneous event, if the possibly data-dependent selected action \(\widetilde a\in\mathcal A_{\mathrm{eval}}\) satisfies
\begin{equation}
\label{eq:app_mpc_selected_action_condition}
\widehat Q_{\widetilde a,n}(\mathcal F)
\le
\delta-
\epsilon_{\mathrm{TV}}-
\epsilon_n(L,\beta),
\end{equation}
then
\begin{equation}
\label{eq:app_mpc_selected_Q_bound}
Q_{\widetilde a}(\mathcal F)
\le
\delta-
\epsilon_{\mathrm{TV}}.
\end{equation}
Applying the deterministic chance certificate in Eq.~\eqref{eq:chance_bound} to \(\widetilde a\) gives
\begin{equation}
\label{eq:app_mpc_selected_true_bound}
P_{\widetilde a}^\star(\mathcal F)\le\delta .
\end{equation}
Equations~\eqref{eq:app_mpc_epsilon_n_def}--\eqref{eq:app_mpc_selected_true_bound} prove the Monte Carlo margin statement Eq.~\eqref{eq:mc_chance_margin}, completing the proof of Theorem~\ref{thm:passive_kl_mpc_certificate}.
\end{proof}

\end{document}